\documentclass[10pt]{article}

\usepackage[accepted]{tmlr}
\usepackage{verbatim}

\usepackage[bookmarksnumbered=true, bookmarksopen=true, bookmarksopenlevel=1,
colorlinks=true,
linkcolor={red!50!black},
citecolor={blue!50!black},
urlcolor={blue!80!black},
pdfstartview=Fit, pdfpagemode=UseOutlines]{hyperref}% hyperlinks
\usepackage{url}            % simple URL typesetting
\usepackage{booktabs}       % professional-quality tables
\usepackage{amsfonts}       % blackboard math symbols
\usepackage{nicefrac}       % compact symbols for 1/2, etc.
\usepackage{microtype}      % microtypography
\usepackage{xcolor}         % colors

\usepackage[textsize=tiny]{todonotes}

\usepackage{graphicx}
\usepackage{subfigure}
\usepackage{wrapfig}
\usepackage{bm}

\usepackage{amsmath}
\usepackage{amssymb}
\usepackage{mathtools}
\usepackage{amsthm, thmtools}
\usepackage{thm-restate}

\usepackage[capitalize,noabbrev]{cleveref}
\newcommand{\crefp}[1]{(\cref{#1})}
\theoremstyle{plain}
\newtheorem{theorem}{Theorem}[section]
\newtheorem{proposition}[theorem]{Proposition}

\newtheorem{corollary}[theorem]{Corollary}
\theoremstyle{definition}
\newtheorem{definition}[theorem]{Definition}
\newtheorem{assumption}[theorem]{Assumption}

\newtheorem{example}[theorem]{Example}
\makeatletter
\def\th@remark{%
	\thm@headfont{\bfseries}%
	\normalfont % body font
	\thm@preskip\topsep \divide\thm@preskip\tw@
	\thm@postskip\thm@preskip
}
\makeatother
\theoremstyle{remark}
\newtheorem{remark}[theorem]{Remark}

\usepackage{algorithm}
\usepackage{algpseudocode}
\usepackage{caption}

\title{Global Safe Sequential Learning via Efficient Knowledge Transfer}

\author{\name Cen-You Li \email cen-you.li@campus.tu-berlin.de \\
\addr Technical University of Berlin, Germany\\
Bosch Center for Artificial Intelligence, Germany
\AND
\name Olaf Duennbier \email olaf.duennbier@de.bosch.com \\
\addr Robert Bosch GmbH, Germany
\AND
\name Marc Toussaint \email toussaint@tu-berlin.de\\
\addr Technical University of Berlin, Germany
\AND
\name Barbara Rakitsch\thanks{Equal contribution} \email barbara.rakitsch@de.bosch.com\\
\addr Bosch Center for Artificial Intelligence, Germany
\AND
\name Christoph Zimmer\footnotemark[1] \email christoph.zimmer@de.bosch.com\\
\addr Bosch Center for Artificial Intelligence, Germany
}

\def\githublink{\url{https://github.com/cenyou/TransferSafeSequentialLearning}}

\def\month{1}  % Insert correct month for camera-ready version
\def\year{2025} % Insert correct year for camera-ready version
\def\openreview{\url{https://openreview.net/forum?id=PtD2gVmb3J}} % Insert correct link to OpenReview for camera-ready version

\begin{document}

\maketitle

\begin{abstract}
Sequential learning methods, such as active learning and Bayesian optimization, aim to select the most informative data for task learning. In many applications, however, data selection is constrained by unknown safety conditions, motivating the development of safe learning approaches.
A promising line of safe learning methods uses Gaussian processes to model safety conditions, restricting data selection to areas with high safety confidence.
However, these methods are limited to local exploration around an initial seed dataset, as safety confidence centers around observed data points. As a consequence, task exploration is slowed down and safe regions disconnected from the initial seed dataset remain unexplored.
In this paper, we propose safe transfer sequential learning to accelerate task learning and to expand the explorable safe region.
By leveraging abundant offline data from a related source task, our approach guides exploration in the target task more effectively.
We also provide a theoretical analysis to explain why single-task method cannot cope with disconnected regions.
Finally, we introduce a computationally efficient approximation of our method that reduces runtime through pre-computations.
Our experiments demonstrate that this approach, compared to state-of-the-art methods, learns tasks with lower data consumption and enhances global exploration across multiple disjoint safe regions, while maintaining comparable computational efficiency.
%Our code is available at \githublink.
\end{abstract}
	%\begin{keyword}
	%	Gaussian Process \sep Bayesian Active Learning \sep Bayesian Optimization \sep Safe Learning
	%\end{keyword}
	
	\section{Introduction}\label{section-introduction}

 Despite the great success of machine learning, acquiring data remains a significant challenge.
One prominent approach is to consider experimental design~\citep{Lindley1956, ChalonerVerdinelli1995, brochu2010tutorial}.
In particular, active learning (AL)~\citep{krause08a, KumarGupta2020} and Bayesian optimization (BO)~\citep{brochu2010tutorial, Snoek_etal12bo} resort to a sequential data selection process in which the most informative data points are incrementally added to the dataset.
The methods begin with a small dataset, iteratively compute an acquisition function to prioritize data points for querying, select new data based on this information, receive observations from the oracle, and update the belief. 
This process is repeated until the learning goal is achieved, or until the acquisition budget is exhausted.
These learning algorithms often utilize Gaussian processes (GPs,~\cite{GPbook}) as surrogate models for the acquisition computation~\citep{krause08a,brochu2010tutorial}.

	%Despite the great success of machine learning, accessing data is a non-trivial task.
	%Testing new molecules and materials is time consuming~\citep{Pyzer-Knapp2018, griffiths_constrained_2020, Lookman2019-yy}; segmenting images takes significant effort; verifying a new investment policy can be expensive.
	%One prominent approach is to consider experimental design~\citep{Lindley1956, ChalonerVerdinelli1995, brochu2010tutorial}.
	%In particular, active learning (AL)~\citep{krause08a, KumarGupta2020} and Bayesian optimization (BO)~\citep{brochu2010tutorial, Snoek_etal12bo} resort to a sequential data selection process where the most informative data are chosen to be evaluated.
	%The methods initiate with a small amount of data, iteratively compute an acquisition function, query new data according to the acquisition score (score of informativeness), receive observations from the oracle, and update the belief, until the learning goal is achieved, or until the acquisition budget is exhausted.
	%These learning algorithms often utilize Gaussian processes (GPs,~\cite{GPbook}) as surrogate models for the acquisition computation~\citep{krause08a,brochu2010tutorial}.

In many applications, such as spinal cord stimulation~\citep{Harkema2011} and robotic learning~\citep{Berkenkamp_2016, dominik_baumann_gosafe_2021}, data acquisition can introduce safety risks due to unknown safety constraints in the input space. For instance, tuning a robot controller requires testing various controller parameters; however, certain parameter settings may lead to unsafe behaviors, such as a drone flying at high speed toward a human—an issue only observed after executing the controller~\citep{Berkenkamp_2016}. This scenario highlights the need for a safe learning approach that selects data points being safe and maximally informative within safety limits.
One effective approach to safe learning is to model safety constraints using additional GPs~\citep{sui15safeopt, Schreiter2015, ZimmerNEURIPS2018_b197ffde, yanan_sui_stagewise_2018, matteo_turchetta_safe_2019, berkenkamp2020bayesian, Sergeyev2020_safe_bo, dominik_baumann_gosafe_2021, cyli2022}. 
These algorithms begin with a small set of safe observations, and define a safe set to restrict exploration to regions with high safety confidence.
As learning progresses, this safe set expands, allowing the explorable area to grow over time.
Safe learning approaches have also been explored in related fields, such as Markov Decision Processes~\citep{matteo_turchetta_safe_2019} and reinforcement learning~\citep{garcia15a_JMLR}.

%In many applications such as spinal cord stimulation~\citep{Harkema2011} and robotic learning~\citep{Berkenkamp_2016, dominik_baumann_gosafe_2021}, data evaluations can trigger safety concerns.
%However, the safety constraints are unknown in the input space (a priori unknown safety conditions).
%For example, a robotic controller is tuned by evaluating the performance of various controller parameters, while bad parameters can result in unsafe behaviors, e.g. a drone flying in high speed towards a human, which we can only observe after executing the controller~\citep{Berkenkamp_2016}.
%This motivates a safe learning approach, which selects safe data that is as informative as possible up to the safety constraint.
%One effective approach of performing safe learning is to model the safety constraints with additional GPs~\citep{sui15safeopt, Schreiter2015, ZimmerNEURIPS2018_b197ffde, yanan_sui_stagewise_2018, matteo_turchetta_safe_2019, berkenkamp2020bayesian, Sergeyev2020_safe_bo, dominik_baumann_gosafe_2021, cyli2022}.
%	The algorithms initiate with given safe observations.
%	A safe set is then defined to restrict the exploration to regions with high safety confidence.
%	The safe set expands as the learning proceeds, and thus the explorable area grows.
%	Safe learning is also considered in related domains such as Markov Decision Processes~\citep{matteo_turchetta_safe_2019} and reinforcement learning~\citep{garcia15a_JMLR}.

While safe learning methods have demonstrated significant impact, several challenges remain. 
First, the GP hyperparameters must be specified before exploration begins~\citep{sui15safeopt, Berkenkamp_2016, berkenkamp2020bayesian} or be fitted using an initially small dataset~\citep{Schreiter2015, ZimmerNEURIPS2018_b197ffde, cyli2022}.
In addition, safe learning algorithms often suffer from local exploration: GP models are typically smooth, with uncertainty increasing beyond the boundaries of the reachable safe set.
This results in slow convergence, and disconnected safe regions are often classified as unsafe and remain unexplored.
We provide a detailed analysis and visual illustration of this issue in~\cref{section4-gp_no_jump}. 
In practice, local exploration complicates the deployment of safe learning algorithms, as domain experts must supply safe data from multiple distinct safe regions.

	%In this paper, we focus on GPs as they are often considered the gold-standard when it comes to calibrated uncertainties.
	%While such safe learning methods have achieved a huge impact, few challenges remain.
	%SafeOpt-like approaches~\citep{sui15safeopt, Berkenkamp_2016, berkenkamp2020bayesian} assume well-calibrated GPs are given prior to the exploration, which is often difficult.
	%In many scenarios, GP models are fitted in the experiments, which however requires more initial data to guarantee stable and accurate model selections~\citep{Schreiter2015, ZimmerNEURIPS2018_b197ffde, cyli2022}.
	%Firstly, GP priors need to be given prior to the exploration~\citep{sui15safeopt, Berkenkamp_2016, berkenkamp2020bayesian} or fitted with initial data (note that accessing the data is expensive)~\citep{Schreiter2015, ZimmerNEURIPS2018_b197ffde, cyli2022}.
	%In addition, safe learning algorithms suffer from local exploration.
	%GPs are typically smooth and the uncertainty increases beyond the reachable safe set boundary.
	%Disconnected safe regions will be classified as unsafe and will remain unexplored.
	%We provide a detailed analysis and illustration of explorable regions in~\cref{section4-gp_no_jump}.
	%In reality, local exploration increases the effort of deploying safe learning algorithms because the domain experts need to provide safe data from multiple safe regions.
	
\begin{figure*}[t]
\vskip 0.2in
\begin{center}
\centerline{\includegraphics[width=\textwidth]{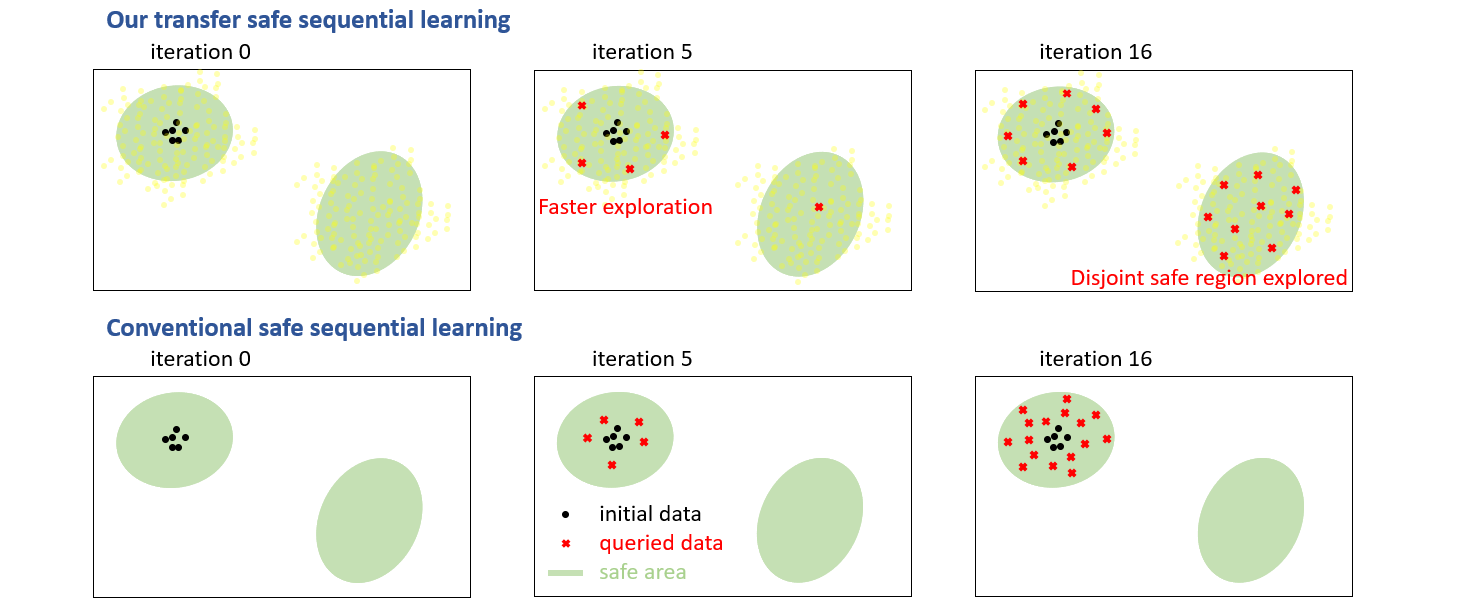}}
\caption{
Illustration: Safe sequential learning with transfer (top) and conventional (bottom) learning.
The light yellow data points represent source data.
The main benefit of transfer learning is to accelerate exploration and identify larger and potentially disjoint safe regions by leveraging the source data.}

\label{figure1}
\end{center}
\vskip -0.2in
\end{figure*}

\paragraph{Our contribution:}\label{section-introduction-our_contribution}
Safe learning generally begins with prior knowledge~\citep{Schreiter2015, sui15safeopt, berkenkamp2020bayesian}.
We assume that correlated experiments have already been performed, and their results are readily available. 
This assumption enables transfer learning, offering two key benefits (see also ~\cref{figure1}): 
(1) Exploration and expansion of safe regions are significantly accelerated, and (2) disconnected safe regions can be explored allowing to discover larger safe regions.
Both advantages are made possible by guidance from the source task.
%(1) exploration and expansion of safe regions are significantly accelerated, and (2) guidance from the source task can support exploration of safe regions disconnected from the initial target data, allowing for the discovery of larger safe regions.
We empirically demonstrate both of the benefits and provide a theoretical analysis showing that conventional single-task approaches cannot identify unconnected safe regions.
Real-world applications of this approach are ubiquitous, including simulation-to-reality transfer~\citep{Marco2017sim2realRLviaBO}, serial production, and multi-fidelity modeling~\citep{ShiboLi20multifidelitybo}.

Transfer learning can be implemented by jointly modeling the source and target tasks as multi-output  GPs~\citep{JournelHuijbregts1976miningGeo, alvarez2012kernels}.
However, GPs are notorious for their cubic time complexity due to the inversion of Gram matrices~\crefp{section-background-GP}.
Consequently, large volumes of source data significantly increase computational time, which is often a bottleneck in real-world experiments.
To address this, we modularize the multi-output GPs, allowing source-related components to be precomputed and fixed, which reduces the computational complexity while retaining the benefits of transfer learning.

In summary, we
1) introduce the idea of safe transfer sequential learning,
2) derive that conventional single-task approaches cannot discover disjoint safe regions,
3) provide a modularized approach to multi-output GPs that alleviates the computational burden of incorporating the source data, and
4) demonstrate the empirical efficacy on safe AL problems.

\paragraph{Related work:}\label{section-introduction-related_work}
Safe learning is considered in many applications including AL, BO, Markov Decision Processes~\citep{matteo_turchetta_safe_2019} and reinforcement learning~\citep{garcia15a_JMLR}.
In this paper, we focus on GP learning problems, as GPs are considered the gold-standard when it comes to calibrated uncertainties which is particularly important for safe learning under uncertainty.
Previous works~\citep{Gelbart2014ConstrBO, hernandez-lobatob15_constrainedBO, Hernandez-Lobato_etal2016_constrainedBO} investigated constrained learning with GPs by incorporating constraints directly into the acquisition function (e.g., discounting the acquisition score by the probability of constraint violation). 
However, these approaches do not exclude unsafe data from the search pool, and generally address non-safety-critical applications.
A safe set concept was introduced for safe BO~\citep{sui15safeopt} and safe AL~\citep{Schreiter2015}, and later extended to BO with multiple safety constraints~\citep{berkenkamp2020bayesian}, to AL for time series modeling~\citep{ZimmerNEURIPS2018_b197ffde}, and to AL for multi-output problems~\citep{cyli2022}.
For safe BO,~\cite{yanan_sui_stagewise_2018} proposed a two-stage approach, separating safe set exploration and BO.
However, all of these methods suffer from local exploration~\crefp{section4-gp_no_jump}.
Some recent methods address disjoint safe regions. 
For example,~\cite{Sergeyev2020_safe_bo} considered regions separated by small gaps where the constraint functions briefly fall below, but remain near, the safety threshold.
\cite{dominik_baumann_gosafe_2021} proposed a global safe BO method for dynamical systems, assuming that unsafe regions can be approached slowly enough such that an intervention mechanism exists to stop the system in time.
Despite these advances, none of these approaches leverages safe transfer learning, which can allow for global exploration by utilizing prior knowledge from source tasks for a wide range of scenarios.

Transfer learning and multitask learning have gained increasing attention.
In particular, multi-output GP methods have been developed for multitask BO~\citep{Swersky13MTBO, poloczek2017MultiInfoSourceOptim}, sim-to-real transfer for BO~\citep{Marco2017sim2realRLviaBO}, and multitask AL~\citep{Zhang_etal16moAL_AAAI1611879}.
However, GPs face cubic time complexity with respect to the number of observations, a challenge that grows with multiple outputs.
In~\cite{Tighineanuetal22transferGPbo}, the authors assume a specific structure of the multi-output kernel, which allows to factorize the computation with an ensembling technique.
This eases the computational burden for transfer sequential learning.
In our paper, we propose a modularized safe transfer learning that avoids the cubic complexity.

\paragraph{Paper structure:}
The remaining of this paper is structured as follows.
We provide the setup and problem statement in~\cref{section-problem_statement}, background on GPs and safe AL in~\cref{section-background}. 
\cref{section4-gp_no_jump} discusses theoretical perspective of safe learning and demonstrate that safe learning approaches based on standard GPs suffer from local exploration.
\cref{section-our_transfer_method} elaborates our safe transfer learning approach and our modular computation scheme.
\cref{section-experiments} is the experimental study.
Finally, we conclude our paper in~\cref{section-conclusion}.

\section{Safe Transfer Active Learning Setup}\label{section-problem_statement}

Transfer Learning aims to transfer knowledge from previous, \textit{source}, systems to a new, \textit{target}, system. Usually, there exist a lot of data from one or more source systems and only few or no data from the target system. Safe Transfer Active Learning will exploit the knowledge from the source systems' data and allows for safe and active data collection on the target system.
Throughout this paper, we inspect regression problems.

\begin{table}[t]
\caption{Key Notation} \label{table-notation}
\begin{center}
\begin{tabular}{l|l}
\toprule
\textbf{Symbols}
&\textbf{Meaning}\\
\hline 
%$\bm{x} \in \mathcal{X}$ & inputs \\ 
%$y$ & outputs \\
%$\bm{z} = (z^1,...,z^J) \}$
%& safety observations for all $J$ constraints \\
%\hline
$N_{\text{init}}$ & number of initial target data points \\
$N_{\text{query}}$ & number of target data points added by AL \\
$N = N_{\text{init}},\ldots, N_{\text{init}}+N_{\text{query}}$ & number of total target points \\
$N_{\text{source}}$ & number of source data points 
\\
\hline
$\mathcal{D}_N=\{ \bm{x}_{1:N}, y_{1:N}, \bm{z}_{1:N} \}$
& dataset of the target task
\\
%, $N=N_{\text{init}}, ..., N_{\text{init}}+N_{\text{query}}$\\
%$\bm{z}_{1:{N}} = \{\bm{z}_n = (z_n^1,...,z_n^J) \}_{n=1}^{N}$
%& safety observations of all $J$ constraints jointly\\
$\mathcal{D}_{N_{\text{source}}}^{\text{source}}=\{\bm{x}_{s, 1:N_{\text{source}}}, y_{s, 1:N_{\text{source}}}, \bm{z}_{s, 1:N_{\text{source}}}\}$
&
dataset of the source task %$\{\bm{x}_{s, 1:N_{\text{source}}}, y_{s, 1:N_{\text{source}}}, \bm{z}_{s, 1:N_{\text{source}}}\}$
\\
$\bm{z} = (z^1,...,z^J)$
& safety variables of the target task
\\
$\bm{z}_s = (z_s^1,...,z_s^J)$
& safety variables of the source task
\\
\hline
$y=f(\bm{x})+\epsilon_f$ 
& model of the target observation $y$\\
$z^j=q^{j}(\bm{x})+\epsilon_{q^{j}}$
%, \forall j=1, ..., J$
& model of the target safety observation $z^j$ \\
%$z^{j} \geq T_j, \forall j=1, ..., J$
%& $j$-th safety condition, $T_j \in \mathbb{R}$ a threshold\\			
$y_s=f_s(\bm{x})+\epsilon_{f_s}$
& model of the source observation $y_s$ \\
$z_s^j=q^{j}_{s}(\bm{x})+\epsilon_{q^{j}_{s}}$
%, \forall j=1, ..., J$
& model of the source safety observation $z_s^j$%, assuming noise $\epsilon_{q^{j}_{s}} \sim \mathcal{N}\left( 0, \sigma_{q^{j}_{s}}^2 \right) $
\\
\hline
%$\bm{f}: \mathcal{X}\times \{ \text{task indices} \}\rightarrow \mathbb{R}$
%& $f_s$ and $f$ jointly as a multitask function\\
%$\bm{q}^{j}: \mathcal{X}\times \{ \text{task indices} \}\rightarrow \mathbb{R}$
%& $q^{j}_{s}$ and $q^{j}$ jointly as a multitask function\\
$f\sim \mathcal{GP}\left( \bm{0}, k_{f} \right)$
& single-output GP prior over target main function $f$
%, kernel $k_{\bm{f}}$ parameterized by $\bm{\theta}_{\bm{f}}= (\theta_{f_s}, \theta_{f}) $ 
\\
${q}^{j}\sim \mathcal{GP}\left( \bm{0}, k_{{q}^{j}} \right)$%, \forall j=1, ..., J$
& single-output GP prior over target safety function $q^j$
\\
$\bm{f}\sim \mathcal{GP}\left( \bm{0}, k_{\bm{f}} \right)$
& multi-output GP prior over main functions $f_s$ and $f$
%, kernel $k_{\bm{f}}$ parameterized by $\bm{\theta}_{\bm{f}}= (\theta_{f_s}, \theta_{f}) $ 
\\
$\bm{q}^{j}\sim \mathcal{GP}\left( \bm{0}, k_{\bm{q}^{j}} \right)$%, \forall j=1, ..., J$
& multi-output GP prior over safety functions $q_s^j$ and  $q^j$
%, kernel $k_{\bm{q}^{j}}$ parameterized by $\bm{\theta}_{\bm{q}^{j}}= (\theta_{ q^{j}_{s} }, \theta_{q^{j}}) $ 
\\
\bottomrule
\end{tabular}\\
\end{center}
\end{table}

\paragraph{Target and Safety -- Notation:}
%Throughout this paper, we inspect regression output and safety values.
Each $D$-dimensional input $\bm{x} \in \mathcal{X} \subseteq \mathbb{R}^D$ has a corresponding regression output $y \in \mathbb{R}$ and safety values jointly expressed as $\bm{z}=(z^1, ..., z^J) \in \mathbb{R}^J$, $J$ is the number of safety variables.
There are $J$ thresholds $T_j \in \mathbb{R}, j=1,...,J$, and an input $\bm{x}$ is safe if the corresponding safety values $z^{j} \geq T_j$ for all $j=1,...,J$.
It is assumed that the underlying functions of $y, z^{1},...,z^{J}$ are all unknown.

\paragraph{Source and Safety -- Notation:}
Similarly, there exist output and safety values of one or more source tasks, again from unknown underlying functions.
The source output value is denoted by $
y_{s} \in \mathbb{R}$ and source safety values by
$\bm{z}_{s}=(z_{s}^1,...,z_{s}^J) \in \mathbb{R}^{J}$, $s$ is the index of source task(s).
The source tasks are defined on the same domain $\mathcal{X}$.
The source and target tasks may have different numbers of constraint variables, but we can add trivial constraints (e.g. $1 \geq -\infty$) to any of the tasks in order to have the same number $J$.
Furthermore, the source data may or may not be measured with the same safety constraints as the target task.
For example, in a simulation-to-reality transfer~\citep{Marco2017sim2realRLviaBO}, the source dataset can be obtained unconstrained.

\paragraph{Datasets -- Notation:}
A dataset over the target task is denoted by $\mathcal{D}_{N} = \{\bm{x}_{1:{N}}, y_{1:{N}}, \bm{z}_{1:{N}}\}$, $\bm{x}_{1:{N}}=\{\bm{x}_1, ..., \bm{x}_{N} \} \subseteq \mathcal{X}$, $y_{1:{N}}=\{y_1, ..., y_{N}\} \subseteq \mathbb{R}$, safety observations $\bm{z}_{1:{N}} \coloneqq \{\bm{z}_n = (z_n^1,...,z_n^J) \}_{n=1}^{N} \subseteq \mathbb{R}^J$, and $N$ is the number of observed data.
In this paper, $N$ is not fixed, as we may actively add new labeled data.
We denote the source data by $\mathcal{D}_{N_{\text{source}}}^{\text{source}}=\{\bm{x}_{s, 1:N_{\text{source}}}, y_{s, 1:N_{\text{source}}}, \bm{z}_{s, 1:N_{\text{source}}}\} \subseteq \mathcal{X} \times \mathbb{R} \times \mathbb{R}^J$, $s$ is the index of source task and  $N_{\text{source}}$ is the number of all source data points.
In our main paper, we consider only one source task for simplicity, while~\cref{appendix-mogp_detail-mogp_more_source} provides formulation on more source tasks.
Please also see~\cref{table-notation} for a summary of our notation.

\paragraph{Safe Active Learning Procedure:}
%This paper focuses on safe AL problems~\citep{Schreiter2015,ZimmerNEURIPS2018_b197ffde,cyli2022}.
The goal of safe AL is to collect data actively and safely on the target system, such that the final dataset helps to model the regression output $y$ on the safe region of input space $\mathcal{X}$, i.e. subset of $\mathcal{X}$ corresponding to $z^{1}\geq T_1, ..., z^{J}\geq T_J$.

Concretely speaking, we are given a small amount of data on the target task, i.e. $\mathcal{D}_{N}$ where the initial size $N=N_{\text{init}}$ is small.
The initial data are typically given by a domain expert and are safe, i.e. for $\mathcal{D}_{N_{\text{init}}} = \{\bm{x}_{1:{N_{\text{init}}}}, y_{1:{N_{\text{init}}}}, \bm{z}_{1:{N_{\text{init}}}}\}, \forall n=1,...,N_{\text{init}}, \bm{z}_n=(z_n^1,...,z_n^J)$  satisfy the safety constraints $z_n^1 \geq T_1, ..., z_n^J \geq T_J$.

At each $N$, one seeks the next point $\bm{x}_* \in \mathcal{X}_{\text{pool}} \subseteq \mathcal{X}$ to be evaluated.
$\mathcal{X}_{\text{pool}} \subseteq \mathcal{X}$ is the search pool which can be the entire space $\mathcal{X}$ or a predefined subspace of $\mathcal{X}$, depending on the applications.
The evaluation is budget consuming and safety critical, and it will return a noisy $y_*$ and noisy safety values $\bm{z}_*$.
Ideally, we need to make sure that $\bm{z}_*=(z_*^{1},...,z_*^{J})$ respect the safety constraints $z_*^{j} \geq T_j$ for all $ j=1,...,J$ and that $y_*$ is informative for the modeling of target $y$. % up to the constraints.
As the safety outputs are unknown when an $\bm{x}_*$ is selected, guaranteeing safety is challenging.
Safe learning methods resort to allowing queries that are safe only with high probability~\citep{sui15safeopt,yanan_sui_stagewise_2018,ZimmerNEURIPS2018_b197ffde,cyli2022}.

Afterward, the labeled point is added to $\mathcal{D}_{N}$ (observed dataset becomes $\mathcal{D}_{N+1}$), and we proceed to the next iterations.
$N$ is initially $N_{\text{init}}$ and grows to $N_{\text{init}}+N_{\text{query}}$.
$N_{\text{query}}$ is the number of the learning iterations, i.e. the number of data points actively added.

\paragraph{Safe Transfer Active Learning Aim:}
In particular, this paper aims to build a new safe transfer AL, a safe AL algorithm with multi-output GPs, so that we leverage the information of the source data $\mathcal{D}_{N_{\text{source}}}^{\text{source}}$ to explore a larger safe area.
Our algorithm aims to
\begin{itemize}
    \item (i) collect as few (small $N_{\text{query}}$) data as possible for building an accurate regression model of $y$ (in the safe part of the input domain $\mathcal{X}$),
    %and the collected data need to be safe with high probability,
    \item (ii) collect the data in a safe way and hereby explore the safe region including its boundaries, % help us make accurate model of $y$ on the safe regions of $\mathcal{X}$, and
    \item (iii) in particular explore larger safe areas than benchmarks in a faster way. %the collected data help us model on large safe regions.
\end{itemize}

\section{Background: Gaussian Processes of Single-Task, Safe Active Learning}\label{section-background}

In this section, we introduce Gaussian Processes (GPs) and state-of-the-art safe Active Learning (safe AL).
GPs are the workhorse of safe AL in which they are routinely applied to select safe and informative data points~\citep{Schreiter2015,ZimmerNEURIPS2018_b197ffde,cyli2022}.

\subsection{Gaussian Processes (GPs)}\label{section-background-GP}

Suppose we aim to model the output $y$ and the safety observations $z^{1}, ..., z^{J}$ with GPs.
Here, we introduce the modeling scheme and the underlying assumptions.
The first assumption is that the data represent functional values blurred with i.i.d. Gaussian noises.

\begin{assumption}[Data: target task]\label{assump1-data_generation_process}
Assume $y = f(\bm{x}) + \epsilon_{f}$, where $\epsilon_{f} \sim \mathcal{N}\left( 0, \sigma_{f}^2 \right)$, for our target observations.
We further assume that $z^j = q^{j}(\bm{x}) + \epsilon_{ q^{j} }$, where $\epsilon_{ q^{j} } \sim \mathcal{N}\left( 0, \sigma_{q^{j}}^2 \right)$, and $j=1,\ldots, J$ indexes the safety constraints.
All of the noise variances $\{\sigma_f^2, \sigma^2_{q^1}, \ldots, \sigma^2_{q^J} \}$ are positive.
\end{assumption}
We then place a GP assumption on each of the underlying functions $f, q^{1}, ..., q^{J}$.
A GP is a stochastic process defined by a mean and a kernel function~\citep{GPbook, KanHenSejSri18, KernelBook}.
In this work, we set the mean to zero — a common practice, as normalized data typically justifies this assumption.
The kernel function, $\mathcal{X} \times \mathcal{X} \rightarrow \mathbb{R}$, specifies the covariance of function values at different input points.
Without prior knowledge of the data, we make the standard assumption that the governing kernels are stationary.
The GP assumption is then formulated as the following.
\begin{assumption}[Model: single-task]\label{assump2-zero_mean_stationary_GP}
For each function, $g \in \{ f, q^{1}, ..., q^{J} \}$, we assume that $g \sim \mathcal{GP}(0, k_{g})$ with a stationary kernel with bounded variance, $k_g(\bm{x}, \bm{x}') \coloneqq k_g(\bm{x} - \bm{x}') \leq 1$.
\end{assumption}
Bounding the kernels by one provides advantages in theoretical analysis~\citep{Srinivas_2012} and is not restrictive since the data is usually normalized to unit variance.

\cref{assump1-data_generation_process} and~\cref{assump2-zero_mean_stationary_GP} provide the predictive distribution of the functions $f, q^{1}, ..., q^{J}$.
We write down the distribution for the function $f$ at a test point $\bm{x}_*$: 	
\begin{align}
\label{eq_GP_pred}
    p\left( f(\bm{x}_{*}) | \bm{x}_{1:N}, y_{1:N} \right) = \mathcal{N}\left( \mu_{f, N}(\bm{x}_{*}) , \sigma_{f, N}^2(\bm{x}_{*}) \right),
\end{align}
where 
\begin{align}
\begin{split}\label{eqn1-GP_posterior}
\mu_{f, N}(\bm{x}_{*})
&= k_{f}(\bm{x}_{1:N}, \bm{x}_{*})^T \left( \bm{K}_f + \sigma_{f}^2 I \right)^{-1} y_{1:N}, \\
\sigma_{f,N}^2(\bm{x}_{*})
&= k_f(\bm{x}_{*}, \bm{x}_{*}) - k_{f}(\bm{x}_{1:N}, \bm{x}_{*})^T \left( \bm{K}_f + \sigma_{f}^2 I \right)^{-1} k_{f}(\bm{x}_{1:N}, \bm{x}_{*}).
	\end{split}
	\end{align}
%Here, $k_{f}: \mathcal{X} \times \mathcal{X} \rightarrow \mathbb{R}$ denotes the kernel function. 
We use the notation $k_{f}(\bm{x}_{1:N}, \bm{x}_{*})=\left( k_{f}(\bm{x}_1, \bm{x}_*), ..., k_{f}(\bm{x}_N, \bm{x}_*) \right)\in \mathbb{R}^{N \times 1}$ to denote the kernel vector between the training points $\bm{x}_{1:N}$ and the test point $\bm{x}_*$.
The kernel matrix $\bm{K}_{f} \in \mathbb{R}^{N\times N}$ contains the covariances between the training points $\bm{x}_{1:N}$ with $\left[ \bm{K}_{f} \right]_{m,n} = k_{f}(\bm{x}_{m}, \bm{x}_n), m,n=1,...,N$.

Typically, $k_f$ is parameterized and can be jointly fitted with the noise variance $\sigma_f^2$. Common fitting techniques involve computing the marginal likelihood, $  \mathcal{N}\left( y_{1:N}|\bm{0}, \bm{K}_f + \sigma_f^2 I \right)$, where the
the runtime complexity is $\mathcal{O}\left( N^3 \right)$, dominated by the inversion of the Gram matrix $\left( \bm{K}_f + \sigma_{f}^2 I \right)^{-1}$.

The predictive distributions of the safety functions $q^{1}, ..., q^{J}$ can be obtained by replacing $f$ with $q^{1}, ..., q^{J}$ and the outputs $y_{1:N}$ with $z_{1:N}^j, j=1,...,J$ in~\cref{eq_GP_pred,eqn1-GP_posterior}.
 Similarly, the log-likelihood can be maximized for each $q^{j}$ by jointly learning $k_{q^j}$ and $\sigma^2_{q^{j}}$ in the same manner, $j=1,\ldots, J$.

\begin{remark}
In our paper, all safety measurements $z^{1},...,z^{J}$ are modeled independently.
If the variables are not independent, our analysis and arguments still apply, as the dependent constraints can be grouped, and the problem reduces back to the independent case.
%because one can consider the dependent constraints jointly and resort the problem back to an independent case.
%For example, among $J$ constraints, the first three $q^{1}, q^{2}, q^{3}$ are dependent, while the corresponding source $q^{1}_{s}, q^{2}_{s}, q^{3}_{s}$ may or may not be dependent.
%Then we may concatenate the input again with constraint index $j$ and consider $q^{123}, q^{4}, ..., q^{J}$ and $q_{s}^{123}, q_{s}^{4}, ...$, where joint functions $q^{123}(\cdot, j=1,2,3)\coloneqq q^{j}(\cdot), q^{123}_{s}(\cdot, j=1,2,3)\coloneqq q^{j}_{s}(\cdot), q^{j>3}(\cdot, j)\coloneqq q^{j}(\cdot), q^{j>3}_{s}(\cdot, j)\coloneqq q^{j}_{s}(\cdot)$.
%The problem is then resorted back to $J-2$ independent safety constraints.
\end{remark}

\subsection{Safe Active Learning (Safe AL)}\label{section-background_sal}

\begin{algorithm}[b]
\caption{Safe AL}
\label{alg-SL}
\begin{algorithmic}[1]
\Require \cref{assump1-data_generation_process},~\cref{assump2-zero_mean_stationary_GP}, $\mathcal{D}_{N_{\text{init}}}, \mathcal{X}_{\text{pool}}$, $\beta$ or $\alpha$, $N_{\text{query}}$, thresholds $T_1,...,T_J$
\For{$N=N_{\text{init}}, ..., N_{\text{init}} + N_{\text{query}} - 1$}
\State Fit GPs $f, q_1, ..., q_J$ with $\mathcal{D}_{N}$
\State $\mathcal{S}_N \gets \cap_{j=1}^{J} \{ \bm{x} \in \mathcal{X}_{\text{pool}} | \mu_{q^{j}, N}(\bm{x}) - \beta^{1/2} \sigma_{q^{j}, N}(\bm{x}) \geq T_j\}$~\crefp{eqn1-GP_posterior,eqn3-safe_set}
\State $\bm{x}_*\gets \text{argmax}_{\bm{x} \in \mathcal{S}_N } \ a(\bm{x} ),
\ a(\bm{x} )=
\sum_{g \in \{f, q^{1}, ..., q^{J}\} } H_g \left[ \bm{x}|\mathcal{D}_N \right]$
\State Query $\bm{x}_*$ to get evaluations $y_*$ and $\bm{z}_*$
\State $\mathcal{D}_{N+1} \gets \mathcal{D}_{N} \cup \{ \bm{x}_*, y_*, \bm{z}_*  \}, \mathcal{X}_{\text{pool}} \gets \mathcal{X}_{\text{pool}} \setminus \{ \bm{x}_* \}$
\EndFor
\State \textbf{Return} $\mathcal{D}_{N_{\text{init}} + N_{\text{query}}}$, trained models $f, q_1, ..., q_J$
\end{algorithmic}
\end{algorithm}

In this section, we introduce state-of-the-art safe AL.
The state-of-the-art safe AL cannot exploit knowledge in form of source data.
Although source data is not considered in this section, we will still write \textit{target task} to make the distinction between the two tasks clear.

%As no source task is considered, we could skip \textit{target} and just write task.
%We will, however, still write \textit{target task} to indicate on which task we will apply safe AL as a benchmark later on.

Safe AL~\citep{Schreiter2015,ZimmerNEURIPS2018_b197ffde,cyli2022} aims to select data actively and safely to learn about a target task.
At a given number of available target data $N$, the goal is to select $\bm{x}_* \in \mathcal{X}_{\text{pool}} \subseteq \mathcal{X}$, that gives us a safe $\bm{z}_*=(z_*^{1},...,z_*^{J})$ and informative $y_*$ (informative in the context of modeling).

The key of safe AL is the safety and data selection criteria.
%One need to infer $y$ and the safety conditions on unseen data.
Commonly, these criteria employ GPs due to the well calibrated predictive uncertainty~\citep{Schreiter2015,ZimmerNEURIPS2018_b197ffde,cyli2022}. It is worth highlighting a closely related field, safe Bayesian optimization (BO)~\citep{sui15safeopt,yanan_sui_stagewise_2018,berkenkamp2020bayesian}, which follows a similar procedure except that the goal is to search for an optimum $y_*$ subject to safety constraints $z_*^1 \geq T_1,...,z_*^J \geq T_J$.

\paragraph{Safety Condition:} The safety variables are modeled by GP functions $q^{1}, ..., q^{J}$~\crefp{section-background-GP}, and the core is to compare the safety confidence bounds~\crefp{eqn1-GP_posterior} with the thresholds $T_1,...,T_J \in \mathbb{R}$.
At each iteration $N$, we can compute the safety probability $p\left( q^{j}(\bm{x}) | \bm{x}_{1:N}, z_{1:N}^{j} \right) = \mathcal{N}\left( \mu_{q^{j}, N}(\bm{x}) , \sigma_{q^{j}, N}^2(\bm{x}) \right)$, for each safety constraint $j=1,...,J$.
\cite{sui15safeopt} defines the safe set $\mathcal{S}_N \subseteq \mathcal{X}_{\text{pool}}$ as
\begin{align}\label{eqn3-safe_set}
\begin{split}
\mathcal{S}_N = \cap_{j=1}^{J} \{ \bm{x} \in \mathcal{X}_{\text{pool}} | \mu_{q^{j}, N}(\bm{x}) - \beta^{1/2} \sigma_{q^{j}, N}(\bm{x}) \geq T_j\},
\end{split}
\end{align}
where $\beta \in \mathbb{R}^+$ is a parameter for probabilistic tolerance control.
This definition is equivalent to $\forall \bm{x} \in \mathcal{S}_N, p\left( q^{1}(\bm{x}) \geq T_1, ..., q^{J}(\bm{x}) \geq T_J \right) \geq \left( 1 - \alpha \right)^J$ when $\alpha = 1 - \Phi(\beta^{1/2})$ ($\alpha \in [0, 1]$) where $\Phi$ is the standard Gaussian cumulative distribution function (CDF)~\citep{Schreiter2015, ZimmerNEURIPS2018_b197ffde, cyli2022}.
Note that $\alpha > 0.5$ corresponds to $\beta^{1/2} < 0$ (assume this exists) which is usually not considered because safe learning aims for high safety confidence while $\alpha > 0.5$ indicates a safety confidence of at most $50\%$ per constraint - so at most a random guess.

\paragraph{Information Criterion:} Safe AL queries a new point by mapping safe candidate inputs to acquisition scores:
\begin{align}\label{eqn4-acquisition_optimization}
\begin{split}
\bm{x}_*
&= \text{argmax}_{\bm{x} \in \mathcal{S}_N } a\left( \bm{x} \right),
\end{split}
\end{align}
where $a(\cdot)$ is an acquisition function.
Notice here that $a(\cdot)$ and $\mathcal{S}_{N}$ both depend on the observed dataset $\mathcal{D}_N$.
In AL problems, a prominent acquisition function is the predictive entropy~\citep{Schreiter2015, ZimmerNEURIPS2018_b197ffde, cyli2022}: $a(\bm{x}) = H_f \left[ \bm{x}|\mathcal{D}_N \right] = \frac{1}{2} \log\left( 2 \pi e \sigma^2_{f, N}(\bm{x}) \right)$, where $\sigma^2_{f, N}$ is defined in~\cref{eqn1-GP_posterior}.
% where $f$ is again a GP function \citep{Schreiter2015, ZimmerNEURIPS2018_b197ffde, cyli2022}.
To accelerate the exploration of the safety functions,~\cite{berkenkamp2020bayesian} incorporate the information of the safety functions by using $a(\bm{x})=\text{max}_{ g\in\{f, q^{1},...,q^{J}\} } \{ H_{g}\left[ \bm{x} | \mathcal{D}_N \right] \}$.
Our acquisition function is built upon this and is written in~\cref{alg-SL}.

Please note again the close connection to BO: it is possible to exchange the acquisition function by the SafeOpt criteria if one wants to address safe BO problems~\citep{sui15safeopt, berkenkamp2020bayesian, rothfuss_meta-learning_2022}).

\paragraph{Constrained Acquisition Optimization:}
Solving~\cref{eqn4-acquisition_optimization} is challenging.
In the literature~\citep{Schreiter2015, cyli2022, sui15safeopt, berkenkamp2020bayesian}, this is solved on a discrete pool with finite elements, i.e. $N_{\text{pool}} \coloneqq |\mathcal{X}_{\text{pool}}| < \infty$.
One would compute the GP posteriors on the entire pool $\mathcal{X}_{\text{pool}}$ to determine the safe set, then optimize the acquisition scores over the safe set.
In our paper, we inherit this finite discrete pool setting.

\paragraph{Time Complexity:}
The complexity comprises $\mathcal{O}\left( N^3 \right)$ for GP training and $\mathcal{O}\left( N_{\text{pool}} N^2 \right)$ for GP inference, assuming the Gram matrices, $\left( \bm{K}_g + \sigma_{g}^2 I \right)^{-1}, \forall g\in \{f, q^1,...,q^J\}$, are already computed during the training~\crefp{eqn1-GP_posterior}.
Importantly, GP inference is only performed once per query, whereas GP training (or more specifically, its most computationally expensive step, the matrix inversion) is repeated multiple times during parameter learning.
As a consequence, the size of the discretized pool $N_{\text{pool}}$ can be much larger than the training dataset $N$, e.g. $\mathcal{X}_{\text{pool}}$ can include up to tens or even hundreds of thousands of samples.
%Training a GP takes $\mathcal{O}\left( N^3 \right)$ for \textbf{multiple} times.
%This factor depends on the optimizer, the number of kernel parameters, and numerical stability.
%Often, the complexity of training is significantly larger than making a GP inference, despite $N_{\text{pool}}$ much larger than $N$.

The whole learning process is summarized in~\cref{alg-SL}.
In the next section, we provide theoretical analysis to demonstrate the presence of local exploration in safe learning approaches.

\section{Safe Learning Solely on Target Task: Local Exploration}\label{section4-gp_no_jump}

Before we introduce our safe transfer AL approach in the next~\cref{section-our_transfer_method}, we analyze a shortcoming of the standard safe AL~\crefp{alg-SL}. We quantify the upper bound for an explorable safe region, and prove that safe AL is limited to local exploration within the given bound.
Note, that the analysis will not involve the acquisition function, and therefore the result applies not only to  safe AL but also to GP based safe BO settings with safe set as in~\cref{eqn3-safe_set}.

Given observations $\mathcal{D}_N$, we would like to know, until how far into the input space the safety confidence is sufficiently high.

\paragraph{Correlation weakened with increasing distance:}
The conventional safe AL~\crefp{alg-SL} builds models based on a standard GP assumption (\cref{assump1-data_generation_process},~\cref{assump2-zero_mean_stationary_GP}), and then the explorable region is obtained by quantifying the safety confidence, conditioned on observed data $\mathcal{D}_N$~\crefp{eqn3-safe_set}.
The safety confidence is calculated from the GP predictive distributions~\crefp{eqn1-GP_posterior}, and it thus depends on the kernel to correlate input points of various locations.
Commonly used stationary kernels measure the distance between a pair of points, while the actual output values do not matter (for two points $\bm{x}, \bm{x}' \in \mathcal{X}$, $\| \bm{x} - \bm{x}' \|$ determines the covariance).
These kernels have the property that closer points have higher kernel values, indicating stronger correlation, while distant points result in small kernel values.
We first formulate this property as the following.
\begin{restatable}{definition}{kernelDeltaR}
\label{property1-kernel_convergence}
We call a kernel $k$ a kernel with \textit{correlation weakened by distance} if 
$k: \mathcal{X} \times \mathcal{X} \rightarrow \mathbb{R}$ fulfills the following property: $\forall \delta > 0$, $\exists r > 0$ s.t. $\|\bm{x} - \bm{x}'\| \geq r \Rightarrow k(\bm{x}, \bm{x}') \leq \delta$ under $L2$ norm.
\end{restatable}
This definition will later help us quantify the upper bound of an explorable region. 
We provide expressions of popular stationary kernels (RBF kernel and Mat{\'e}rn kernels), as well as their relations between input distance $r$ and covariance $\delta$ in the~\cref{appendix-r_delta_relation_for_common_kernels}.
%\textcolor{red}{is ther any kernel who does not fulfill the definition? could also be named in the appendix. }

\paragraph{Low correlation leading to small safety probability:} With~\cref{property1-kernel_convergence}, we can now derive a theorem measuring the explorable region.
The main idea is:
when an unlabeled point $\bm{x}_*$ is far away from the observed inputs, the value of $\delta$ can become very small (i.e. small covariance measured by the kernel).
Thus, for each $j=1,...,J$, the model weakly correlates $q^{j}(\bm{x}_*)$ to the observations.
As a result, the predictive mean is close to zero (GP prior) and the predictive uncertainty is large, both of which imply that the method has small safety confidence,\footnote{A small safety confidence indicates $\exists j=1,...,J$, $p\left( (q^{j}(\bm{x}_*) \geq T_j) | \bm{x}_{1:N}, z_{1:N}^j \right)$ is not high enough.
} at least if $\forall j=1,...,J, q^{j} \geq T_j$ is not a trivial condition, e.g. a trivial condition would be if a function $q^{j}$ has values majorly distributed in $[-1, 1]$ but thresholded at $T_j = -2$. %w%hich is equivalent to an unconstrained problem requiring no safe AL.

\begin{restatable}[Local exploration of single-output GPs]{theorem}{thmGPnoJump}
%\begin{theorem}[Local exploration of single-output GPs]
\label{thm-BoundSafety}
We are given $\bm{x}_{1:N} \subseteq \mathcal{X}$.
For any safety constraint indexed by $j=1,...,J$, let $z_{1:N}^j \coloneqq (z_1^j, ..., z_N^j )$ be the observed noisy safety values and let $\|(z_1^j, ..., z_N^j)\|\leq \sqrt{N}$.
The safety value $z^{j}=q^{j}(\bm{x}) + \epsilon_{q^{j}}$ satisfies the GP assumptions (\cref{assump1-data_generation_process},~\cref{assump2-zero_mean_stationary_GP}): $q^{j} \sim \mathcal{GP}(0, k_{q^{j}}),
k_{q^{j}}(\cdot, \cdot) \leq 1,
\epsilon_{q^{j}}~\sim \mathcal{N}\left( 0, \sigma_{q^{j}}^2 \right)$.
The kernel $k_{q^{j}}$ is a kernel with correlation weakened by distance~\crefp{property1-kernel_convergence}.
Denote $k_{scale}^j \coloneqq max \ k_{q^{j}}(\cdot, \cdot)$.
Then $\forall \delta \in (0, \sqrt{k_{scale}^j} \sigma_{q^{j}} / \sqrt{N}), \exists r>0$ s.t. $\forall \bm{x}_* \in \mathcal{X}$ that fulfill $ \text{min}_{ \bm{x}_i \in \bm{x}_{1:N} } \|\bm{x}_* - \bm{x}_i\| \geq r$,
the probability thresholded on a constant $T_j$ is bounded by $p\left( (q^{j}(\bm{x}_*) \geq T_j) | \bm{x}_{1:N}, z_{1:N}^j \right) \leq \Phi\left( \frac{N\delta/\sigma_{q^{j}}^2 - T_j}{\sqrt{k_{scale}^j-(\sqrt{N}\delta/\sigma_{q^{j}})^2}} \right)$.
$\Phi$ is the CDF of standard Gaussian.
%\end{theorem}	
\end{restatable}

%This theorem means that for appropriate choices of $\delta$, points $x^*$ cannot be reached if they are more than $r$ away from old points and the safety requirement is higher than $\Phi_{safe,j}$. 
%\textcolor{red}{was there a mistake here? I don't understand "and the safety requirement..."}
%
%\textcolor{red}{also, aren't this part to some extend redundant to the paragraph after the next paragraph: "Our theorem provides the maximum safety probability..."?}

To prove this theorem, we need~\cref{eqn1-GP_posterior} to compute the safety likelihood, and then we use the eigenvalues of the kernel Gram matrix together with~\cref{property1-kernel_convergence} to derive the final bound (see~\cref{appendix-gp_no_jump} for the detailed proof).

Our theorem provides the maximum safety probability of a point as a function of its distance to the observed data in $\mathcal{X}$.
The safe set tolerance parameter $\beta$ or $\alpha$~\crefp{eqn3-safe_set} can be used to compute the covariance bound $\delta$.
For example, when $J=1$ and $\beta^{1/2}=2$, which means $p(q(\bm{x}) \geq T) \geq \Phi(2)$ is safe~\crefp{eqn3-safe_set}, we choose a $\delta$ such that $\Phi\left( \frac{N\delta/\sigma_{q}^2 - T}{\sqrt{k_{scale}-(\sqrt{N}\delta/\sigma_{q})^2}} \right) \leq \Phi(2)$ ($j$ omitted when $J=1$).
Such a $\delta$ exists in all situation of our interest, as we will soon discuss.
Given a $\delta$, we can then determine a corresponding radius $r$ (see e.g.~\cref{appendix-r_delta_relation_for_common_kernels}).
Interpreting $r$ as the radius around the observed data, the safety confidence outside this region always remains low.
Since safety confidence decides the explorable regions~\crefp{eqn3-safe_set,eqn4-acquisition_optimization}, this theorem measures an upper bound of explorable safe area.
The upper bound is given for one safety constraint, and we can see from~\cref{eqn3-safe_set} that the final bound of safety confidence is the product of the $\Phi$ term over different $j$.
In other words, the more safety constraints, the smaller the explorable regions may be, which is intuitive.

Notice that $\| (z_1^j, ..., z_N^j) \| \leq \sqrt{N}$ ($\forall j=1,...,J$) is not very restrictive because an unit-variance dataset has $\| (z_1^j, ..., z_N^j) \| = \sqrt{N}$.
Note further that $\delta \in \left(0, \sqrt{k_{scale}^j} \sigma_{q^{j}} / \sqrt{N}\right)$ implies $k_{scale}^j > (\sqrt{N}\delta/\sigma_{q^{j}})^2$, which means the bound
$\Phi\left( \frac{N\delta/\sigma_{q^{j}}^2 - T_j}{\sqrt{k_{scale}^j-(\sqrt{N}\delta/\sigma_{q^{j}})^2}} \right)$
is always valid.

This theorem indicates that a standard GP with commonly used kernels explores only neighboring regions of the initial $\bm{x}_{1:N}$.
%In practice, we consider safe AL on a discrete pool $\mathcal{X}_{\text{pool}} \subseteq \mathcal{X}$, which means the GP explores only neighboring discretized points.
The theorem is independent of the acquisition functions, and thus the local exploration problems present in all safe learning methods based on standard GPs.

\paragraph{Existence of $\delta$ for common safe learning situations:}

We would like to illustrate an example of using our theorem to compute an explorable bound.
Before that, we will make a statement relating the safety level $\beta$ to the quantity $\delta$ used in Theorem \ref{thm-BoundSafety}. This shows that a $\delta$ and, therefore, local exploration is present in all but some (at least) uncommon scenarios, which are in fact out of interest for the sake of safe exploration.
%In such pathological scenarios, we may not be able to find $\delta$, given the thresholds and a chosen tolerance parameter $\beta$~\crefp{eqn3-safe_set}.

%We consider for cases of combination of values for our safety level $\beta$ and the safety threshold $T_j$:
%\begin{enumerate}
%    \item $\beta^{\frac{1}{2}} > 0$, $T_j \geq 0$:\\
%    This is covered by the following corollary.
%    \item $\beta^{\frac{1}{2}} > 0$, $T_j < 0$, $\beta^{1/2} > \frac{|T_j|}{ \sqrt{k_{scale}^j} }$:\\
%    This is also covered by our corollary.
%    \item $\beta^{\frac{1}{2}} > 0$, $T_j < 0$, $\beta^{1/2} < \frac{|T_j|}{ \sqrt{k_{scale}^j} }|$:\\
%    Written as $ 0 < \beta \leq |T_j/j^j_{scale} |$ one sees that this describes a situation of a safety threshold $T_j$ that has a rather high negative value compared to the rather small safety requirement $\beta$. This means a soft constraint with low safety requirements. If this situation occurs a GP with normalized data would be safe in the extrapolation area. \textcolor{red}{[please check and possibly explain]}.  This is a scenario that is not of interest for safe learning as safe learning relies on data (so interpolation) to determine safe areas. 
%    \item $\beta^{\frac{1}{2}} \leq 0$: \\
%    We can rule $\beta \leq 0$ as uncommon for safe learning because $\beta \leq 0$ leads to $\Phi(\beta) \leq 0.5$. This would mean a scenario that requests a safety of at most $50\%$ - so at most a random guess. This can be ruled uncommon for safe learning with usually high safety requirements.
%\end{enumerate}

\begin{restatable}[Existence of $\delta$]{corollary}{betaOfT}
\label{thm-beta_of_threshold}
We are given the assumptions in~\cref{thm-BoundSafety}.
For each $j=1,...,J$, if either
(1) $T_j \geq 0, \beta^{1/2} > 0$ or
(2) $T_j < 0, \beta^{1/2} > \frac{|T_j|}{ \sqrt{k_{scale}^j} }$,
then $\exists \delta \in (0, \sqrt{k_{scale}^j} \sigma_{q^{j}} / \sqrt{N})$ s.t. $\Phi\left( \frac{N\delta/\sigma_{q^{j}}^2 - T_j}{\sqrt{k_{scale}^j-(\sqrt{N}\delta/\sigma_{q^{j}})^2}} \right) \leq \Phi(\beta^{1/2})$.
%We are given a $j\in\{1,...,J\}$ and a threshold $T_j \in \mathbb{R}$.
%The safety value $z^{j}=q^{j}(\bm{x}) + \epsilon_{q^{j}}$ satisfies the GP assumptions (\cref{assump1-data_generation_process},~\cref{assump2-zero_mean_stationary_GP}): $q^{j} \sim \mathcal{GP}(0, k_{q^{j}}),
%k_{q^{j}}(\cdot, \cdot) \leq 1,
%\epsilon_{q^{j}}~\sim \mathcal{N}\left( 0, \sigma_{q^{j}}^2 \right)$.
%The kernel $k_{q^{j}}$ is a kernel with correlation weakened by distance~\crefp{property1-kernel_convergence}.
%Denote $k_{scale}^j \coloneqq max \ k_{q^{j}}(\cdot, \cdot)$.
%$N$ is any positive integer.
%If either (1) $T_j \geq 0, \beta^{1/2} > 0$ or (2) $T_j < 0, \beta^{1/2} > -T_j / k_{scale}^j$, then there always exists a $\delta \in (0, \sqrt{k_{scale}^j} \sigma_{q^{j}} / \sqrt{N})$ s.t. $\Phi\left( \frac{N\delta/\sigma_{q^{j}}^2 - T_j}{\sqrt{k_{scale}^j-(\sqrt{N}\delta/\sigma_{q^{j}})^2}} \right) \leq \Phi(\beta^{1/2})$.
%$\Phi$ is the standard Gaussian CDF.
\end{restatable}

This corollary can be proved by inspecting the boundary of each constant (detailed in~\cref{appendix-gp_no_jump}).

The key insight is that, a $\delta$ in~\cref{thm-BoundSafety}, which bounds the safety likelihood, always exists for common selection of safety level $\beta^{1/2}$.
There are two scenarios considered in~\cref{thm-beta_of_threshold}.
The first scenario, $T_j \geq 0, \beta^{1/2}>0$ is common because $\beta^{1/2} > 0$ is always desired for safe exploration and stricter safety thresholds $T_1 \geq 0,...,T_J \geq 0$ may also occur.
In the second scenario, the thresholds are softer, i.e. one or more of the thresholds $T_1,...,T_J$ are smaller than zero.
It turns out that $\beta^{1/2} > |T_j| / \sqrt{k_{scale}^j}, j=1,...,J $ is desired as well for safe learning.
Consider $|T_j| / \sqrt{k_{scale}^j} \geq \beta^{\frac{1}{2}} > 0$ for a $j\in\{1,...,J\}$, which is the scenario not fulfilling the condition.
We focus on normalized variable $z_j$, where the underlying function is modeled by a GP $q^{j} \sim \mathcal{GP}(0, k_{q^{j}})$.
When this model extrapolates in regions where data are absent, the inference is highly based on the prior $q^{j}(\bm{x}) \sim \mathcal{N}(0, k_{scale}^{j})$.
The safe set considers $p(q^{j}(\bm{x})\geq T_j) \geq \Phi(\beta^{1/2})$ as a safety condition on the $j$-th constraint, but the prior indicates $\Phi(-T_j / \sqrt{k_{scale}^{j}} ) \geq \Phi(\beta^{1/2})$ which becomes a trivial condition when $|T_j| / \sqrt{k_{scale}^j} \geq \beta^{\frac{1}{2}} > 0$.
In other words, any input $\bm{x}$ has a safe prior unless the data disagree.
This is a scenario that is not of interest for safe learning.

Therefore, for all common selection of safety level $\beta^{1/2}$,~\cref{thm-beta_of_threshold} implies that we can find a $\delta$ and apply~\cref{thm-BoundSafety} to quantify the upper bound of explorable safe set, which shows the presence of local exploration.
%However, unless those constraints are less safety critical, one would still choose a higher $\beta^{1/2}$ (lower unsafe tolerance) to avoid overconfidence on the safety condition.
%For example, a normalized problem has $k_{scale}^j=1$, for a $j\in \{1,...,J\}$, and has safety values distributed majorly in $[-2, 2]$, as how a normal distribution is.
%If $T_j=-1.5$, which is a soft threshold, then we see that a common choice $\beta^{1/2}=2$~\citep{sui15safeopt,berkenkamp2020bayesian} satisfies the condition $\beta^{1/2} > 1.5$, and we can thus find a $\delta$ and apply~\cref{thm-BoundSafety} to quantify the explorable bound of a safe set.

\begin{figure}[t]
\vskip 0.2in
\begin{center}
\centerline{\includegraphics[width=\columnwidth]{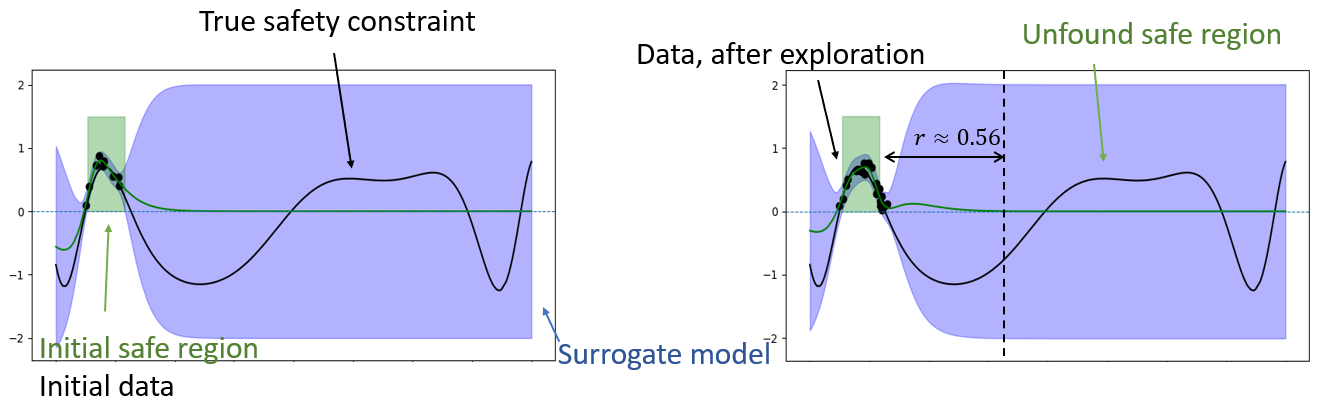}}
\caption{
A safety function (shown in black) with two safe regions above threshold zero. 
Left graphics: 
Based on the initial data within one of the safe regions, a GP surrogate is trained. The blue line represents the mean prediction, while the blue shaded area indicates the uncertainty (e.g., confidence interval) around the mean. 
The green area indicates the learned safe area.
Right graphics: 
After exploration, more points are sampled within the first safe region. However, the gap to the second safe region exceeds 
$r$, preventing the discovery of the second region, rendering the learned safe area almost unchanged.
The true safety function used here is $q(x)= \sin\left( 10 x^3 - 5x - 10 \right) + \frac{1}{3} x^2 - \frac{1}{2}$. The observations are with noise drawn from $\mathcal{N}(0, 0.1^2)$.}
\label{figure2-local_safe_exploration}
\end{center}
\vskip -0.2in
\end{figure}

\paragraph{Illustrating the theoretical result:} In the following, we plug exact numbers into~\cref{thm-BoundSafety} for an illustration.
	
\begin{example}\label{example-explorable_bound}
We consider a one-dimensional toy dataset visualized in~\cref{figure2-local_safe_exploration}.
Assume $N=10$, $\sigma_{q}^2 = 0.01$ and constraint $T=0$.
We omit safety constraint index $j$, since $J=1$ in this case.
In this example, the generated data have $\|z_{1:N}\| \leq \sqrt{10}$.
$\sigma_{q} / \sqrt{N}$ is roughly $0.0316$.
We train an unit-variance ($k_{scale}=max \ k_{q}(\cdot, \cdot)=1$, theorem requires $\delta < 0.0316$) Mat{\'e}rn-5/2 kernel on this example, resulting in a learned lengthscale of around $0.1256$.
The Mat{\'e}rn-5/2 kernel is a kernel with correlation weakened by distance~\crefp{property1-kernel_convergence}.
In particular, $r=4.485*0.1256=0.563316 \Rightarrow \delta \leq 0.002$~\crefp{appendix-r_delta_relation_for_common_kernels}, noticing that $\delta = 0.002 \Rightarrow \Phi\left( \frac{N\delta/\sigma_{q}^2 - T}{\sqrt{1-(\sqrt{N}\delta/\sigma_{q})^2}} \right) \approx \Phi(2)$.
Consider $\beta^{1/2}=2$, then it is safe only when the safety likelihood is above $\Phi(2)$.
We can thus know from~\cref{thm-BoundSafety} that safe regions that are $0.563316$ further from the observed ones are always identified as unsafe and is not explorable.
In~\cref{figure2-local_safe_exploration}, the two safe regions are more than $0.7$ distant from each other, indicating that the right safe region is never explored by conventional safe learning methods.
\end{example}

\paragraph{GP might even explore less in practice:}
Our probability bound $\Phi\left( \frac{N\delta/\sigma_{q^{j}}^2 - T_j}{\sqrt{k_{scale}^j-(\sqrt{N}\delta/\sigma_{q^{j}})^2}} \right)$ (for each $j=1,...,J$) is the worst case obtained with very mild assumptions.
	%For example, this upper bound is larger when we have more observations (larger $N$) or when it is less noisy (smaller $\sigma_{q}^2$).
	Empirically, the explorable regions found by GP models are smaller (see~\cref{figure2-local_safe_exploration}).

\paragraph{Transfer learning may overcome the local exploration:}
We extended the~\cref{example-explorable_bound} to compare the standard GP model against a transfer task GP which will be introduced in the next~\cref{section-our_transfer_method}.
In~\cref{figure3-gp_local_vs_lmc_global}, a linear model of corregionalization is trained (a kind of multitask GP,~\cite{alvarez2012kernels}).
On the right region, which is beyond the explorable bound of a standard single-task GP, the transfer task GP incorporates the source data allowing to build high safety confidence.
As a result, the right region can be included into the explorable safe set (detailed in the following~\cref{section-our_transfer_method}).
We will also confirm in our experiments in Section~\ref{section-experiments} that guidance from source data enables our new safe transfer AL framework to explore beyond the immediate neighborhood of the target points $\bm{x}_{1:N}$.
%, taking the source data into consideration.

To summarize, we see in this section that the safe set of standard GPs~\crefp{eqn3-safe_set} is limited to a local region.
In the next section, we transfer knowledge from the source data to expand the exploration beyond the seed dataset of the target task.
%From a theoretical perspective, the same covariance bound $\delta$ in~\cref{thm-BoundSafety} still results in a radius $r$, but the explorable regions center around observations and we take the source data into consideration.

\begin{figure}[t]
\vskip 0.2in
\begin{center}
\centerline{\includegraphics[width=\columnwidth]{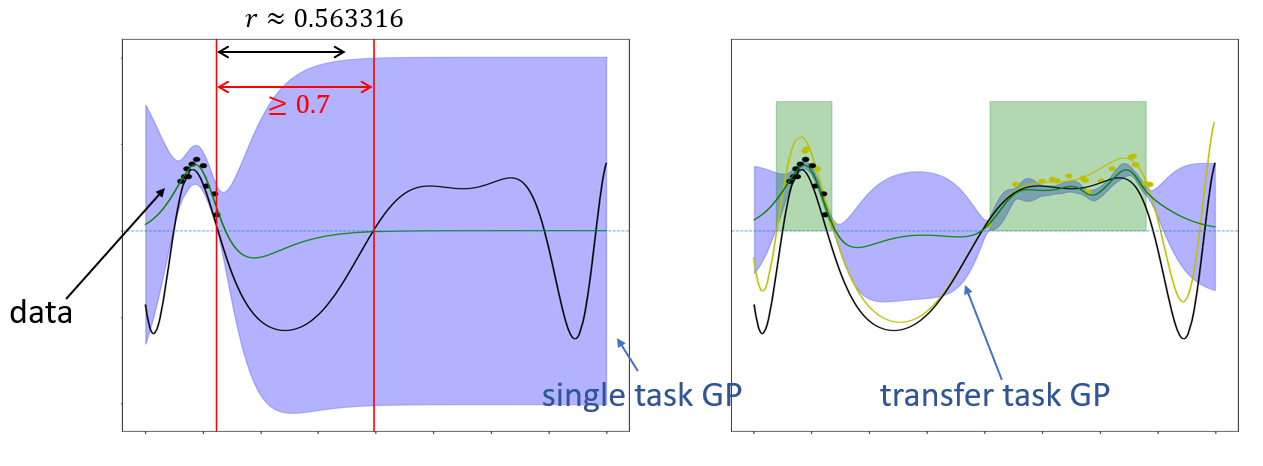}}
\caption{
The same safety constraint as in~\cref{figure2-local_safe_exploration} with two safe regions. Left: the single-task GP cannot reach the right safe region as the distance is greater than the radius $r$. Right: The multitask GP is able to exploit knowledge from the source data and build high safety confidence on the right region.
The source data comes from the function $q_s(x)=\sin\left( 10 x^3 - 5x - 10 \right) + \sin( x^2 ) - \frac{1}{2}$ and is shown in yellow.
%The models are given in~\cref{example-explorable_bound}.
}
\label{figure3-gp_local_vs_lmc_global}
\end{center}
\vskip -0.2in
\end{figure}

\section{Safe Transfer Active Learning and Source Pre-Computation}\label{section-our_transfer_method}

In this section, we formulate our safe transfer AL method.
We start from introducing transfer task GPs, then we leverage the source data and the transfer task GPs to adapt~\cref{alg-SL}.
We state the resulting new constrained optimization problem for safe transfer AL.
We then explain the complexity and consider a modular computation to facilitate the algorithm.
We conclude the section by describing our kernel choices for the experiments.

\subsection{Background: Transfer Task GPs}\label{section-our_transfer_method-MOGP}

In the presence of a source task, one can model the source task and the target task jointly with multi-output GPs~\citep{JournelHuijbregts1976miningGeo, alvarez2012kernels, Tighineanuetal22transferGPbo}.
The key idea is to augment the input with a task index variable, allowing the model to distinguish between tasks while sharing information across them.
Leveraging a source task in this way improves data efficiency on the target system, as relevant information can flow from the source to the target task.
To proceed, it is necessary to first make a hypothesis on the source data, similar to~\cref{assump1-data_generation_process} made for standard GPs.

\begin{assumption}[Data: source task]\label{assump3-data_generation_process_source}
The source data are modeled as $
y_{s} = f_s(\bm{x}) + \epsilon_{f_s},
z_{s}^j = q^{j}_{s}(\bm{x}) + \epsilon_{q^{j}_{s}}
$, where
$\{f_s, q^{1}_{s},..., q^{J}_{s}\}$ are unknown source main and safety functions, and
$s$ indexes the source task.
We assume additive noise distributed as
$\epsilon_{f_s} \sim \mathcal{N}\left( 0, \sigma_{f_s}^2 \right)$,
$\epsilon_{q^{j}_{s}} \sim \mathcal{N}\left( 0, \sigma_{q^{j}_{s}}^2 \right)$
with all noise variances
$\{\sigma_{f_s}^2, \sigma_{q_s^{1}}^2,..., \sigma_{q_s^{J}}^2\}$ being positive.
\end{assumption}

Next, we introduce task indices, with $s$ for the source task, and $t$ for the target task.
These indices allow us to describe the source and target functions jointly as multitask functions.
%\footnote{While our framework easily extends to multiple source tasks, 
%we consider in the main paper only a single source task. 
%A formulation of multiple source tasks can be found in~\cref{appendix-mogp_detail-mogp_more_source}.}
The data are then concatenated with the task indices, and, based on~\cref{assump1-data_generation_process} and~\cref{assump3-data_generation_process_source}, we define the multitask functions $\bm{f},\bm{q}^{1},...,\bm{q}^{J}: \mathcal{X} \times \{ \text{task indices} \} \rightarrow \mathbb{R}$, where $\bm{f}(\cdot, s) = f_s(\cdot)$ corresponds to the source main function, $\bm{f}(\cdot, t) = f(\cdot)$ to the target main function, $\bm{q}^{j}(\cdot, s) = q^{j}_{s}(\cdot)$ to the source safety function and $\bm{q}^{j}(\cdot, t) = q^{j}(\cdot)$ to the target safety function.%, where $j=1,...,J$. 
We use bold symbols to indicate the multitask functions.
Subsequently, we assign GP priors to the multitask functions. %in contrast to the previous section where a GP assumption is assumed only on the target task.
\begin{assumption}[Model: multitask]\label{assump3-mogp}
For each multitask function $\bm{g} \in 
\{\bm{f},\bm{q}^1,\ldots,\bm{q}^J\}$, we assume
$ \bm{g} \sim \mathcal{GP}\left( \bm{0}, k_{\bm{g}} \right)$ with 
kernel $ k_{\bm{g}}: (\mathcal{X}\times \{ \text{task index} \}) \times (\mathcal{X}\times \{ \text{task index} \}) \rightarrow \mathbb{R}$.
%$\bm{f} \sim \mathcal{GP}\left( \bm{0}, %k_{\bm{f}} \right)$ and 
%$\bm{q}^{j} \sim \mathcal{GP}\left( \bm{0}, k_{\bm{q}^{j}} \right), \forall j=1,...,J$ for some kernels $k_{\bm{f}}, k_{\bm{q}^{1}},...,k_{\bm{q}^{J}}: (\mathcal{X}\times \{ \text{task indices} \}) \times (\mathcal{X}\times \{ \text{task indices} \}) \rightarrow \mathbb{R}$.
\end{assumption}

Note that the structure of our new assumption resembles~\cref{assump2-zero_mean_stationary_GP} of a standard GP.
However, the GP is now defined jointly over the source and target task, allowing information to flow between them. 
Example kernels are provided in~\cref{section-our_transfer_method-separable_kernel}.
We proceed by presenting the predictive distribution for the main target task, leveraging source and target data. % which can leverage information from the source task to improve data efficiency.
%We write down a form for one single source task, while the distributions conditioned on data of multiple source tasks are provided in~\cref{appendix-mogp_detail-mogp_more_source}.
To incorporate task indices into the given input data, we use a hat notation:
We denote the source input, $\bm{x}_{s,1:N_{\text{source}}}$, paired with the source index $s$, with $\hat{\bm{x}}_{s, 1:N_{\text{source}}} \coloneqq \{ (\bm{x}_{s, i}, s) | \bm{x}_{s,i} \in \bm{x}_{s,1:N_{\text{source}}} \}$.
Analogously, the target training and test points, $\bm{x}_{1:N}$ and $\bm{x}_{*}$, paired with the target index $t$,  are represented as $\hat{\bm{x}}_{1:N} \coloneqq \{ (\bm{x}_i, t) | \bm{x}_i \in \bm{x}_{1:N} \}$ and
$\hat{\bm{x}}_{*} \coloneqq (\bm{x}_*, t)$.
The predictive distribution is then given by:
%We denote the input data concatenated with task indices, and we can write down the GP predictive distributions.
%Denote $\hat{\bm{x}}_{s, 1:N_{\text{source}}} \coloneqq \{ (\bm{x}_{s, i}, s) | \bm{x}_{s,i} \in \bm{x}_{s,1:N_{\text{source}}} \}$ and $\hat{\bm{x}}_{1:N} \coloneqq \{ (\bm{x}_i, t) | \bm{x}_i \in \bm{x}_{1:N} \}$.
%$\hat{\bm{x}}_{*} \coloneqq (\bm{x}_*, t)$ is a test point with target task index.
%Then the predictive distribution becomes (we write down the distribution for $\bm{f}$, while distributions for $\bm{q}^{1}, ..., \bm{q}^{J}$ can be obtained by replacing $\bm{f}$ with $\bm{q}^{1}, ..., \bm{q}^{J}$ and $y_\cdot$ with $z_\cdot^{1}, ..., z_\cdot^{J}$):
\begin{eqnarray*}
    p\left( \bm{f}(\bm{x}_{*}, t) | \bm{x}_{1:N}, y_{1:N}, \bm{x}_{s, 1:N_{\text{source}}}, y_{s, 1:N_{\text{source}}} \right) = \mathcal{N}\left( \mu_{\bm{f}, N}(\bm{x}_{*}) , \sigma_{\bm{f}, N}^2(\bm{x}_{*}) \right),
\end{eqnarray*}
where the predictive mean $\mu_{\bm{f}, N}(\bm{x}_{*})$ and variance $\sigma_{\bm{f},N}^2(\bm{x}_{*})$ are given by
\begin{align}
\begin{split}
\label{eqn2-TransferGP_posterior}
\mu_{\bm{f}, N}(\bm{x}_{*}) &=
	\bm{v}_{f}^T
	\Omega_{\bm{f}}^{-1}
	\begin{pmatrix}
	y_{s, 1:N_{\text{source}}} \\ y_{1:N}
	\end{pmatrix}, 
\\
\sigma_{\bm{f},N}^2(\bm{x}_{*}) &=
	k_{\bm{f}}\left(
	\hat{\bm{x}}_{*}, \hat{\bm{x}}_{*}
	\right)
	- \bm{v}_{\bm{f}}^T
	\Omega_{\bm{f}}^{-1}
	\bm{v}_{\bm{f}}.
\end{split}
\end{align}

The vector $\bm{v}_{f}$ represents the covariances between the training points, aggregated over the source points $\hat{\bm{x}}_{s, 1:N_{\text{source}}}$
and the target training points $\hat{\bm{x}}_{1:N}$, and the target test point $\hat{\bm{x}}_{*}$.
It is defined as:
\begin{align}
\bm{v}_{f} &=
\begin{pmatrix}
k_{\bm{f}}(\hat{\bm{x}}_{s, 1:N_{\text{source}}}, \hat{\bm{x}}_{*}) \\
k_{\bm{f}}(\hat{\bm{x}}_{1:N}, \hat{\bm{x}}_{*})
\end{pmatrix}
\text{($\in \mathbb{R} ^{(N_{\text{source}}+N) \times 1}$)}.
\end{align}
The matrix $\Omega_{\bm{f}}$ combines the kernel matrices and noise variances for both source and target tasks, forming a block structure:
\begin{align}
\label{eqn-covariance-block}
\Omega_{\bm{f}} &=
\begin{pmatrix}
K_{f_s} + \sigma_{f_s}^2 I_{N_{\text{source}}} &
K_{f_s, f} \\
K_{f_s, f}^T &
K_{f} + \sigma_{f}^2 I_{N}
\end{pmatrix}
\text{($\in \mathbb{R} ^{(N_{\text{source}}+N) \times (N_{\text{source}}+N)}$)},
\end{align}
where $K_{f_s} = k_{\bm{f}}( \hat{\bm{x}}_{s, 1:N_{\text{source}}}, \hat{\bm{x}}_{s, 1:N_{\text{source}}} )$ denotes the kernel matrix between the source data points,
$K_{f_s, f} = k_{\bm{f}}( \hat{\bm{x}}_{s, 1:N_{\text{source}}}, \hat{\bm{x}}_{1:N} )$ between source and target training points, and
$K_{f} = k_{\bm{f}}( \hat{\bm{x}}_{1:N}, \hat{\bm{x}}_{1:N} )$ between target training points.
For brevity, we omitted the task index from the predictive terms $\mu_{\bm{f}, N}(\bm{x}_{*})$ and 
$\sigma_{\bm{f},N}^2(\bm{x}_{*})$, as this paper focuses exclusively on predictions for the target task.
For brevity, we present formulas only for the main function; the safety functions are analogous.

%The GP model $\bm{f}$ is governed by the multitask kernel $k_{\bm{f}}$ (and $k_{\bm{q}^{j}}$ for $j=1,...,J$) and noise parameters $\sigma_{f_s}^2, \sigma_{f}^2$ (and $\sigma_{q^{j}_{s}}^2, \sigma_{q^{j}}^2$, $j=1,...,J$) which can be fitted with observations.

The time complexity of the predictive distribution is $\mathcal{O}\left( (N_{\text{source}} + N)^3 \right)$ due to the inversion of the Gram matrix $\Omega_{\bm{f}}$.
Similarly, the runtime for estimating the likelihood, and consequently for training the hyperparameters, is also  $\mathcal{O}\left( (N_{\text{source}} + N)^3 \right)$.
While this higher time complexity introduces additional computational overhead, it is offset by the benefit of improved data efficiency through the joint modeling of source and target tasks.
Our main paper considers one single source task, while~\cref{appendix-mogp_detail} elaborate the GP formulation of multiple source tasks.

\subsection{Safe AL with Transfer Task GPs}

In comparison to the conventional safe AL~\crefp{alg-SL}, we employ multitask GPs to model the target task jointly with the source data.
As introduced in~\cref{section-our_transfer_method-MOGP}, the multitask functions $\bm{g} \in \{\bm{f}, \bm{q}^{1},...,\bm{q}^{J}\}$ are assumed to be GPs.
At an unobserved point $\bm{x} \in \mathcal{X}$, $
p\left(
\bm{g}(\bm{x}, t) | \mathcal{D}_{N}, \mathcal{D}_{N_{\text{source}}}^{\text{source}}
\right) = \mathcal{N}\left(
\mu_{\bm{g}, N}(\bm{x}), \sigma_{\bm{g},N}^2(\bm{x})
\right)$ ($t$ is the target task index,~\cref{eqn2-TransferGP_posterior}).
The safe set and the acquisition function may then incorporate the source task information:
\begin{align}\label{eqn5-transfer_sal_acq}
\begin{split}
\mathcal{S}_N &= \cap_{j=1}^{J} \{ \bm{x} \in \mathcal{X}_{\text{pool}} | \mu_{\bm{q}^{j}, N}(\bm{x}) - \beta^{1/2} \sigma_{\bm{q}^{j}, N}(\bm{x}) \geq T_j\},\\
a\left( \bm{x} \right) &=
\sum_{\bm{g} \in \{\bm{f}, \bm{q}^{1}, ..., \bm{q}^{J}\} } H_{\bm{g}} \left[ \bm{x}|\mathcal{D}_N, \mathcal{D}_{N_{\text{source}}}^{\text{source}} \right]\\
&= \sum_{\bm{g} \in \{\bm{f}, \bm{q}^{1}, ..., \bm{q}^{J}\} } \frac{1}{2} \log\left( 2 \pi e \sigma^2_{\bm{g}, N}(\bm{x}) \right),\\
\bm{x}_*
&= \text{argmax}_{\bm{x} \in \mathcal{S}_N } a\left( \bm{x} \right).
\end{split}
\end{align}
In contrast to the standard safe AL, $a(\cdot)$ and $\mathcal{S}_{N}$ here depend on the observed target data $\mathcal{D}_N$ and the source data $\mathcal{D}_{N_{\text{source}}}^{\text{source}}$, as they rely on  $\bm{q}$ and $\bm{f}$ (multitask functions in bold symbols),
which represent the multitask GPs based on source and target data.

The whole learning process is summarized in~\cref{alg-fullSTL}.
Its computational complexity is dominated by fitting the GPs (line 2).
Common fitting techniques include Type II ML, Type II MAP and Bayesian treatment~\citep{Snoek_etal12bo, riis2022bayesian} over kernel and noise parameters~\citep{GPbook}. 
All of these approaches have in common that they require computing the marginal likelihoods,
\begin{eqnarray*}
\mathcal{N}\left(
\begin{pmatrix} y_{s, 1:N_{\text{source}}} \\ y_{1:N} \end{pmatrix} | \bm{0}, \Omega_{\bm{f}}
\right) \text{ and } \mathcal{N}\left(
\begin{pmatrix} z_{s, 1:N_{\text{source}}}^j \\ z_{1:N}^j \end{pmatrix} | \bm{0}, \Omega_{\bm{q}^{j}}
\right),
\end{eqnarray*} 
for each safety constraint $j=1,...,J$.
In this work, we do not consider Bayesian treatments due to the high computational cost of Monte Carlo sampling.
Obtaining $\Omega_{\bm{f}}^{-1}$ (and $\Omega_{\bm{q}^{j}}^{-1}, j=1,...,J$) for the marginal likelihood takes $\mathcal{O}\left( (N_{\text{source}} + N)^3 \right)$ time, where $N_{\text{source}}$ can be large in our set-up.
Moreover, the process must be iterated for $N=N_{\text{init}},...,N_{\text{init}}+N_{\text{query}}$ adding to the computational burden.
In the next section, we demonstrate how the computational burden can be significantly reduced by pre-computing the source-specific terms necessary for the matrix inversion.

\begin{algorithm}[b]
\caption{Full safe transfer AL}
\label{alg-fullSTL}
\begin{algorithmic}[1]
\Require \cref{assump1-data_generation_process},~\cref{assump3-data_generation_process_source},~\cref{assump3-mogp}, $\mathcal{D}_{N_{\text{init}}}, \mathcal{D}_{N_{\text{source}}}^{\text{source}}, \mathcal{X}_{\text{pool}}$, $\beta$ or $\alpha$, $N_{\text{query}}$, thresholds $T_1,...,T_J$
\For{$N=N_{\text{init}}, ..., N_{\text{init}} + N_{\text{query}} - 1$}
\State Fit GPs $\bm{f}, \bm{q}_1, ..., \bm{q}_J$ with $\mathcal{D}_{N}, \mathcal{D}_{N_{\text{source}}}^{\text{source}}$
\State $\mathcal{S}_N \gets \cap_{j=1}^{J} \{ \bm{x} \in \mathcal{X}_{\text{pool}} | \mu_{\bm{q}^{j}, N}(\bm{x}) - \beta^{1/2} \sigma_{\bm{q}^{j}, N}(\bm{x}) \geq T_j\}$~\crefp{eqn5-transfer_sal_acq}
\State $\bm{x}_*\gets \text{argmax}_{\bm{x} \in \mathcal{S}_N } \ a(\bm{x} ),
\ a(\bm{x} )=
\sum_{\bm{g} \in \{\bm{f}, \bm{q}^{1}, ..., \bm{q}^{J}\} } H_{\bm{f}} \left[ \bm{x}|\mathcal{D}_N, \mathcal{D}_{N_{\text{source}}}^{\text{source}} \right]$
\State Query $\bm{x}_*$ to get evaluations $y_*$ and $\bm{z}_*$
\State $\mathcal{D}_{N+1} \gets \mathcal{D}_{N} \cup \{ \bm{x}_*, y_*, \bm{z}_*  \}, \mathcal{X}_{\text{pool}} \gets \mathcal{X}_{\text{pool}} \setminus \{ \bm{x}_* \}$
\EndFor
\State \textbf{Return}  $\mathcal{D}_{N_{\text{init}} + N_{\text{query}}}$, trained models $\bm{f}, \bm{q}_1, ..., \bm{q}_J$
\end{algorithmic}
\end{algorithm}

\begin{algorithm}[b]
\caption{Modularized safe transfer AL}\label{alg-modSTL}
\begin{algorithmic}[1]
\Require \cref{assump1-data_generation_process},~\cref{assump3-data_generation_process_source},~\cref{assump3-mogp}, $\mathcal{D}_{N_{\text{init}}}, \mathcal{D}_{N_{\text{source}}}^{\text{source}}, \mathcal{X}_{\text{pool}}$, $\beta$ or $\alpha$, $N_{\text{query}}$, thresholds $T_1,...,T_J$
\State Fit GPs $\bm{f}, \bm{q}_1, ..., \bm{q}_J$ with $\mathcal{D}_{N_{\text{source}}}^{\text{source}}$
\State Fix source specific parameters $\theta_{f_s}, \theta_{q^{j}_{s}}, \sigma_{f_s}, \sigma_{q^{j}_{s}}$, $\forall j=1,...,J$
\State Compute and fix $L_{f_s}, L_{q^{j}_{s}}$, $\forall j=1,...,J$ (line 5, 6, 7 below faster)
\For{$N=N_{\text{init}}, ..., N_{\text{init}}+N_{\text{query}} - 1$}
\State Fit GPs with $\mathcal{D}_{N}$ and $\mathcal{D}_{N_{\text{source}}}^{\text{source}}$ (free parameters $\theta_{f}, \theta_{q^{j}}, \sigma_{f}, \sigma_{q^{j}}$, $\forall j=1,...,J$)
\State $\mathcal{S}_N \gets \cap_{j=1}^{J} \{ \bm{x} \in \mathcal{X}_{\text{pool}} | \mu_{\bm{q}^{j}, N}(\bm{x}) - \beta^{1/2} \sigma_{\bm{q}^{j}, N}(\bm{x}) \geq T_j\}$~\crefp{eqn5-transfer_sal_acq}
\State $\bm{x}_*\gets \text{argmax}_{\bm{x} \in \mathcal{S}_N } \ a(\bm{x} ),
\ a(\bm{x} )=
\sum_{\bm{g} \in \{\bm{f}, \bm{q}^{1}, ..., \bm{q}^{J}\} } H_{\bm{f}} \left[ \bm{x}|\mathcal{D}_N, \mathcal{D}_{N_{\text{source}}}^{\text{source}} \right]$
\State Query $\bm{x}_*$ to get evaluations $y_*$ and $\bm{z}_*$
\State $\mathcal{D}_{N+1} \gets \mathcal{D}_{N} \cup \{ \bm{x}_*, y_*, \bm{z}_*  \}, \mathcal{X}_{\text{pool}} \gets \mathcal{X}_{\text{pool}} \setminus \{ \bm{x}_* \}$
\EndFor
\State \textbf{Return} $\mathcal{D}_{N_{\text{init}} + N_{\text{query}}}$, trained models $\bm{f}, \bm{q}_1, ..., \bm{q}_J$
\end{algorithmic}
%\end{minipage}
%\vspace{-10pt}
%\end{wrapfigure}
\end{algorithm}

\subsection{Source Pre-Computation}
%The transfer learning scheme~\crefp{alg-fullSTL} may still be improved by reducing the complexity.
%The complexity of~\cref{alg-fullSTL} is dominated by fitting the GPs (line 2).
%Common fitting techniques include Type II ML, Type II MAP and Bayesian treatment~\citep{Snoek_etal12bo, riis2022bayesian} over kernel and noise parameters~\citep{GPbook}, all of which need to compute the marginal likelihoods
%$\mathcal{N}\left(
%\begin{pmatrix} y_{s, 1:N_{\text{source}}} \\ y_{1:N} \end{pmatrix} | \bm{0}, \Omega_{\bm{f}}
%\right)$ and $\mathcal{N}\left(
%\begin{pmatrix} z_{s, 1:N_{\text{source}}}^j \\ z_{1:N}^j \end{pmatrix} | \bm{0}, \Omega_{\bm{q}^{j}}
%\right)$, for all safety constraint $j=1,...,J$.
%In our paper, Bayesian treatment is not considered because MC sampling is time-consuming.
%Obtaining $\Omega_{\bm{f}}^{-1}$ (and $\Omega_{\bm{q}^{j}}^{-1}, j=1,...,J$) for the marginal likelihood takes $\mathcal{O}\left( (N_{\text{source}} + N)^3 \right)$ in time, $N_{\text{source}}$ is large in our problem, and we iterate for $N=N_{\text{init}},...,N_{\text{init}}+N_{\text{query}}$.
In this section, we propose an efficient algorithm to mitigate the computational burden
 of repeatedly calculating
 $\Omega_{\bm{f}}^{-1}$ and $\Omega_{\bm{q}^{j}}^{-1}$ in full.
For clarity, we describe the approach for
 $\Omega_{\bm{f}}^{-1}$, the same principles apply to $\Omega_{\bm{q}^{j}}^{-1}$ for all $j=1,\ldots, J$.
 
For GP models, the matrix inversion is routinely achieved by performing a Cholesky decomposition $L(\Omega_{\bm{f}})$, which has cubic complexity.
This decomposes $ \Omega_{\bm{f}}$ as $ \Omega_{\bm{f}} = L(\Omega_{\bm{f}})L(\Omega_{\bm{f}})^T$, where $L(\Omega_{\bm{f}})$ is a lower triangular matrix~\citep{GPbook}.
Once the decomposition is obtained, operations such as $L(\Omega_{\bm{f}})^{-1}C$, for any matrix $C$, can be efficiently computed by solving a linear system with minor complexity.
The Cholesky decomposition is well known for its numerically stability and computationally efficiency, making it a widely preferred approach for efficient GP computations.

We propose to perform the Cholesky decomposition in two steps, as described below.
The key idea is to precompute the source-specific terms of the Cholesky decomposition, which account for a large amount of the computational costs.
Importantly, our technique is general and can be applied to any multi-output kernel.
Recall from~\cref{eqn-covariance-block} 
that the covariance $\Omega_{\bm{f}}$ has a block structure, in which the source block $K_{f_s} + \sigma_{f_s}^2 I_{N_{\text{source}}}$ has size $N_{\text{source}} \times N_{\text{source}}$ that dominates the computation.
The Cholesky decomposition can also be expressed as block structure,
\begin{eqnarray*}
L(\Omega_{\bm{f}}) = 
\begin{pmatrix}
L_{f_s} & 0 \\
L_{f_s,f} & L_f \\
\end{pmatrix},
\end{eqnarray*}
where the source block $L_{f_s}$ can be precomputed independently of the remaining covariances~\citep{press1988numerical}.
Once $L_{f_s}$ is obtained, it is then used to compute the cross-term $L_{f_s,f}$ and target block $ L_f$ that are both a function of the source block $L_{f_s}$(details in~\cref{appendix-mogp_detail-modularized_mogp_computation}). 
If the source covariance, $K_{f_s} + \sigma_{f_s}^2 I_{N_{\text{source}}}$, remains unchanged between different covariances $\Omega_{\bm{f}}$, its precomputed Cholesky decomposition $L_{f_s}$ can be reused, significantly reducing computational overhead.

%Recall from~\cref{section-our_transfer_method-MOGP} that the block $K_{f_s} + \sigma_{f_s}^2 I_{N_{\text{source}}}$ in $\Omega_{\bm{f}}$ \eqref{eqn-covariance-block}, a $N_{\text{source}}$ by $N_{\text{source}}$ block of source specific components, depends only on $k_{\bm{f}}\left( (\cdot, s), (\cdot, s) \right)$ and $\sigma_{f_s}^2$, where the multitask kernel is considered in the form $
%k_{\bm{f}}( (\cdot, \{s, t\}), (\cdot, \{s, t\})) =
%\begin{pmatrix}
%k_{\bm{f}}\left( (\cdot, s), (\cdot, s) \right) &
%k_{\bm{f}}\left( (\cdot, s), (\cdot, t) \right)\%\
%k_{\bm{f}}\left( (\cdot, t), (\cdot, s) \right) %& k_{\bm{f}}\left( (\cdot, t), (\cdot, t) \right)
%\end{pmatrix}
%$, $s$ and $t$ are the source and target task indices.

\paragraph{Fixed source parameters for efficient training:}
To leverage the precomputed Cholesky decomposition $L_{f_s}$ during our safe transfer AL scheme, the parameters governing the source covariance, $K_{f_s} + \sigma_{f_s}^2 I_{N_{\text{source}}}$, must remain fixed throughout the algorithm.

To achieve this, we split the kernel parameters $\bm{\theta}_{\bm{f}}$ into two groups, $\bm{\theta}_{\bm{f}} = (\theta_{f_s}, \theta_{f})$, where $\theta_{f_s}$ include all parameters required for computing $K_{f_s}$ and $\theta_f$ contains the remaining parameters.
We first train on the source data $\mathcal{D}_{N_{\text{source}}}^{\text{source}}$ alone, then fix $\theta_{f_s}$ and $\sigma_{f_s}^2$. 
Once these parameters are fixed, the Cholesky decomposition $L_{f_s}$ can be precomputed and reused across all subsequent iterations when acquiring the target dataset $\mathcal{D}_{N}$.
During this phase, we can still update the parameters $\theta_{f}$ and the target noise variance $\sigma_f^2$.

The learning procedure is summarized in~\cref{alg-modSTL}.
In each iteration (line 5 of~\cref{alg-modSTL}), the time complexity reduces from  $\mathcal{O}\left( (N_{\text{source}} + N)^3 \right)$ to $\mathcal{O}(N_{\text{source}}^2 N) + \mathcal{O}(N_{\text{source}} N^2) + \mathcal{O}(N^3)$.
We provide mathematical details in the~\cref{appendix-mogp_detail-modularized_mogp_computation}.
Note that our approach offers a trade-off: it reduces parameter flexibility in exchange for computational efficiency. 
We will discuss the pros and cons of this approach, depending on the kernel choice, in more detail in the following section.

%Consequently, the expensive Cholesky decomposition $L_{f_s} = L(K_{f_s} + \sigma_{f_s}^2 I_{N_{\text{source}}})$ is pre-computed and can be fixed, disregarding the varying  during the learning.
% In summary, the key is to separate the GP parameters of $k_{\bm{f}}$ into $\bm{\theta}_{\bm{f}} = (\theta_{f_s}, \theta_{f})$, where $\theta_{f_s}$ is a vector of all parameters necessary for the source specific covariance map $k_{\bm{f}}\left( (\cdot, s), (\cdot, s) \right)$.
%Then we may train on the given $\mathcal{D}_{N_{\text{source}}}^{\text{source}}$, fix $\theta_{f_s}, \sigma_{f_s}^2$, and compute $L_{f_s} = L(K_{f_s} + \sigma_{f_s}^2 I_{N_{\text{source}}})$ which remains invariant.
%The tradeoff is that the training on target task is also less flexible, as the free parameters are $\theta_{f}$ only.

\subsection{Kernel Selection}
\label{section-our_transfer_method-separable_kernel}

Multitask kernels are often defined as a matrix of functions~\citep{JournelHuijbregts1976miningGeo, alvarez2012kernels}, where each element maps $\mathcal{X} \times \mathcal{X} \rightarrow \mathbb{R}$ similar to a standard kernel.
The task indices determine which element of the matrix is used. 
Specifically, for task indices $i, i'\in \{s, t\}$, the kernel can be expressed as
\begin{eqnarray*}
k_{\bm{g}}( (\cdot, i), (\cdot, i')) =
\begin{pmatrix}
k_{\bm{g}}\left( (\cdot, s), (\cdot, s) \right) &
k_{\bm{g}}\left( (\cdot, s), (\cdot, t) \right)\\
k_{\bm{g}}\left( (\cdot, t), (\cdot, s) \right) & k_{\bm{g}}\left( (\cdot, t), (\cdot, t) \right)
\end{pmatrix}_{i, i'},
\end{eqnarray*}
where the dots $\cdot$ are placeholders for input data from $\mathcal{X}$.
%It is then obvious that the source specific covariance block in $\Omega_{\bm{f}}$, i.e. $K_{f_s}$, depends only on $k_{\bm{f}}\left( (\cdot, s), (\cdot, s) \right)$.
%This will be important in~\cref{section-our_transfer_method}.
%In the following, we briefly review existing multi-output GP models and motivate selection of the model we use later in our experiments.
Here, each $\bm{g}\in \{ \bm{f}, \bm{q}^{1},...,\bm{q}^{J} \}$ is a multi-output GP correlating source and target tasks for the main and safety functions.
%	The task indices $s, t$ are binary.
%As described previously, the multitask kernel can be expressed as a matrix of functions: $
%k_{\bm{g}}( (\cdot, \{s, t\}), (\cdot, \{s, t\})) =
%\begin{pmatrix}
%k_{\bm{g}}\left( (\cdot, s), (\cdot, s) \right) %&
%k_{\bm{g}}\left( (\cdot, s), (\cdot, t) \right)\%\
%k_{\bm{g}}\left( (\cdot, t), (\cdot, s) \right) %& k_{\bm{g}}\left( (\cdot, t), (\cdot, t) %\right)
%\end{pmatrix}
%$.

\paragraph{Linear model of coregionalization (LMC):} 
A widely investigated multi-output framework is the linear model of coregionalization (LMC) which we also use in our experiments.
In our setup, the kernel is defined as
\begin{eqnarray*}
k_{\bm{g}}( (\cdot, \{s, t\}), (\cdot, \{s, t\}))=
\sum_{l=1}^{2}
\left(
W_l W_l^T +
\begin{pmatrix}
\kappa_s & 0 \\
0 & \kappa \\
\end{pmatrix}
\right)
\otimes k_{l}(\cdot, \cdot),
\end{eqnarray*}
where $\otimes$ denote the Kronecker product. 
We assume two latent effects and each latent effect is specified by the base kernel
$k_{l}(\cdot, \cdot): \mathcal{X} \times \mathcal{X} \rightarrow \mathbb{R}$.
The parameters  $W_l \in \mathbb{R}^{2 \times 1}, \kappa_s, \kappa >0$ model the task correlations induced by the $l$-th latent pattern encoded by $k_l$~\citep{alvarez2012kernels}.
Here, each $\bm{g}$ has its own kernel, but for brevity, we omit $\bm{g}$ in the parameter subscripts.
Furthermore, if $k_l$ includes a variance scaling term, e.g. Mat{\'e}rn kernels, it is fixed to $1$ because the scale can be absorbed into $W_l, \kappa_s$ and $\kappa$.

This kernel design is tied to our experimental setup and facilitates the transfer of information from the source to the target task.
However, when paired with~\cref{alg-modSTL}, training can become unstable, because the algorithm assumes that the kernel parameters can be cleanly separated between source and task terms. 
In the case of the LMC, this separation is not straightforward:
The latent components 
$W_l$ encode shared task correlations, while 
$\kappa_s$ and $\kappa$ represent task-specific effects. 
Training on source data alone provides insufficient information to disentangle these shared and individual contributions, potentially leading to suboptimal solutions that destabilize the training process.

%it is difficult to seperate the parameters between soure and target tasks

%especially during the first phase (line 1 of~\cref{alg-modSTL}). This instability may stem from multiple local optima, as the LMC learns joint patterns across tasks simultaneously.
	
\paragraph{Hierarchical GP (HGP):}In~\cite{poloczek2017MultiInfoSourceOptim, Marco2017sim2realRLviaBO, Tighineanuetal22transferGPbo}, the authors consider a hierarchical GP (HGP) framework, where the kernel is defined as:
\begin{eqnarray*}
k_{\bm{g}}( (\cdot, \{s, t\}), (\cdot, \{s, t\}))=
\begin{pmatrix}
k_s(\cdot, \cdot) & k_s(\cdot, \cdot)\\
k_s(\cdot, \cdot) & k_s(\cdot, \cdot) + k_t(\cdot, \cdot)
\end{pmatrix},
\end{eqnarray*}
with $k_s, k_t: \mathcal{X} \times \mathcal{X} \rightarrow \mathbb{R}$ as base kernels.
HGP is a variant of LMC, where the target task is modeled as the sum of the source kernel $k_s$ and the target-specific residual $k_t$~\citep{Tighineanuetal22transferGPbo}.
This formulation has the benefit that the fitting of terms $k_s$ and $k_t$
can be easily decoupled, making HGP particularly well suited for the use of~\cref{alg-modSTL}. %(set $\theta_{g_s}$ the parameters of $k_s$ and $\theta_{g_s}$ the parameters of $k_t$).
%, subscripts of $k_s$ and $k_t$ are only for numbering here and are not relevant to the definition in previous sections.
%Similarly , each $\bm{g}$ has its own kernel, but we omit $\bm{g}$ in the matrix subscripts for brevity.
	
In~\cite{Tighineanuetal22transferGPbo}, the authors derived an ensembling technique that also supports source pre-computation.
While their approach is equivalent to our method when applied to HGP, our framework generalizes to any multi-output kernel, provided that the fitting of the source and target parameters can be decoupled. 
In contrast, the ensembling technique is explicitly tailored to HGP and does not generalize to other kernel structures.
%Their technique is equivalent to our method when we use HGP, but our approach can be generalized to any multi-output kernels (with implicit restriction that a source fitting of the chosen model needs to be accurate) while the ensembling technique is limited to HGP.
	
\paragraph{Kernel selection in our experiments:}
In our experiments, we perform~\cref{alg-modSTL} with HGP as our main method, and~\cref{alg-fullSTL} with LMC and HGP as ablation methods.
While our main method is more computationally efficient, LMC offers greater flexibility in model task correlations.
Running HGP with~\cref{alg-fullSTL} and~\cref{alg-modSTL} allows us to study the effect of sequential parameter learning against joint parameter learning, with the latter having an increased runtime. For both HGP and LMC, we construct the kernels using Matérn-5/2 kernels with $D$ lengthscale parameters.
This choice is not restrictive and can be replaced with other base kernels suited to specific applications. 

Although we did not pair~\cref{alg-modSTL} with LMC as discussed above, our modularized computation scheme can still provide benefits in closely related settings, e.g. (i) datasets with multiple sources or (ii) sequential learning frameworks where GPs are refitted only after a batch of query points has been acquired.

\section{Experiments}\label{section-experiments}

\begin{figure*}[t]
	\vskip 0.2in
	\begin{center}
		\centerline{\includegraphics[width=0.95\textwidth]{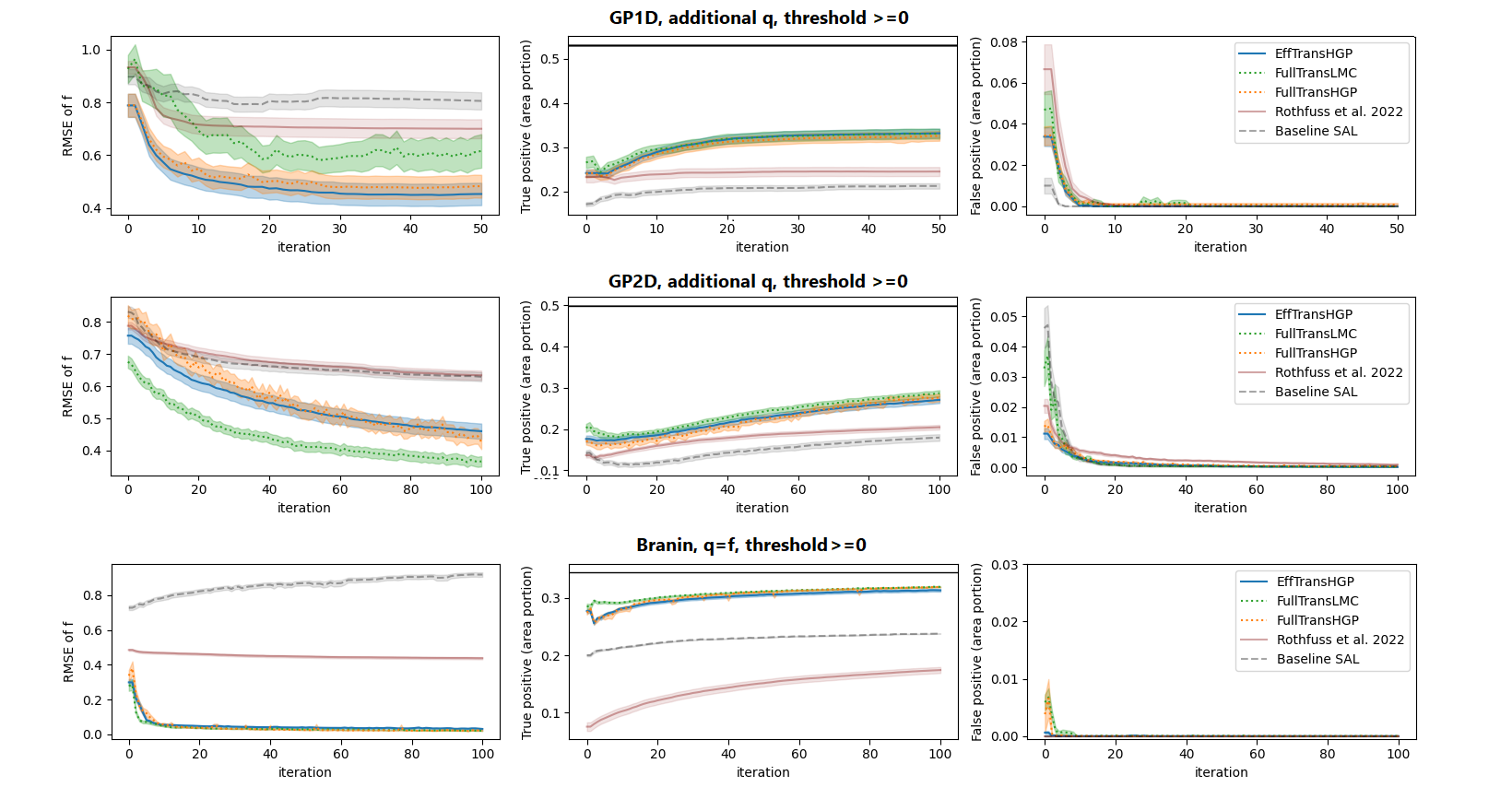}}
		\centerline{\includegraphics[width=0.95\textwidth]{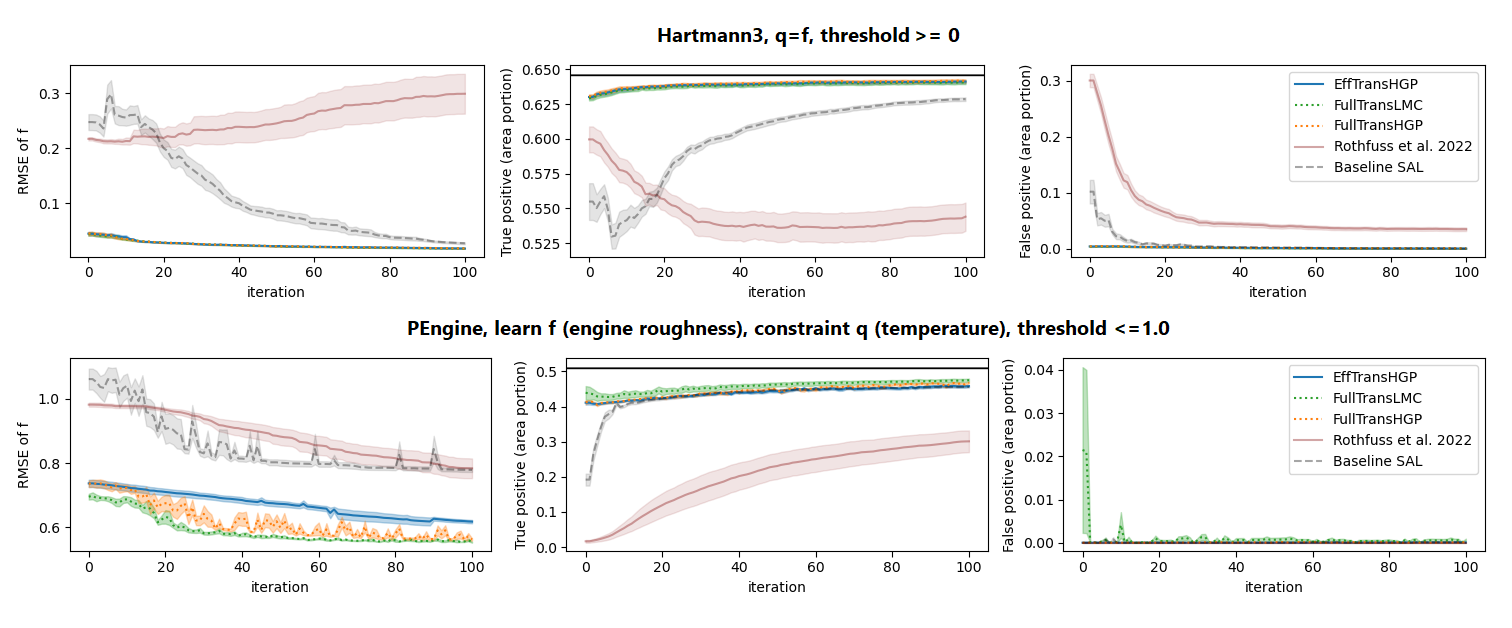}}
		\centerline{\includegraphics[width=0.95\textwidth]{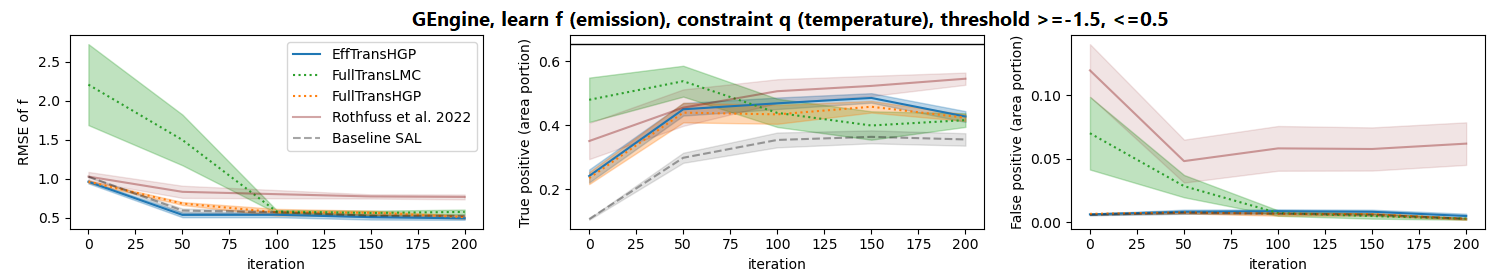}}
		\caption{
        Empirical performance across all six benchmark datasets: RMSE to assess model convergence, TP rate to measure the coverage of the safe space explored, and FP rate to evaluate the safety of each approach. Both TP and FP compute the rates to the area of $\mathcal{X}_{\text{pool}}$. The ground truth safe area portion for each dataset is indicated by a black line in the second column. Our approach generally shows improved convergence in terms of model performance and the extent of explored safe regions, while maintaining safety levels comparable to the baseline SAL. %In contrast, Rothfuss2022 is not safety-aware. However, our LMC approach falls short on the GEngine dataset, which we discuss in more detail in the main text.
        On GEngine, we additionally provide a zoomed-in RMSE figure~\crefp{figure4-gengine_zoomin}.
		}
		\label{main-figure}
	\end{center}
	\vskip -0.2in
\end{figure*}

In this section, we empirically evaluate our approach against state-of-the-art competitors on a range of synthetic and real-world datasets.
We first provide details on the experimental setup in Section~\ref{subsection-experiment-details}. 
Then, we analyze whether our transfer learning scheme is more data-efficient than conventional methods in Section~\ref{subsection-experiments-power}, 
whether it facilitates the learning of disconnected safe regions in Section~\ref{subsection-experiments-disconnected}, 
and how the runtime of our modularized approach compares in Section~\ref{subsection-experiments-runtime}.

Our code is available at \githublink.

\subsection{Experimental Details}\label{subsection-experiment-details}

First, we describe comparison partners and the datasets we use in our experiments.

\subsubsection{Comparison Partners} 
We compare five different methods:
\textbf{1)} EffTransHGP:~\cref{alg-modSTL} with multi-output HGP,
\textbf{2)} FullTransHGP:~\cref{alg-fullSTL} with multi-output HGP,
\textbf{3)} FullTransLMC:~\cref{alg-fullSTL} with multi-output LMC,
\textbf{4)} Rothfuss2022: GP model meta learned with the source data by applying~\cite{rothfuss_meta-learning_2022}, and
\textbf{5)} SAL: the conventional~\cref{alg-SL} with single-output GPs.

The first three methods are our proposed approaches, listed in order of increasing complexity. 
The HGP kernel is a variant of the LMC kernel.
We test two variations of the HGP: one using our modularized implementation~\crefp{alg-modSTL}, with a runtime complexity comparable to the single-task approach, 
and another one using a naive implementation~\crefp{alg-fullSTL} that has a similar runtime complexity as LMC.
For the safety tolerance, we always fix $\beta=4$, which corresponds to $\alpha= 1 - \Phi(\beta^{1/2}) = 0.02275$~\crefp{eqn3-safe_set,eqn5-transfer_sal_acq}, implying that each fitted GP safety model allows $2.275 \%$ unsafe tolerance.
For the baseline following~\cite{rothfuss_meta-learning_2022}, the GP model parameters are meta learned up-front using source data, and remain fixed throughout the experiments.
While the authors of the original paper applied this approach to safe BO problems, we modify the acquisition function to entropy, transforming it into a safe AL method.
All methods use Mat{\'e}rn-5/2 kernels as the base kernels.
%In our experiments, we consider the noisy variables when we compute the safe set (model $z^{j}\geq T_j$ instead of $q^{j} \geq T_j$).
%This is consistent to the implementation of~\cite{rothfuss_meta-learning_2022}, and has no effect to our argument~\crefp{appendix-safeset_noisefree_vs_noisy}.
%
To be consistent with~\cite{rothfuss_meta-learning_2022}, we consider noisy variables when we compute the safe set (model $z^{j}\geq T_j$ instead of $q^{j} \geq T_j$ for all $j=1,...,J$) in our experiments. We elaborate in~\cref{appendix-safeset_noisefree_vs_noisy} that our theoretical analysis of~\cref{section4-gp_no_jump} extends to this case. 

\begin{table*}[t]
    \centering
    \caption{
    Dataset Summary: For each dataset, we list the input dimension $D$, the size of the source dataset $N_{\text{source}}$, the size of the initial target dataset $N_{\text{init}}$, the number of queries $N_{\text{query}}$, decription, safety threshold and whether the disjoint safe regions can be tracked. Datasets are listed in order of increasing complexity.
    Each task has one safety variable.
    }
    \label{table-datasets}
    \begin{tabular}{l|ccccccc}
        \toprule
        Dataset    & $D$ & $N_{\text{source}}$ & $N_{\text{init}}$ & $N_{\text{query}}$ & Description & Threshold & Disjoint regions  \\
        \midrule
        GP1D       & 1   & 100                 & 10                & 50                 & Synthetic, $\bm{f}\neq\bm{q}$, & $\geq 0$ & Tracked \\
        GP2D       & 2   & 250                 & 20                & 100                & Synthetic, $\bm{f}\neq\bm{q}$, & $\geq 0$ & Tracked \\
        Branin     & 2   & 100                 & 20                & 100                & Synthetic, $\bm{f}=\bm{q}$, & $\geq 0$ & Tracked\\
        Hartmann3  & 3   & 100                 & 20                & 100                & Synthetic, $\bm{f}=\bm{q}$ & $\geq 0$ & Intractable \\
        PEngine    & 2   & 500                 & 20                & 100                & Semi-real-world, $\bm{f}\neq\bm{q}$ & $\leq 1.0$ & Intractable \\
        GEngine    & 13  & 500                 & 20                & 200                & Real-world, $\bm{f}\neq\bm{q}$ & $\geq -1.5$, $\leq 0.5$ & Intractable \\
        \bottomrule
    \end{tabular}
\end{table*}

 \subsubsection{Datasets}

We benchmark our methods on six datasets.
An overview of the datasets is given in Table~\ref{table-datasets}.

\paragraph{Synthetic Datasets:}
We first create two low-dimensional synthetic datasets, GP1D ($D=1$) and GP2D ($D=2$), generating multi-output GP samples following algorithm 1 of ~\cite{KanHenSejSri18}. 
For each dataset, we treat the first output as the source task and the second as the target task. 
Each dataset has a main function $\bm{f}$ and an additional safety function $\bm{q}$.
We generate 10 datasets and repeat each experiment five times for each method on every dataset.
For the Branin dataset, we follow the settings from~\cite{rothfuss_meta-learning_2022, Tighineanuetal22transferGPbo} to produce five datasets and run five repetitions for each method on each dataset. 
Unlike GP1D and GP2D, Branin uses the same function for both main and safety tasks.
In these initial experiments, we simulate multiple datasets but retain only those in which the target task has at least two disjoint safe regions, with each disjoint region also having a safe counterpart in the source dataset. 
This design aligns with our use of the Matérn-5/2 kernel, which measures similarity between data points based on proximity. 
%Consequently, disjoint regions in the target dataset are also disjoint in the source dataset.
%Modeling more complex transfer patterns, such as correlations in abstract feature spaces, would require selecting more advanced base kernels (e.g.~\cite{Bitzer2022kernel_search}).
Our fourth synthetic dataset is the Hartmann3 dataset ($D=3$), created using the settings from~\cite{rothfuss_meta-learning_2022, Tighineanuetal22transferGPbo}.
We generate five datasets and repeat experiments on each datasets five times. 
Here, the source and initial target datasets are sampled randomly, unlike the structured, disjoint safe regions in GP1D, GP2D, and Branin. 
All datasets are normalized, and the constraint thresholds are set to zero.
Further details on our synthetic datasets are provided in~\cref{appendix-experiment_details-numerical}.

\paragraph{Semi-Real-World Dataset (PEngine):}
The PEngine dataset consists of two datasets measured on the same engine prototype under varying conditions.
The outputs temperature, roughness, HC, and NOx emissions are recorded.
We perform separate AL experiments to learn roughness~\crefp{main-figure} and temperature (Appendix~\cref{figureS4}), both constrained by a normalized temperature $q$, threshold on noisy observation $z \leq 1.0$, resulting in a safe set covering approximately 52.93\% of the input space.\footnote{
\label{footnote-experiment_safety_notation}
In general, we use the notation 
$\bm{z}=\{z^1, \ldots, z^J\}$ to represent $J$ safety constraints.
However, since all datasets in our experiments involve only a single safety function, we simplify the notation to $z$.
%$J$ will be omitted because each dataset has only one safety variable.
}
The upper bound constraint is equivalent to $-z \geq -1.0$ as described in our~\cref{section-problem_statement}, $-z$ being the negative noisy temperature.
The raw datasets contain four input variables: two free variables and two contextual variables, with the contextual inputs recorded with noise. 
To fix the contextual inputs at constant values, we interpolate these noisy values using a multi-output GP simulator trained on the full datasets. 
This allows us to perform active learning experiments solely on the two-dimensional space of the free variables, creating a semi-simulated environment. 
Further details are provided in~\cref{appendix-experiment_details-numerical}.
Our GitHub repository provides a link to the dataset.

\paragraph{Real-World Dataset (GEngine):}
Our final benchmark is a high-dimensional, real-world problem involving two datasets, each recorded by a related but distinct engine (one as the source and the other as the target task) from~\cite{cyli2022}. 
The original dataset are split into training and test sets, and we conduct AL experiments on the training sets, while RMSE and safe set performance are evaluated on the target test set. 
These datasets are dynamic, and our model applies a history structure by concatenating relevant past points into the inputs, resulting in an input dimension of $D=13$.
The recording include emissions and temperature, and we learn the normalized emission ($f$), subject to normalized temperature $q$, threshold on noisy value $-1.5\leq z \leq 0.5$.\footref{footnote-experiment_safety_notation}
The upper bound on temperature is crucial for safety, while the lower bound increases robustness against outliers.
Overall, the safe region covers approximately $65\%$ of the target dataset.
For the source task, we sample the data under a different constraint of $-2\leq z_s \leq 0.5$ to make the model more resistant to outliers.
More details can be found in~\cref{appendix-experiment_details-numerical}.
Our GitHub repository provides a link to the dataset.

\subsection{Modeling Performance \& Safety Coverage}\label{subsection-experiments-power}
In the following, we study the empirical performance of our algorithms to find out whether our methods can accelerate space exploration and model convergence while maintaining safety.

\paragraph{Metrics:}
We evaluate model convergence of the main function $\bm{f}$ using root mean square error (RMSE) between the GP predictive mean and test $y$ sampled from true safe regions.
To measure the performance of our safety function $\bm{q}$, we use the area of $\mathcal{S}_N$, as this indicates the explorable coverage of the space.
Specifically, we consider the area of $\mathcal{S}_N \cap \mathcal{S}_{\text{true}}$ (true positive or TP area, the larger the better) and $\mathcal{S}_N \cap \left(\mathcal{X}_{\text{pool}} \setminus \mathcal{S}_{\text{true}} \right)$ (false positive or FP area, the smaller the better).
Here, $\mathcal{S}_{\text{true}} \subseteq \mathcal{X}_{\text{pool}}$ denotes the set of true safe candidate inputs, which we can precompute as we use a fixed data pool.
$\mathcal{S}_{\text{true}}$ takes noise-free variables except for GEngine where only noisy data are available.
Area of $\mathcal{X}_{\text{pool}}, \mathcal{S}_N, \mathcal{S}_{\text{true}}$ are all measured by counting the number of points.

\paragraph{Results:}
We report results in~\cref{main-figure}.

\paragraph{Results on GP1D, GP2D, Branin:}
We begin by focusing on the GP1D, GP2D, and Branin datasets, which have been simulated to contain multiple disconnected safe regions. 
On these datasets, only methods capable of jumping between regions can achieve optimal performance. 
In~\cref{main-figure}, we observe that our transfer learning approaches achieve lower RMSE and significantly greater safe set coverage than competing methods, while maintaining small false detection rate of safe area.
These results suggest that our methods can successfully identify and explore disconnected safe regions, while our competitor methods cannot. 
We will conduct an in-depth analysis of this aspect in the next section.
The higher RMSE of our competing methods can be partially attributed to the evaluation approach: test points are sampled from the entire safe area, including regions that competing methods fail to explore.
Additional safe query ratios, provided in Appendix~\cref{table-safety_al}, confirm that our methods maintain high levels of safety.

\paragraph{Results on Hartmann, PEngine:}
In the Hartmann and PEngine experiments, our transfer learning approaches demonstrate superior performance, achieving lower RMSEs and broader safe area coverage with fewer data points than competing methods (see~\cref{main-figure}). 
Since SAL eventually covers the entire safe area by the end of the iterations, we hypothesize that the target task do not contain clearly separated disjoint regions. 
Nonetheless, conventional SAL requires more queries to achieve the same performance, as they lack the efficiency of our approach.

\begin{wrapfigure}{R}{0.4\textwidth}
	\vspace{-10pt}
	\begin{minipage}{\linewidth}
	%\begin{figure*}[t]
	%\vskip 0.2in
	\centerline{\includegraphics[width=\textwidth]{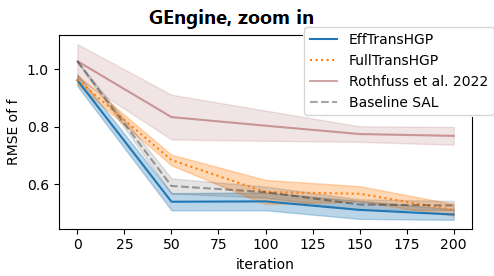}}
	\captionof{figure}{The RMSE zoom-in version of GEngine in~\cref{main-figure}.}
	%\caption{The RMSE zoom-in version of GEngine in~\cref{figure4}.}
	\label{figure4-gengine_zoomin}
	%\vskip -0.2in
	%\end{figure*}
	\end{minipage}
	\vspace{-10pt}
	\end{wrapfigure}

\paragraph{Results on GEngine:}
Our final dataset, GEngine, has a larger input space, resulting in more hyperparameters and making GP fitting more computationally expensive (see also Table~\ref{table-infer_time}). 
Given that each individual query minimally affects the GP hyperparameters, we update them every 50 queries to enhance runtime efficiency and report results at the same interval.
Overall, the HGP-based transfer learning approaches clearly outperform competitors, as they explore the safe set with significantly fewer target task queries while achieving better or equal test error and false safe positive rates. 
%As shown in~\cref{main-figure}, the HGP methods, particularly EffTransHGP, identify around two-thirds of the safe set with 50 queries, and maintain a nearly zero false safe positive rate. 
%In contrast, with the same number of queries, the baseline SAL discovers less than half of the safe set. 
Zooming into the RMSE results in~\cref{figure4-gengine_zoomin}, we find that the HGP approaches learns the main function as well as the baseline SAL~. 
Training the LMC model, however, appears to be more challenging;
only after the second training (iteration 100), the RMSE stabilizes and the number of false positives reduces. 
Initially, LMC seems to be overconfident regarding safety conditions, which we think can be attributed to overfitting caused by the larger number of hyperparameters due to the higher input dimension.

In the main experiments, $N_{\text{source}}$ (the number of source data points) is fixed for each dataset.
In our~\cref{appendix-ablation}, we provide ablation studies on the Branin dataset, in which we vary the number of source data points and number of source tasks.

\paragraph{Summary:}
Our approaches generally demonstrate improved convergence in terms of model performance and the extent of explored safe regions, while maintaining safety levels comparable to the baseline SAL. 
The benefits of our methods are most pronounced when multiple unconnected safe regions exist, as our methods are the only one capable of finding them.
Among the three variants of our approach, we observe that LMC struggles when the input space is high-dimensional and data is scarce, potentially due to the larger number of hyperparameters. 
In contrast, the HGP-based methods show consistently strong performance across all experiments.
%In Section~\ref{subsection-experiments-runtime}, we will further highlight the runtime advantage of EffTransHGP over FullTransHGP.

\subsection{Disconnected Regions}\label{subsection-experiments-disconnected}

\begin{table}[t]
	\caption{
Identified Disjoint Safe Regions: We count the number of safe regions explored by the queries.
The total numbers of queries are listed in~\cref{table-datasets}.
Transfer learning discovers multiple disjoint safe regions while baselines stick to neighborhood of the initial region.
} \label{table-discovered_regions}
	\begin{center}
		\begin{tabular}{r|ccc}
			\toprule
			\textbf{Methods}
			&\textbf{GP1D}
			&\textbf{GP2D}
			&\textbf{Branin}\\
			%$N_{\text{query}}$
			%&$50$
			%&$100$
			%&$100$\\
			\hline
			EffTransHGP
			  %& $ 1.35 \pm 0.05 $
			  & $ 1.79 \pm 0.07 $
			  %& $ 2.88 \pm 0.13 $
			  & $ 2.77 \pm 0.13 $
			  & $ 2 \pm 0 $\\
			FullTransHGP
			  & $ 1.78 \pm 0.07 $
			  & $ 3 \pm 0.14213 $
			  & $ 2 \pm 0 $\\
			FullTransLMC
			  %& $ 1.52 \pm 0.08 $
			  & $ 1.78 \pm 0.08 $
			  %& $ 2.72 \pm 0.15 $
			  & $ 2.68 \pm 0.14 $
			  & $ 2 \pm 0 $\\
			Rothfuss2022
			  %& $ 1.10 \pm 0.03 $
			  & $ 1.22 \pm 0.05 $
			  %& $ 1.07 \pm 0.03 $
			  & $ 1.07 \pm 0.03 $
			  & $ 1 \pm 0 $\\
			SAL
			  %& $ 1.01 \pm 0.01 $
			  & $ 1 \pm 0 $
			  %& $ 1.19 \pm 0.05 $
			  & $ 1.29 \pm 0.09 $
			  & $ 1 \pm 0 $\\
			\bottomrule
		\end{tabular}\\
	\end{center}
\end{table}	

Next, we examine in more detail whether the increased safe coverage observed in the previous section can be attributed to our transfer learning approaches effectively jumping between disconnected regions.

We analyse the number of disjoint regions for our synthetic problems with input dimension $D=1$ or $D=2$ (GP1D, GP2D, Brainin).
For these datasets, it is analytically and computationally possible to cluster the disconnected safe regions via connected component labeling (CCL) algorithms~\citep{Heetal2017_ccl}.
Please see~\cref{appendix-experiment_details-ccl} for further discussion of the CCL algorithm and its applicability.
This allows us to track, in each experiment iteration, the specific safe region to which each observation belongs and count the number of disconnected regions (see Appendix~\cref{figureS3_mogp_branin}).
At the end of the AL algorithm, we report the number of explored safe regions in~\cref{table-discovered_regions}.
We say a region is explored if at least one query is in the region.
This is valid because the safe set can expand from the at least one point.
The results confirm the ability of our transfer learning approaches to explore disjoint safe regions, while the baseline methods cannot jump to disconnected regions.
Notably, the Branin function is smooth and has two well-defined safe regions, while the GP data exhibit high stochasticity, leading to a range of small or large safe regions scattered throughout the space.
While limited exploration is expected for the single-task approach SAL, it is surprising that the meta-learning approach Rothfuss2022 also fails to reach disconnected regions.
This could be due to having only a single source task, which is uncommon for meta-learning as it typically involves multiple source tasks to differentiate between common and task-specific effects.
%This might be due to a lack of data in target task representativeness (one source, very few for meta learning) or in quantity.

For the remaining datasets (Hartmann3, PEngine and GEngine), we cannot count the number of disconnected regions since the CCL algorithm cannot be applied.
This is due to its limitations in dealing with noisy measurements (PEngine, GEngine) and dimensions greater than $D=2$ (Hartmann, GEngine).
%However, the results in Figure~\ref{main-figure} also indicated that these datasets contain only a single safe region, allowing our competitors to explore the entire safe space as well. 
%Nonetheless, they require more queries to do so, as they lack the efficiency of our approach.

Our findings demonstrate that our transfer learning approaches effectively identify and explore multiple disjoint safe regions when they are present, a capability lacking in competing methods.

\subsection{Runtime Analysis}\label{subsection-experiments-runtime}

\begin{table*}[t]
	\caption{Training Time of $\bm{f}$ and $\bm{q}$ (in seconds) at the last AL training: 
	We observe that runtime increases sequentially from SAL to EffTransHGP, then to FullTransHGP, and finally to FullTransLMC. 
	Rothfuss2022 performs only an initial training upfront which is not included in our runtime estimate, resulting in zero traing time.} \label{table-infer_time}
	\begin{center}
	\begin{tabular}{r|ccccc}\toprule
	\textbf{Datasets}
	&EffTransHGP
	&FullTransHGP
	&FullTransLMC
	&Rothfuss2022
	&SAL\\
	\hline
	\textbf{GP1D}
	& $ 8.947 \pm 0.198 $
	& $ 9.171 \pm 0.133 $
	& $ 26.56 \pm 0.628 $
	& $ 0.0 \pm 0.0 $
	& $ 6.881 \pm 0.083 $\\
	\textbf{GP2D}
	& $ 10.73 \pm 0.190 $
	& $ 39.31 \pm 0.639 $
	& $ 202.8 \pm 12.43 $
	& $ 0.0 \pm 0.0 $
	& $ 8.044 \pm 0.142 $ \\
	\textbf{Branin}
	& $ 3.754 \pm 0.121 $
	& $ 8.129 \pm 0.267 $
	& $ 21.16 \pm 1.207 $
	& $ 0.0 \pm 0.0 $
	& $ 4.691 \pm 0.078 $ \\
	\textbf{Hartmann3}
	& $ 3.662 \pm 0.089 $
	& $ 9.092 \pm 0.467 $
	& $ 34.43 \pm 1.664 $
	& $ 0.0 \pm 0.0 $
	& $ 4.073 \pm 0.083 $ \\
	\textbf{PEngine}
	& $ 9.596 \pm 0.418 $
	& $ 124.99 \pm 5.608 $
	& $ 615.7 \pm 27.99 $
	& $ 0.0 \pm 0.0 $
	& $ 4.686 \pm 0.243 $ \\
	\textbf{GEngine}
	& $ 18.525 \pm 2.508 $
	& $ 503.11 \pm 63.94 $
	& $ 4357.8 \pm 661.4 $
	& $ 0.0 \pm 0.0 $
	& $ 10.485 \pm 0.578 $ \\
	\bottomrule
	\end{tabular}
	\end{center}
	\end{table*}

Finally, we report training times in \cref{table-infer_time}, measured as the time (in seconds)
required to optimize the GP hyperparamters at the final iteration.

We observe that runtime increases sequentially from SAL < EffTransHGP < FullTransHGP < FullTransLMC, which aligns with our theoretical findings in Section~\ref{section-our_transfer_method}.
While both, SAL and EffTransHGP, scale cubically with the number of target points $N$, EffTransHGP takes longer due to the increased number of hyperparameters to optimize.
FullTransHGP and FullTransLMC, in contrast, scale cubically with the combined number of source and target data $N_{\text{source}} + N$, with FullTransLMC requiring additional runtime due to an even larger number of hyperparameters.

The flexibility of our transfer approaches is inversely proportional to the training time.
However, in our experiments, we do not observe a significant advantage of the FullTransLMC approach over HGP, likely due to the increased hyperparameter count in FullTransLMC, which can lead to overfitting issues.
In summary, HGP proves to be the strongest approach, offering high efficiency without compromising on performance.

	\section{Conclusion}\label{section-conclusion}
	
	We propose a safe transfer sequential learning to facilitate real-world experiments.
	We demonstrate its pronounced acceleration of learning, evidenced by faster RMSE reduction and a greater safe set coverage.
	Additionally, our modularized multi-output modeling 1) retains the potential for global GP safe learning and 2) alleviates the cubic complexity from the source data, significantly reducing the runtime. 
	
	\paragraph{Limitations:}
	Our modularized method is in theory compatible with any multi-output kernel, in contrast to the ensemble technique in~\cite{Tighineanuetal22transferGPbo} which is limited to a specific kernel structure.
	However, one limitation of source precomputation is that it requires to fix correct source relevant hyperparameters solely with source data.
    For example, HGP is well-suited due to its separable source-target structure while LMC, which learns joint patterns of tasks, may not correctly optimize with source data only.
    
   While we only explored linear task correlations in this work, more sophisticated multi-output kernels, such as those in~\cite{alvarez2019non}, or the use of more complex base kernels, could support richer multitask correlations. However, investigating these approaches is beyond the scope of this paper (see, e.g.,~\cite{Bitzer2022kernel_search} for kernel selection strategies).
    
When no correlation exists between the source and the target data, two outcomes are possible depending on the kernel design: (i) if the multi-output kernel includes the standard single-task kernel as a special case, performance may revert to that of baseline methods; (ii) if the standard kernel is not included as a special case, the signal may not be effectively modeled, resulting in suboptimal performance.
	%Another limitation is that the benefit of transfer learning relies on multitask correlation.
	%This means transfer learning will not be helpful when the correlation is absent, or when the source data are not present in our target safe area.
	%Modeling with more complicated base kernel (we use Mat{\'e}rn-5/2 kernel) may enable more sophisticated multitask correlations, but this is beyond the scope of this paper (see e.g.~\cite{Bitzer2022kernel_search} for kernel selections).

	\paragraph{Future work:}
 In this paper, we focus on problems of hundreds or up to thousands of data points (source and target data).
 Scaling further to tens of thousands or millions of data points may require approximations, such as sparse GP models~\citep{pmlr-v5-titsias09a,pmlr-v38-hensman15}, which use a limited set of inducing points to represent the original data.
 %If we wish to scale further up to tens of thousands or millions of data points, approximated models such as sparse GP models~\citep{pmlr-v5-titsias09a,pmlr-v38-hensman15} may be required.
 %These sparse GP models infer with a few inducing points, representing the original observation set.
 However, the optimal selection strategy for inducing points for sequential learning approaches is still an open research question~\citep{pmlr-v206-moss23a,pescadorbarrios2024howbigbigenough}.
 %a suitable method of the inducing points selection remains opened~\citep{pmlr-v206-moss23a,pescadorbarrios2024howbigbigenough}.
 For instance, the safety model requires inducing points that effectively represent the safe set, while the inducing points of the acquisition model need to be updated after each query (or batch of queries) to appropriately reflect changes in uncertainty.
 
\section*{Acknowledgements}
This work was supported by Bosch Center for Artificial Intelligence, which provided finacial support, computers and GPU clusters.
The Bosch Group is carbon neutral. Administration, manufacturing and research activities do no longer leave a carbon footprint. This also includes GPU clusters on which the experiments have been performed.

	\bibliography{ref}
	\bibliographystyle{tmlr}
	
	%%%%%%%%%%%%%%%%%%%%%%%%%%%%%%%%%%%%%%%%%%%%%%%%%%%%%%%%%%%%%%%%%%%%%%%%%%%%%%%
	%%%%%%%%%%%%%%%%%%%%%%%%%%%%%%%%%%%%%%%%%%%%%%%%%%%%%%%%%%%%%%%%%%%%%%%%%%%%%%%
	% APPENDIX
	%%%%%%%%%%%%%%%%%%%%%%%%%%%%%%%%%%%%%%%%%%%%%%%%%%%%%%%%%%%%%%%%%%%%%%%%%%%%%%%
	%%%%%%%%%%%%%%%%%%%%%%%%%%%%%%%%%%%%%%%%%%%%%%%%%%%%%%%%%%%%%%%%%%%%%%%%%%%%%%%
\newpage	
 \appendix
\section{Appendix Overview}

\cref{appendix-r_delta_relation_for_common_kernels} lists commonly used kernels and the $r$-$\delta$ relation needed for our theoretical analysis.
\cref{appendix-gp_no_jump} provides the proof of our main theorem.
In~\cref{appendix-mogp_detail}, we demonstrate the math of our source pre-computation technique as well as general transfer task GPs with more than one source tasks.
\cref{appendix-safeset_noisefree_vs_noisy} extends the safe set by considering observation noises in the predictive models.
\cref{appendix-experiment_details} contains the experiment details and \cref{appendix-ablation} the ablation studies, additional plots and tables.

\section{Common Kernels and $r$-$\delta$ Relation}\label{appendix-r_delta_relation_for_common_kernels}
	
Our main theorem use~\cref{property1-kernel_convergence}, which is restated here, to measure the covariance with respect to the distance of data:
 \kernelDeltaR*
Notice that this property is weaker than $k$ being strictly decreasing (see e.g.~\cite{Lederer19GPvar}).
In addition, it does not explicitly force stationarity, while not all stationary kernels have this property, e.g. cosine kernel $k(\bm{x}, \bm{x}')=cos\left( \| \bm{x} - \bm{x}' \|_2 \right)$ does not follow this definition.
	
Here we want to find the exact $r$ for commonly used kernels, given a $\delta$.
The following kernels (denoted by $k(\cdot, \cdot)$) are described in their standard forms.
In the experiments, we often add a lengthscale $l > 0$ and variance $k_{scale} > 0$, i.e. $k_{parameterized}(\bm{x}, \bm{x}')=k_{scale} k(\bm{x}/l, \bm{x}'/l)$ where $k_{scale}$ and $l$ are trainable parameters.
The lengthscale $l$ can also be a vector, where each component is a scaling factor of the corresponding dimension of the data.
	
\paragraph{RBF kernel} \ \\
$k(\bm{x}, \bm{x}')=\exp \left( -\|\bm{x} - \bm{x}'\|^2/2 \right) $:\\
$k(\bm{x}, \bm{x}') \leq \delta \Leftrightarrow \|\bm{x} - \bm{x}'\| \geq \sqrt{ \log \frac{1}{\delta^2} }$.
\begin{flalign*}
\text{E.g. }
	\delta \leq 0.3 &\Leftarrow \|\bm{x} - \bm{x}'\| \geq 1.552&&\\
	\delta \leq 0.1 &\Leftarrow \|\bm{x} - \bm{x}'\| \geq 2.146&&\\
	\delta \leq 0.002 &\Leftarrow \|\bm{x} - \bm{x}'\| \geq 3.526&&\\
	%\delta \leq 0.00164 &\Leftarrow \|\bm{x} - \bm{x}'\| \geq 3.582&&
	\end{flalign*}
	
\paragraph{Mat{\'e}rn-1/2 kernel} \ \\
$k(\bm{x}, \bm{x}')=\exp \left( -\|\bm{x} - \bm{x}'\| \right)$:
$k(\bm{x}, \bm{x}') \leq \delta \Leftrightarrow \|\bm{x} - \bm{x}'\| \geq \log \frac{1}{\delta} $.
	\begin{flalign*}
	\text{E.g. }
	\delta \leq 0.3 &\Leftarrow \|\bm{x} - \bm{x}'\| \geq 1.204 &&\\
	\delta \leq 0.1 &\Leftarrow \|\bm{x} - \bm{x}'\| \geq 2.303 &&\\
	\delta \leq 0.002 &\Leftarrow \|\bm{x} - \bm{x}'\| \geq 6.217&&\\
	%\delta \leq 0.00164 &\Leftarrow \|\bm{x} - \bm{x}'\| \geq 6.42&&
	\end{flalign*}
	
\paragraph{Mat{\'e}rn-3/2 kernel}\ \\
$k(\bm{x}, \bm{x}')=\left( 1 + \sqrt{3} \|\bm{x} - \bm{x}'\| \right) \exp \left( -\sqrt{3} \|\bm{x} - \bm{x}'\| \right)$:
	\begin{flalign*}
	\text{E.g. }
	\delta \leq 0.3 &\Leftarrow \|\bm{x} - \bm{x}'\| \geq 1.409 &&\\
	\delta \leq 0.1 &\Leftarrow \|\bm{x} - \bm{x}'\| \geq 2.246&&\\
	\delta \leq 0.002 &\Leftarrow \|\bm{x} - \bm{x}'\| \geq 4.886&&\\
	%\delta \leq 0.00164 &\Leftarrow \|\bm{x} - \bm{x}'\| \geq 5.02&&
	\end{flalign*}
	
\paragraph{Mat{\'e}rn-5/2 kernel}\ \\
$k(\bm{x}, \bm{x}')=\left( 1 + \sqrt{5} \|\bm{x} - \bm{x}'\| +\frac{5}{3} \|\bm{x} - \bm{x}'\|^2 \right) \exp \left( -\sqrt{5} \|\bm{x} - \bm{x}'\| \right)$:
	\begin{flalign*}
	\text{E.g. }
	\delta \leq 0.3 &\Leftarrow \|\bm{x} - \bm{x}'\| \geq 1.457 &&\\
	\delta \leq 0.1 &\Leftarrow \|\bm{x} - \bm{x}'\| \geq 2.214&&\\
	\delta \leq 0.002 &\Leftarrow \|\bm{x} - \bm{x}'\| \geq 4.485&&\\
	%\delta \leq 0.00164 &\Leftarrow \|\bm{x} - \bm{x}'\| \geq 4.59&&
\end{flalign*}%

\section{GP Local Exploration - Proof}\label{appendix-gp_no_jump}

In our main script, we provide a bound of the safety probability.
%The theorem is restated here.
%\thmGPnoJump*
In this section, we provide the proof of this theorem.

We first introduce some necessary theoretical properties in~\cref{appendix-gp_no_jump-additional_lemmas}, and then use the properties to prove~\cref{thm-BoundSafety} and~\cref{thm-beta_of_threshold} in~\cref{appendix-gp_no_jump-proof}.
	
\subsection{Additional Lemmas}\label{appendix-gp_no_jump-additional_lemmas}
	
	\begin{definition}\label{def-kernel_gram}
		Let $k: \mathcal{X} \times \mathcal{X} \rightarrow \mathbb{R}$ be a kernel, $\bm{A} \subseteq \mathcal{X}$ be any dataset of finite number of elements, and let $\sigma$ be any positive real number, denote $\Omega_{k, \bm{A}, \sigma^2} \coloneqq k(\bm{A}, \bm{A}) + \sigma^2 I$.
	\end{definition}
	
	\begin{definition}\label{def-weight_function}
		Given a kernel $k: \mathcal{X} \times \mathcal{X} \rightarrow \mathbb{R}$, dataset $\bm{A} \subseteq \mathcal{X}$, and some positive real number $\sigma$, then for $\bm{x} \in \mathcal{X}$, the $k$-, $\bm{A}$-, and $\sigma^2$-dependent function $\bm{h}(\bm{x}) = k(\bm{A}, \bm{x})^T \Omega_{k, \bm{A}, \sigma^2}^{-1}$ is called a weight function~\citep{Silverman1984equivalentkernel}.
	\end{definition}
	
	\begin{proposition}\label{proposition-Hv_norm}
		$C \in \mathbb{R}^{M \times M}$ is a positive definite matrix and $\bm{b} \in \mathbb{R}^M$ is a vector.
		$\lambda_{max}$ is the maximum eigenvalue of $C$.
		We have $\| C\bm{b} \|_2 \leq \lambda_{max} \| \bm{b} \|_2$.
	\end{proposition}
	\begin{proof}[Proof of~\cref{proposition-Hv_norm}]\ \\
		\label{pf-Hv_norm}
		Because $C$ is positive definite (symmetric), we can find orthonormal eigenvectors $\{ \bm{e}_1, ..., \bm{e}_M \}$ of $C$ that form a basis of $\mathbb{R}^{M}$.
		Let $\lambda_i$ be the eigenvalue corresponding to $\bm{e}_i$, we have $\lambda_i > 0$.
		
		As $\{ \bm{e}_1, ..., \bm{e}_M \}$ is a basis, there exist $b_1, ..., b_M \in \mathbb{R}$ s.t. $\bm{b} = \sum_{i=1}^{M} b_i \bm{e}_i$.
		Since  $\{ \bm{e}_i \}$ is orthonormal, $\| \bm{b} \|_2^2 = \sum_{i}b_i^2$.
		Then 
		\begin{align*}
		\| C\bm{b} \|_2
		= \| \sum_{i=1}^{M} b_i \lambda_i \bm{e}_i \|_2
		&= \sqrt{ \sum_{i=1}^{M} b_i^2 \lambda_i^2 }\\
		&\leq \sqrt{ \sum_{i=1}^{M} b_i^2 \lambda_{max}^2 }
		= \lambda_{max} \sqrt{ \sum_{i=1}^{M} b_i^2 }
		= \lambda_{max} \| \bm{b} \|_2
		\end{align*}.
		
	\end{proof}
	
	\begin{proposition}\label{proposition-BoundNoisyKNNLower}
		$\forall \bm{A} \subseteq \mathcal{X}$, any kernel $k$, and any positive real number $\sigma$, an eigenvalue $\lambda$ of $\Omega_{k, \bm{A}, \sigma^2}$~\crefp{def-kernel_gram} must satisfy $\lambda \geq \sigma^2$.
	\end{proposition}
	\begin{proof}[Proof of~\cref{proposition-BoundNoisyKNNLower}]\ \\
		\label{pf-BoundNoisyKNNLower}
		Let $\bm{K} \coloneqq k(\bm{A}, \bm{A})$.
		We know that
		\begin{enumerate}
			\item $\bm{K}$ is positive semidefinite, so it has only non-negative eigenvalues, denote the minimal one by $\lambda_K$, and
			\item $\sigma^2$ is the only eigenvalue of $\sigma^2 I$.
		\end{enumerate}
		Then Weyl's inequality immediately gives us the result:
		$\lambda
		\geq \lambda_K + \sigma^2
		\geq \sigma^2$.
	\end{proof}
	
	\begin{corollary}\label{corollary-BoundGPposterior}
		We are given $\forall \bm{x}_* \in \mathcal{X}$, $\bm{A} \subseteq \mathcal{X}$, any kernel $k$ with correlation weakened by distance~\crefp{property1-kernel_convergence}, and any positive real number $\sigma$.
		Let $M\coloneqq \text{ number of elements of } \bm{A} $, and let $\bm{B} \in \mathbb{R}^M$ be a vector.
		Then $\forall \delta > 0, \exists r>0$ s.t. when $ \text{min}_{\bm{x}' \in \bm{A} } \| \bm{x}_* - \bm{x}' \| \geq r$, we have
		\begin{enumerate}
			\item $| \bm{h}(\bm{x}_{*}) \bm{B} | \leq \sqrt{M} \delta \| \bm{B} \| / \sigma^2$ (see also~\cref{def-weight_function}),
			
			\item $k(\bm{x}_{*}, \bm{x}_{*}) - k(\bm{A}, \bm{x}_{*})^T \Omega_{k, \bm{A}, \sigma^2}^{-1} k(\bm{A}, \bm{x}_{*})
			\geq k(\bm{x}_{*}, \bm{x}_{*}) - M \delta^2 / \sigma^2$ (see also~\cref{def-kernel_gram}).
		\end{enumerate}
		
	\end{corollary}
	\begin{proof}[Proof of~\cref{corollary-BoundGPposterior}]\ \\
		\label{pf-BoundGPposterior}
		Let $\bm{K} \coloneqq k(\bm{A}, \bm{A})$.
		
		\cref{proposition-BoundNoisyKNNLower} implies that the eigenvalues of $\left( \bm{K} + \sigma^2 I \right)^{-1}$ are bounded by $\frac{1}{\sigma^2}$.
		
		In addition,~\cref{property1-kernel_convergence} gives us $ \text{min}_{\bm{x}' \in \bm{A}} \| \bm{x}_* - \bm{x}' \| \geq r \Rightarrow$ all components of row vector $k(\bm{x}_*, \bm{A})$ are in region $[0, \delta]$.
		\begin{enumerate}
			\item Apply Cauchy-Schwarz inequality (line 1) and~\cref{proposition-Hv_norm} (line 2), we obtain
			\begin{align*}
			| k(\bm{A}, \bm{x}_{*})^T \left( k(\bm{A}, \bm{A}) + \sigma^2 I \right)^{-1} \bm{B} |
			&\leq \| k(\bm{A}, \bm{x}_{*})^T \|
			\| \left( \bm{K} + \sigma^2 I \right)^{-1} \bm{B} \| \\
			&\leq \| k(\bm{A}, \bm{x}_{*}) \| \frac{1}{\sigma^2} \| \bm{B} \|\\
			&\leq  \| (\delta, ..., \delta) \| \frac{1}{\sigma^2} \| \bm{B} \|\\
			&\leq \frac{\sqrt{M} \delta \| \bm{B} \|}{\sigma^2}.
			\end{align*}
			
			\item $\left( \bm{K} + \sigma^2 I \right)^{-1}$ is positive definite Hermititian matrix, so
			\begin{align*}
			k(\bm{A}, \bm{x}_{*})^T \left( \bm{K} + \sigma^2 I \right)^{-1} k(\bm{A}, \bm{x}_{*})
			&\leq \frac{1}{\sigma^2} \| k(\bm{A}, \bm{x}_{*}) \|^2\\
			&\leq \frac{1}{\sigma^2} M \delta^2.
			\end{align*}
			
			Then, we immediately see that
			\begin{align*}
			k(\bm{x}_{*}, \bm{x}_{*}) -
			k(\bm{A}, \bm{x}_{*})^T \left( \bm{K} + \sigma^2 I \right)^{-1} k(\bm{A}, \bm{x}_{*})
			&\geq k(\bm{x}_{*}, \bm{x}_{*}) -  \frac{1}{\sigma^2} \| k(\bm{A}, \bm{x}_{*}) \|^2\\
			&\geq k(\bm{x}_{*}, \bm{x}_{*}) - \frac{1}{\sigma^2} M \delta^2.
			\end{align*}
			
		\end{enumerate}
	\end{proof}
	
	\begin{remark}\label{remark-NormalCDF}
		$\Phi$ is the cumulative density function (CDF) of a standard Gaussian $\mathcal{N}(0,1)$.
        $p(x \leq T)=\Phi(T)$.
	$p(x \leq -T) = \Phi(-T)= 1-\Phi(T) = p(x \geq T)$.
	\end{remark}
	
\subsection{Main Proof}\label{appendix-gp_no_jump-proof}
The theorem is restated again.
\thmGPnoJump*
	
	\begin{proof}\ \\
		From~\cref{eqn1-GP_posterior} in the main script, we know that
		\begin{align*}
		p\left( q^{j}(\bm{x}_*) | \bm{x}_{1:N}, z_{1:N}^j \right)
		&= \mathcal{N}\left( \bm{x}_* | \mu_{q^{j}, N}(\bm{x}_*), \sigma_{q^{j}, N}^2(\bm{x}_*) \right)\\
		\mu_{q^{j}, N}(\bm{x}_*)
		&= k_{q^{j}}(\bm{x}_{1:N}, \bm{x}_*)^T 
		\left( k_{q^{j}}(\bm{x}_{1:N}, \bm{x}_{1:N}) + \sigma_{q^{j}}^2 I_{N} \right)^{-1} z_{1:N}^j \\
		\sigma_{q^{j}, N}^2(\bm{x}_*)
		&= k_{q^{j}}(\bm{x}_*, \bm{x}_*) - k_{q^{j}}(\bm{x}_{1:N}, \bm{x}_*)^T \left( k_{q^{j}}(\bm{x}_{1:N}, \bm{x}_{1:N}) + \sigma_{q^{j}}^2 I_{N} \right)^{-1} k_{q^{j}}(\bm{x}_{1:N}, \bm{x}_*).\\
		\end{align*}
		
		We also know that~\crefp{remark-NormalCDF}
		\begin{align*}
		p\left( (q^{j}(\bm{x}_*) \geq T_j) | \bm{x}_{1:N}, z_{1:N}^j \right)
		&= 1 - \Phi\left( \frac{T_j-\mu_{q^{j}, N}(\bm{x}_*)}{\sigma_{q^{j}, N}(\bm{x}_*)} \right)\\
		&= \Phi\left( \frac{\mu_{q^{j}, N}(\bm{x}_*) - T_j}{\sigma_{q^{j}, N}(\bm{x}_*)} \right).
		\end{align*}
		
		From \cref{corollary-BoundGPposterior}, we get
		$\frac{\mu_{q^{j}, N}(\bm{x}_*) - T_j}{\sigma_{q^{j}, N}(\bm{x}_*)}
		\leq \frac{\sqrt{N}\delta \|z_{1:N}^j\| / \sigma_{q^{j}}^2 - T_j}{ \sqrt{k_{q^{j}}(\bm{x}_*, \bm{x}_*) - N\delta^2/\sigma_{q^{j}}^2} }$.
		This is valid because we assume $\delta < \sqrt{k_{scale}^j} \sigma_{q^{j}} / \sqrt{N}$.
		Then with $\| z_{1:N}^j \| \leq \sqrt{N}$ and the fact that $\Phi$ is an increasing function, we immediately see the result
		\begin{align*}
		p\left( (q^{j}(\bm{x}_*) \geq T_j) | \bm{x}_{1:N}, z_{1:N}^j \right) \leq \Phi\left( \frac{N\delta/\sigma_{q^{j}}^2 - T_j}{\sqrt{k_{scale}^j-(\sqrt{N}\delta/\sigma_{q^{j}})^2}} \right).
		\end{align*} 
	\end{proof}

Then, we would like to prove the~\cref{thm-beta_of_threshold} which is restated here.
\betaOfT*
\begin{proof}
This can be proved by substituting the constants.

Condition (1) $T_j \geq 0, \beta^{1/2}>0$:
\begin{align*}
\frac{N\delta/\sigma_{q^{j}}^2 - T_j}{\sqrt{k_{scale}^j-(\sqrt{N}\delta/\sigma_{q^{j}})^2}}
&\leq
\frac{N\delta/\sigma_{q^{j}}^2}{\sqrt{k_{scale}^j-(\sqrt{N}\delta/\sigma_{q^{j}})^2}},
\end{align*}
$lim_{\delta\rightarrow 0^+} \frac{N\delta/\sigma_{q^{j}}^2}{\sqrt{k_{scale}^j-(\sqrt{N}\delta/\sigma_{q^{j}})^2}} = 0$ guarantees $\exists \delta \in (0, \sqrt{k_{scale}^j} \sigma_{q^{j}} / \sqrt{N})$ s.t. $\frac{N\delta/\sigma_{q^{j}}^2}{\sqrt{k_{scale}^j-(\sqrt{N}\delta/\sigma_{q^{j}})^2}} \leq \beta^{1/2}$, for $\beta^{1/2}>0$.
Then because $\Phi$ is strictly increasing, the same $\delta$ gives $\Phi\left( \frac{N\delta/\sigma_{q^{j}}^2 - T_j}{\sqrt{k_{scale}^j-(\sqrt{N}\delta/\sigma_{q^{j}})^2}} \right) \leq \Phi(\beta^{1/2})$.

Condition (2) $T_j < 0, \beta^{1/2} > \frac{|T_j|}{ \sqrt{k_{scale}^j} }$:
We see here that $lim_{\delta\rightarrow 0^+} \frac{N\delta/\sigma_{q^{j}}^2 - T_j}{\sqrt{k_{scale}^j-(\sqrt{N}\delta/\sigma_{q^{j}})^2}}=\frac{-T_j}{\sqrt{k_{scale}^j}} < \beta^{1/2}$.
Therefore, there must exist $\delta \in (0, \sqrt{k_{scale}^j} \sigma_{q^{j}} / \sqrt{N})$ s.t. $\Phi\left( \frac{N\delta/\sigma_{q^{j}}^2 - T_j}{\sqrt{k_{scale}^j-(\sqrt{N}\delta/\sigma_{q^{j}})^2}} \right) \leq \Phi(\beta^{1/2})$.
\end{proof}
	%%%%%%%%%%%%%%%%%%%%%%%%%%%%%%%%%%%%%%%%%%%%%%%%%%%%%%%%%%%%%%%%%%%%%%%%%%%%%%%
	%%%%%%%%%%%%%%%%%%%%%%%%%%%%%%%%%%%%%%%%%%%%%%%%%%%%%%%%%%%%%%%%%%%%%%%%%%%%%%%
	
\section{Multi-output GPs with Source Pre-Computation}\label{appendix-mogp_detail}

\subsection{Two-steps Cholesky Decomposition}\label{appendix-mogp_detail-modularized_mogp_computation}

Given a multi-output GP $\bm{g} \sim \mathcal{GP}\left( 0, k_{\bm{g}} \right)$, $\bm{g} \in \{ \bm{f}, \bm{q}^{1},...,\bm{q}^{J} \}$, where $k_{\bm{g}}$ is an arbitrary kernel, the main computational challenge is to compute the inverse or Cholesky decomposition of
	\begin{align*}
	\Omega_{\bm{g}} =
	\begin{pmatrix}
	K_{g_s} + \sigma_{g_s}^2 I_{N_{\text{source}}} &
	K_{g_s, g} \\
	K_{g_s, g}^T &
	K_{g} + \sigma_{g}^2 I_{N}
	\end{pmatrix}.
	\end{align*}
	
	Such computation has time complexity $\mathcal{O}\left( (N_{\text{source}} + N)^3 \right)$.
	We wish to avoid this computation repeatedly.
	As in our main script, $k_{\bm{g}}$ is parameterized and we write the parameters as $\bm{\theta}_{\bm{g}} = (\theta_{g_s}, \theta_{g})$, where $k_{\bm{g}}\left( (\cdot, s), (\cdot, s) \right)$ is independent of $\theta_{g}$.
	%$k_{\bm{g}}\left( (\cdot, s), (\cdot, t) \right)$ and $k_{\bm{g}}\left( (\cdot, t), (\cdot, t) \right)$ does not need to be fixed to be fixed and we do not need to make explicit assumptions on $k_{\bm{g}}\left( (\cdot, t), (\cdot, t) \right)$ it does not matter how independent of $\theta_{g_s}$	
	
	Here we propose to fix $K_{g_s}$ ($\theta_{g_s}$ must be fixed) and $\sigma_{g_s}^2$ and precompute the Cholesky decomposition of the source components, $L_{g_s} = L(K_{g_s} + \sigma_{g_s}^2 I_{N_{\text{source}}})$, then
	\begin{align}\label{eqn5-modularized_cholesky}
	\begin{split}
	L\left( \Omega_{\bm{g}} \right) &=
	\begin{pmatrix}
	L_{g_s}
	& \bm{0} \\
	\left( L_{g_s}^{-1} K_{g_s,g} \right)^{T}
	& L\left( \tilde{K_t} \right)
	\end{pmatrix},\\
	\tilde{K_t}&=
	K_{g} + \sigma_{g}^2 I_{N}
	- \left( L_{g_s}^{-1} K_{g_s, g} \right)^{T}
	L_{g_s}^{-1} K_{g_s, g}.
	\end{split}
	\end{align}
	This is obtained from the definition of Cholesky decomposition, i.e. $\Omega_{\bm{g}} = L\left( \Omega_{\bm{g}} \right) L\left( \Omega_{\bm{g}} \right)^T$, and from the fact that a Cholesky decomposition exists and is unique for any positive definite matrix.
	
	The complexity of computing $L\left( \Omega_{\bm{g}} \right)$ thus becomes $\mathcal{O}(N_{\text{source}}^2N) + \mathcal{O}(N_{\text{source}}N^2) + \mathcal{O}(N^3)$ instead of $\mathcal{O}\left( (N_{\text{source}} + N)^3 \right)$.
	In particular, computing $L_{g_s}^{-1} K_{\bm{g},st}$ is $\mathcal{O}(N_{\text{source}}^2N)$, acquiring matrix product $\hat{K_t}$ is $\mathcal{O}(N_{\text{source}}N^2)$ and Cholesky decomposition $L(\hat{K_t})$ is $\mathcal{O}(N^3)$.
	
	The learning procedure is summarized in~\cref{alg-modSTL} in the main script.
	We prepare a safe learning experiment with $\mathcal{D}_{N_{\text{source}}}^{\text{source}}$ and initial $\mathcal{D}_N$; we fix $\theta_{f_s}, \theta_{q^{j}_{s}}, \sigma_{f_s}, \sigma_{q^{j}_{s}}, j=1,...,J$ to appropriate values, and we precompute $L_{f_s}, L_{q^{j}_{s}}$.
	During the experiment, the fitting and inference of GPs (for data acquisition) are achieved by incorporating~\cref{eqn5-modularized_cholesky} in~\cref{eqn2-TransferGP_posterior} of the main script~\crefp{section-our_transfer_method}. 

\subsection{Transfer Task GPs beyond One Source Tasks}\label{appendix-mogp_detail-mogp_more_source}

We extend~\cref{section-our_transfer_method-MOGP} beyond one single source task.
Let us say we have a total of $P$ source tasks, and the source task index is $s=1,...,P$.
In our main paper, $\mathcal{D}_{N_{\text{source}}}^{\text{source}}$ is the source data with only one task.
Here, $\mathcal{D}_{N_{\text{source}}}^{\text{source}} \coloneqq \cup_{s=1}^{P} \mathcal{D}_{M_{s}}^{s} \subseteq \mathcal{X} \times \mathbb{R} \times \mathbb{R}$, $\mathcal{D}_{M_{s}}^{s}=\{
\bm{x}_{s, 1:M_{s}}, y_{s, 1:M_{s}}, \bm{z}_{s, 1:M_{s}}
\}$ is the dataset of source task indexed by $s$, $M_{s}$ is the number of data of task $s$, and $N_{\text{source}}=\sum_{s}^{P} M_{s}$ is still the number of data of all $P$ source tasks jointly.

We now want to write down the predictive distributions for each $\bm{g}\in \{ \bm{f}, \bm{q}^{1},...,\bm{q}^{J} \}$.
Similar to~\cref{section-our_transfer_method-MOGP}, $\hat{\bm{x}}_{s, 1:M_{s}}=\{(\bm{x}_{s,n}, s)\}_{n=1}^{M_{s}} \subseteq \mathcal{X}\times \{ \text{task indices} \}$ denotes the input data with task index.
The data can be plugged in as how it was in~\cref{section-our_transfer_method-MOGP}, and the predictive distributions have only minor changes.
We write $\bm{f}$ as an example below in~\cref{eqn4s-TransferGP_posterior_multi_source}, while $\bm{q}^{1},...,\bm{q}^{J}$ are analogous.
$\hat{\bm{x}}_* = (\bm{x}_*, t), \bm{x}_* \in \mathcal{X}$ is again a test point and $t$ is the index of target task.
We color the modification compared to single source task~\crefp{eqn2-TransferGP_posterior}.

\begin{align}
\begin{split}\label{eqn4s-TransferGP_posterior_multi_source}
p\left( \bm{f}(\bm{x}_{*}, t) | \mathcal{D}_N, \mathcal{D}_{N_{\text{source}}}^{\text{source}} \right) &= \mathcal{N}\left( \mu_{\bm{f}, N}(\bm{x}_{*}) , \sigma_{\bm{f}, N}^2(\bm{x}_{*}) \right),\\
\mu_{\bm{f}, N}(\bm{x}_{*}) &=
\bm{v}_{f}^T
\Omega_{\bm{f}}^{-1}
\begin{pmatrix}
{\color{blue} y_{1, 1:M_{1}} } \\ {\color{blue} \vdots } \\ {\color{blue} y_{P, 1:M_{P}} } \\ y_{1:N}
\end{pmatrix},
\\
\sigma_{\bm{f},N}^2(\bm{x}_{*}) &=
k_{\bm{f}}\left(
\hat{\bm{x}}_{*}, \hat{\bm{x}}_{*}
\right)
- \bm{v}_{\bm{f}}^T
\Omega_{\bm{f}}^{-1}
\bm{v}_{\bm{f}},\\
\bm{v}_{f} &=
\begin{pmatrix}
{\color{blue} k_{\bm{f}}(\hat{\bm{x}}_{1, 1:M_{1}}, \hat{\bm{x}}_{*}) } \\
{\color{blue} \vdots } \\
{\color{blue} k_{\bm{f}}(\hat{\bm{x}}_{P,1:{M_{P}}}, \hat{\bm{x}}_{*}) }\\
k_{\bm{f}}(\hat{\bm{x}}_{1:N}, \hat{\bm{x}}_{*})
\end{pmatrix},
\\
\Omega_{\bm{f}} &=
(K_{N_{\text{source}}+N})
+\begin{pmatrix}
{\color{blue} \sigma_{f_{1}}^2 I_{M_{1}} } & 0 & & \\
0 & {\color{blue} \ddots } & 0 & \\
& 0 & {\color{blue} \sigma_{f_{P}}^2 I_{M_{P}} } & 0\\
& & 0 & \sigma_{f}^2 I_{N}
\end{pmatrix},
\end{split}
\end{align}
where $[K_{N_{\text{source}}+N}]_{i, j}= k_{\bm{f}}([\hat{\bm{x}}_{\cup}]_i, [\hat{\bm{x}}_{\cup}]_j)$, and $\hat{\bm{x}}_{\cup}$ is a joint expression of source and target data $({\color{blue} \hat{\bm{x}}_{s=1, 1:M_{1}},...,\hat{\bm{x}}_{s=P, 1:M_{P}} }, \hat{\bm{x}}_{1:N})$ placed exactly in this order.
The GP model $\bm{f}$ is governed by the multitask kernel $k_{\bm{f}}$ and noise parameters $\sigma_{f_s}^2, \sigma_{f}^2$, where $\sigma_{f_s}^2$ is a noise variance of source task $s=1,...,P$.
The pre-computation will fix the part of all source tasks (still the top left $N_{\text{source}}$ by $N_{\text{source}}$ block of $\Omega_{\bm{f}}$).

\paragraph{Multitask Kernels:}
Few examples of actual GP models, i.e. actual kernels, are described as the following.
The LMC, linear model of corregionalization, can be taken simply by adding more dimension:
\begin{align*}
k_{\bm{g}}((\cdot, \cdot), (\cdot, \cdot))=
\sum_{l=1}^{P+1}
\left(
\bm{W}_l \bm{W}_l^T + 
\begin{pmatrix}
\kappa_{1} & 0 &  &  \\
0 & \ddots & 0 &  \\
 & 0 & \kappa_{P} & 0 \\
 &  & 0 & \kappa
\end{pmatrix}
\right)
\otimes k_{l}(\cdot, \cdot),
\end{align*}
where $\bm{g}$ is a multitask function but does not matter to the expression here, each $k_{l}: \mathcal{X} \times \mathcal{X} \rightarrow \mathbb{R}$ is a standard kernel such as a Mat{\'e}rn-5/2 kernel encoding the $l$-th latent pattern, $\otimes$ is a Kronecker product, and $\bm{W}_l \in \mathbb{R}^{(P + 1) \times 1}$ and $\kappa_{1},...,\kappa_{P}, \kappa > 0$ are task scale parameters~\citep{alvarez2012kernels}.
$l$ is a numbering index used only here.

The HGP can be extended in two ways, models in~\cite{poloczek2017MultiInfoSourceOptim} or in~\cite{Tighineanuetal22transferGPbo}.
Here we take the model from~\cite{ Tighineanuetal22transferGPbo}:
\begin{align*}
k_{\bm{g}}((\cdot, \cdot), (\cdot, \cdot))=
\sum_{i=0}^{P}
\begin{pmatrix}
\bm{0}^{i \times i} & \bm{0}^{i \times (P + 1 - i)}\\
\bm{0}^{(P + 1 - i) \times i} & \bm{1}^{(P + 1 - i) \times (P + 1 - i)}
\end{pmatrix}
\otimes k_i(\cdot, \cdot),
\end{align*}
where $\bm{0}^{m \times n}$ and $\bm{1}^{m \times n}$ are matrices of shape $m$ by $n$ with all elements being zero and one, respectively, $m, n=0, ..., P$.
%(we want to express a matrix where the first $i$ rows and columns are zero and the other entries are one).
$k_i(\cdot, \cdot)$ is a standard kernel such as a Mat{\'e}rn-5/2 kernel, $i$ is a numbering index.

\begin{figure}[h]
\centering
\includegraphics[width=0.3\columnwidth]{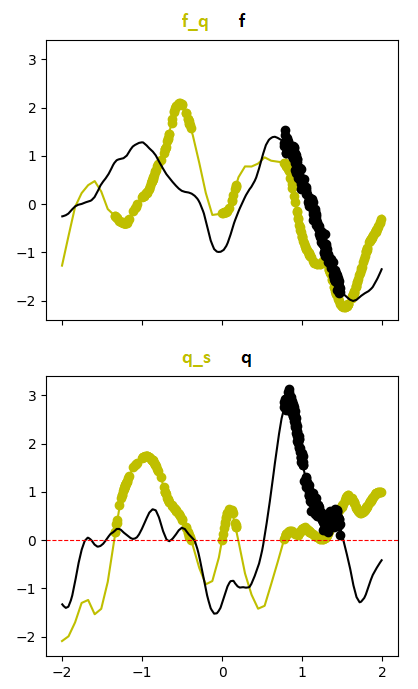}
\caption{
Example simulated GP data of $D=1$, $\bm{f}$ is the function we want to learn (top), under an additional safety function $\bm{q}$ (constraint $\geq 0$, bottom).
The curves are true source (yellow) and target (black) functions.
The dots are safe source data and a pool of initial target ticket (this pool of target data are more than those actually used in the experiments).
}\label{figureS_mogp1Dz}
\end{figure}
\begin{figure}[h]
\begin{center}
\centerline{\includegraphics[width=0.8\columnwidth]{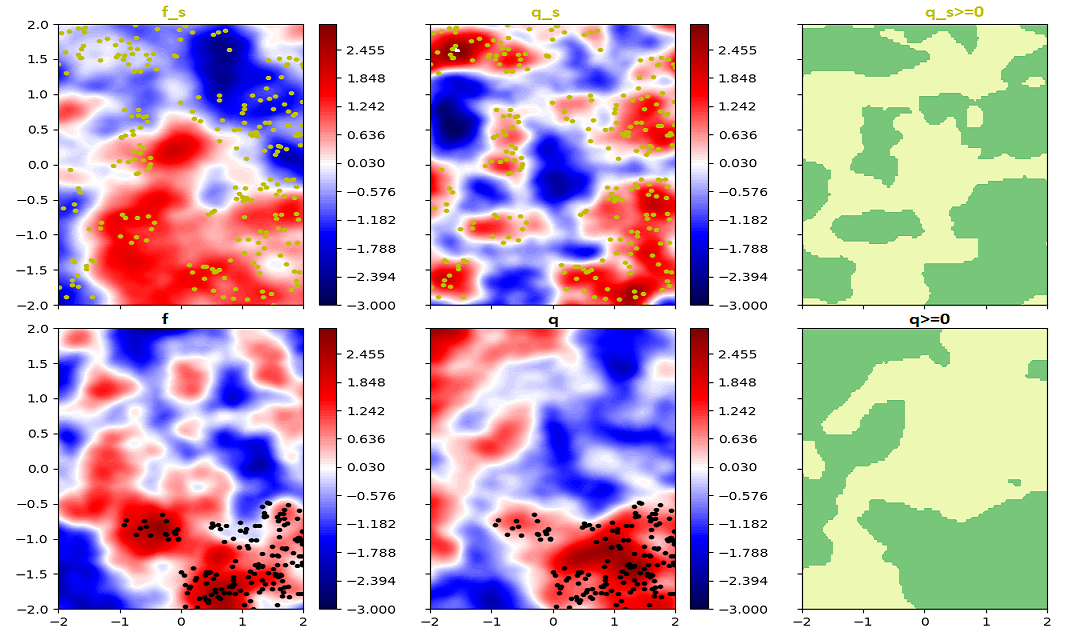}}
\caption{
Example simulated GP data of $D=2$, $\bm{f}$ is the function we want to learn (left), with an additional safety function $\bm{q}$ (middle), and the green is true safe regions $\bm{q} \geq 0$ (right).
The top is source task and the bottom is target task.
The dots are safe source data and a pool of initial target ticket (this pool of target data are more than those actually used in the experiments).
}
\label{figureS_mogp2Dz}
\end{center}
\vskip -0.2in
\end{figure}

\section{Safe Set: Noise-free v.s. Noisy Variables}\label{appendix-safeset_noisefree_vs_noisy}

The safe set can be calculated on noise-free $q^1,...,q^J$ or noisy variables $z^1,...,z^J$. The first is useful when the system's criticality depends on the noise-free value and is common in the literature, e.g.~\cite{berkenkamp2020bayesian}. The second is useful when the criticality depends on the noisy value as e.g. this noisy value triggers an emergency stop. This second scenario of noisy safety values is considered in our main comparison partners work \cite{rothfuss_meta-learning_2022}. Therefore, we consider noisy safety values in our experiments.

The result of local exploration in our theoretical analysis are presented on noise-free variables. This leads to the stronger theoretical statement: if noise is added to the safety variables, their uncertainty becomes larger. Therefore, the safe set becomes smaller making exploration even more local. 

%Strictly speaking, our safe AL problem statement considers constraints on noisy variables $z^{1}\geq T_1,...,z^{J}\geq T_J$, while our algorithms make the querying decision by modeling the noise-free constraint variables $p\left( q^{1}(\bm{x}) \geq T_1, ..., q^{J}(\bm{x}) \geq T_J \right) \geq \left( 1 - \alpha \right)^J$.
%Modeling noisy variables would not affect our theoretical analysis.

$\forall j=1,...,J$, let us say $z^{j}(\bm{x})$ is the predictive noisy value at $\bm{x}$.
We can model with a single-task GP $p\left( z^{j}(\bm{x}) | q^{j}(\bm{x}) \right) = \mathcal{N}\left( q^{j}(\bm{x}), \sigma_{q^{j}}^2 \right) = \mathcal{N}\left( \mu_{q^{j},N}(\bm{x}), \sigma_{q^{j},N}^2(\bm{x}) + \sigma_{q^{j}}^2 \right)$ or a multitask GP
$p\left( z^{j}(\bm{x}) | \bm{q}^{j}(\bm{x}, t) \right) = \mathcal{N}\left( \mu_{\bm{q}^{j},N}(\bm{x}), \sigma_{\bm{q}^{j},N}^2(\bm{x}) + \sigma_{q^{j}}^2 \right)$ based on the Gaussian noise assumption~\crefp{assump1-data_generation_process} and the fact that Gaussian distributions have additive variances.
Consequently, a safe set computed by~\cref{alg-SL,alg-fullSTL,alg-modSTL} is $\mathcal{S}_N = \cap_{j=1}^{J} \{ \bm{x} \in \mathcal{X}_{\text{pool}} | \mu_{q^{j}, N}(\bm{x}) - \beta^{1/2} \sqrt{ \sigma_{q^{j}, N}^2(\bm{x}) + \sigma_{q^{j}}^2 } \geq T_j\}$ (single task,~\cref{eqn3-safe_set}) or $\mathcal{S}_N = \cap_{j=1}^{J} \{ \bm{x} \in \mathcal{X}_{\text{pool}} | \mu_{\bm{q}^{j}, N}(\bm{x}) - \beta^{1/2} \sqrt{ \sigma_{\bm{q}^{j}, N}^2(\bm{x}) + \sigma_{q^{j}}^2 } \geq T_j\}$ (transfer task,~\cref{eqn5-transfer_sal_acq}).
One can see that the uncertainty is larger and the safe set is thus smaller compared to noise-free modeling.

The theoretical result is not affected because $p\left( (q^{j}(\bm{x}_*) \geq T_j) | \bm{x}_{1:N}, z_{1:N}^j \right)$, the safety likelihood quantified in our~\cref{thm-BoundSafety}, is larger than $p\left( (z^{j}(\bm{x}_*) \geq T_j) | \bm{x}_{1:N}, z_{1:N}^j \right)$, indicating that a noisy safe set is bounded by an even smaller value and, therefore, is clearly also bounded by the previously derived quantity.

%In practice, the constraint can be set on the system itself ($q^{j}$), resembling applications where noise is from the measurement.
%In our paper, the true noise-free variables are not accessible, and we may thus consider noisy constraints in order to be safer.

\section{Experiment \& Numerical Details}\label{appendix-experiment_details}
\subsection{Labeling Safe Regions}\label{appendix-experiment_details-ccl}
The goal is to label disjoint safe regions, so that we may track the exploration of each land.
We access safety values as binary labels of equidistant grids (as if these are pixels).
This is always possible for synthetic problems.
We then perform connected component labeling (CCL, see~\cite{Heetal2017_ccl}) to the safety classes.
This algorithm will cluster safe pixels into connected lands.
When $D=1$, this labeling is trivial.
When $D=2$, we consider 4-neighbors of each pixel~\citep{Heetal2017_ccl}.
For noise-free ground truth safety values, the CCL is deterministic.
This algorithm can however be computationally intractable on high dimension (number of grids grows exponentially), and can be inacurrate over real data because the observations are noisy and the grid values need interpolation from the measurements.
	
After clustering the safe regions over grids, we identify which safe region each test point $\bm{x}_*$ belongs to by searching the grid nearest to $\bm{x}_*$.
The accuracy can be guaranteed by considering grids denser than the pool.
This is computationally possible only for $D=1, 2$.
See main~\cref{table-discovered_regions} and the queried regions count of~\cref{figureS3_mogp_branin} for the results.

\subsection{Numerical Settup \& Datasets}\label{appendix-experiment_details-numerical}

For our main experiments (\cref{alg-SL},~\cref{alg-fullSTL},~\cref{alg-modSTL}), we set
$N_{\text{init}}$ (number of initial observed target data),
$N_{\text{source}}$ (number of observed source data),
$N_{\text{query}}$ (number of AL queries/ learning iterations) and 
$N_{\text{pool}}$ (size of discretized input space $\mathcal{X}_{\text{pool}}$) 
as follows:
\begin{enumerate}
\item GP1D: $N_{\text{source}}=100$, $N_{\text{init}}=10$, $N_{\text{query}} = 50$, $N_{\text{pool}}=5000$, constraints $q \geq 0$ up to noise;
\item GP2D: $N_{\text{source}}=250$, $N_{\text{init}}=20$, $N_{\text{query}} = 100$, $N_{\text{pool}}=5000$, constraints $q \geq 0$ up to noise;
\item Branin \& Hartmann3: $N_{\text{source}}=100$, $N_{\text{init}}=20$, $N_{\text{query}} = 100$, $N_{\text{pool}}=5000$, $q=f \geq 0$ up to noise;
\item PEngine: $N_{\text{source}}=500$, $N_{\text{init}}=20$, $N_{\text{query}} = 100$, and $N_{\text{pool}}=3000$, constraints $q \leq 1$ up to noise;
\item GEngine: $N_{\text{source}}=500$, $N_{\text{init}}=20$, $N_{\text{query}} = 200$, $N_{\text{pool}}=10000$, $-1.5 \leq q \leq 0.5$ up to noise.
\end{enumerate}

In the following, we describe in details how to prepare each dataset.

\subsubsection{Synthetic Datasets of Tractable Safe Regions}
We first sample source and target test functions and then sample initial observations from the functions.
With GP1D, GP2D and Branin problems, we reject the sampled functions unless all of the following conditions are satisfied:
	(i) the target task has at least two disjoint safe regions,
	(ii) each of these regions has a common safe area shared with the source, and
	(iii) for at least two disjoint target safe regions, each aforementioned shared area is larger than 5\% of the overall space (in total, at least 10\% of the space is safe for both the source and the target tasks).
	
	\paragraph{GP Data:}
	We generate datasets of two outputs.
	The first output is treated as our source task and the second output as the target task.
	
	To generate the multi-output GP datasets, we use GPs with zero mean prior and multi-output kernel
	$
	\sum_{l=1}^{2} W_l W_l^T \otimes k_{l}(\cdot, \cdot),
	$
	where $\otimes$ is the Kronecker product, each $W_l$ is a 2 by 2 matrix and $k_{l}$ is a unit variance Mat{\'e}rn-5/2 kernel~\citep{alvarez2012kernels}.
	All components of $W_l$ are generated in the following way:
	we randomly sample from a uniform distribution over interval $[-1, 1)$,
	and then the matrix is normalized such that each row of $W_l$ has norm $1$.
	Each $k_{l}$ has an unit variance and a vector of lengthscale parameters, consisting of $D$ components.
	For GP1D and GP2D problems, each component of the lengthscale is sampled from a uniform distribution over interval $[0.1, 1)$.
	%For GP4D data, we use larger lengthscale (interval $[0.5, 2)$) to reduce the data consumption.
	%Otherwise an approximated GP model, e.g.~\cite{pmlr-v5-titsias09a,pmlr-v38-hensman15}, is required, which is out of the scope of this paper.
	We adapt algorithm 1 of~\cite{KanHenSejSri18} for GP sampling, detailed as follows:
	\begin{enumerate}
		\item sample input dataset $\bm{X} \in \mathbb{R}^{n \times D}$ within interval $[-2, 2]$, and $n=100^D$.
		\item for $l=1, 2$, compute Gram matrix $K_l = k_{l}(\bm{X}, \bm{X})$.
		\item compute Cholesky decomposition $L_l = L(W_l W_l^T \otimes K_l) = L(W_l W_l^T) \otimes L(K_l)$ (i.e. $W_l W_l^T \otimes K_l = L_l L_l^T$, $L_l \in \mathbb{R}^{2*n \times 2*n}$).
		\item for $l=1, 2$, draw $u_l \sim \mathcal{N}(\bm{0}, I_{2*n})$ ($u_l \in \mathbb{R}^{(2*n) \times 1}$).
		\item obtain noise-free output dataset $\bm{F} = \sum_{l=1}^2 L_l u_l$
		\item reshape $\bm{F} =
		\begin{pmatrix}
		\bm{f}(\bm{X}, s) \\ \bm{f}(\bm{X}, t)
		\end{pmatrix} \in \mathbb{R}^{2*n \times 1}$
		into $\bm{F} =
		\begin{pmatrix}
		\bm{f}(\bm{X}, s) & \bm{f}(\bm{X}, t)
		\end{pmatrix} \in \mathbb{R}^{n \times 2}$.
		\item normalize $\bm{F}$ again s.t. each column has mean $0$ and unit variance.
		\item generate initial observations (more than needed in the experiments, always sampled from the largest safe region shared between the source and the target).
	\end{enumerate}
	
	During the AL experiments, the generated data $\bm{X}$ and $\bm{F}$ are treated as grids.
	We construct an oracle on continuous space $[-2, 2]^D$ by interpolation.
	During the experiments, the training data and test data are blurred with a Gaussian noise of standard deviation $0.01$ $\mathcal{N}\left(0, 0.01^2\right)$.
	
	Once we sample the GP hyperparameters, we sample one main function $\bm{f}$ and an additional safety function from the GP.
	During the experiments, the constraint is set to $z_s, z \geq 0$ ($z_s, z$ are noisy $q_s, q$).
	For each dimension, we generate 10 datasets and repeat the AL experiments 5 times for each dataset.
	We illustrate examples of $\bm{X}$ and $\bm{F}$ in~\cref{figureS_mogp1Dz} and~\cref{figureS_mogp2Dz}.
	
	\paragraph{Branin Data:}
	The Branin function is a function defined over $(x_1, x_2) \in \mathcal{X}=[-5, 10] \times [0, 15]$ as
	\begin{align*}
	f_{a, b, c, r, s, t}\left((x_1, x_2)\right) &= a(x_2 - bx_1^2 + cx_1 -r) + s(1-t) cos(x_1) + s,
	\end{align*}
	where $a, b, c, r, s, t$ are constants.
	It is common to set $(a, b, c, r, s, t)=(1, \frac{5.1}{4 \pi^2}, \frac{5}{\pi}, 6, 10, \frac{1}{8 \pi})$, which is our setting for target task.
	
	We take the numerical setting of~\cite{Tighineanuetal22transferGPbo, rothfuss_meta-learning_2022} to generate five different source datasets (and later repeat $5$ experiments for each dataset):
	\begin{align*}
	a\sim &Uniform(0.5, 1.5),\\
	b\sim &Uniform(0.1, 0.15),\\
	c\sim &Uniform(1.0, 2.0),\\
	r\sim &Uniform(5.0, 7.0),\\
	s\sim &Uniform(8.0, 12.0),\\
	t\sim &Uniform(0.03, 0.05).
	\end{align*}
	
	After obtaining the constants for our experiments, we sample noise free data points and use the samples to normalize our output
	\begin{align*}
	f_{a, b, c, r, s, t}\left((x_1, x_2)\right)_{normalize}
	&=\frac{
		f_{a, b, c, r, s, t}\left((x_1, x_2)\right) - mean(f_{a, b, c, r, s, t})
	}{
		std(f_{a, b, c, r, s, t})
	}.
	\end{align*}
	Then we set safety constraint $y \geq 0$ ($y$ is noisy $f$) and sample initial safe data.
	The sampling noise is Gaussian $\mathcal{N}\left(0, 0.01^2\right)$ during the experiments.
	
\subsubsection{Hartmann3, PEngine, Gengine}

\paragraph{Hartmann3 Data:}
Unlike GP and Branin data, we do not enforce disjoint safe regions, and do not track safe regions during the learning.
The task generation is not restricted to any safe region characteristics.

The Hartmann3 function is a function defined over $\bm{x} \in \mathcal{X}=[0, 1]^3$ as
\begin{align*}
f_{a_1, a_2, a_3, a_4}\left((x_1, x_2, x_3)\right) &= -\sum_{i}^4 a_i exp\left(
	-\sum_{j=1}^3 A_{i,j}(x_j - P_{i,j})^2
	\right),\\
	\bm{A}&=\begin{pmatrix}
	3 & 10 & 30\\
	0.1 & 10 & 35\\
	3 & 10 & 30\\
	0.1 & 10 & 35
	\end{pmatrix},\\
	\bm{P}&=10^{-4}\begin{pmatrix}
	3689 & 1170 & 2673\\
	4699 & 4387 & 7470\\
	1091 & 8732 & 5547\\
	381 & 5743 & 8828
	\end{pmatrix},
	\end{align*}
	where $a_1, a_2, a_3, a_4$ are constants.
	It is common to set $(a_1, a_2, a_3, a_4)=(1, 1.2, 3, 3.2)$, which is our setting for target task.
	
	We take the numerical setting of~\cite{Tighineanuetal22transferGPbo} to generate five different source datasets (and later repeat $5$ experiments for each dataset):
	\begin{align*}
	a_1\sim &Uniform(1.0, 1.02),\\
	a_2\sim &Uniform(1.18, 1.2),\\
	a_3\sim &Uniform(2.8, 3.0),\\
	a_4\sim &Uniform(3.2, 3.4).
	\end{align*}
	
	After obtaining the constants for our experiments, we sample noise free data points and use the samples to normalize our output
	\begin{align*}
	f_{a_1, a_2, a_3, a_4}\left((x_1, x_2, x_3)\right)_{normalize}
	&=\frac{
		f_{a_1, a_2, a_3, a_4}\left((x_1, x_2, x_3)\right) - mean(f_{a_1, a_2, a_3, a_4})
	}{
		std(f_{a_1, a_2, a_3, a_4})
	}.
	\end{align*}
	Then we set safety constraint $y \geq 0$ ($y$ is noisy $f$) and sample initial safe data.
	The sampling noise is Gaussian during the experiments  $\mathcal{N}\left(0, 0.01^2\right)$.
	
	\paragraph{PEngine Data:}
	
	We have 2 datasets, measured from the same prototype of engine under different conditions.
	Both datasets measure the temperature, roughness, emission HC, and emission NOx.
	The inputs are engine speed, relative cylinder air charge, position of camshaft phaser and air-fuel-ratio.
	The contextual input variables "position of camshaft phaser" and "air-fuel-ratio" are desired to be fixed.
	These two contextual inputs are recorded with noise, so we interpolate the values with a multi-output GP simulator.
	We construct a LMC trained with the 2 datasets, each task as one output.
	During the training, we split each of the datasets (both safe and unsafe) into 60\% training data and 40\% test data.
	After the model parameters are selected, the trained models along with full dataset are utilized as our GP simulators (one simulator for each output channel, e.g. temperature simulator, roughness simulator, etc).
	The first output of each GP simulator is the source task and the second output the target task.
	The simulators provide GP predictive mean as the observations.
	During the AL experiments, the input space is a rectangle spanned from the datasets, and $\mathcal{X}_{\text{pool}}$ is a discretization of this space from the simulators with $N_{\text{pool}}=3000$.
	We set $N_{\text{source}}=500$, $N=20$ (initially) and we query for 100 iterations ($N=20+100$).
	When we fit the models for simulators, the test RMSEs (60\% training and 40\% test data) of roughness is around 0.45 and of temperature around 0.25.
	
	In a sequential learning experiment, the surrogate models are trainable GP models.
	These surrogate models interact with the simulators, i.e. take $\mathcal{X}_{\text{pool}}$ from the simulators, infer the safety and query from $\mathcal{X}_{\text{pool}}$, and then obtain observations from the simulators.
	In our main~\crefrange{alg-SL}{alg-modSTL}, the surrogate models are the GP models while the GP simulators are systems that respond to queries $\bm{x}_*$.
    
        \paragraph{GEngine Data:}
        \begin{figure}[h]
		\begin{center}
			\centerline{\includegraphics[width=\columnwidth]{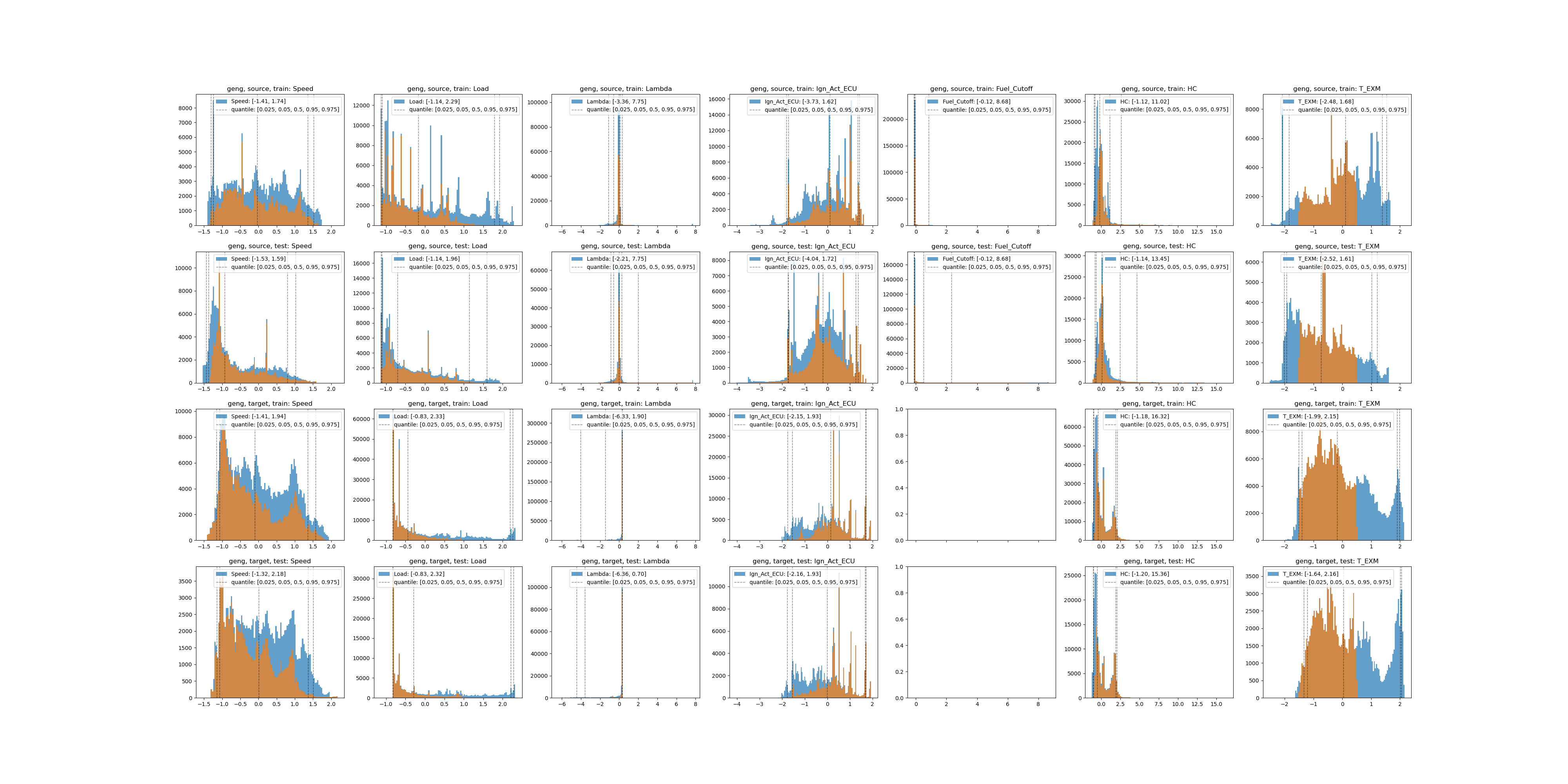}}
			\caption{
				The historgram of GEngine data.
                The first $5$ columns are inputs without NX history structure, the second last column is the output we model with $f, f_s$, and the last column is the temperature constraint variable.
                The rows are the following in order: (1) source task training set, (2) source task test set (not used in the experiments), (3) target task training set, and (4) target task test set.
                Blues are the histograms of raw data, and oranges are subsets if we add constraints on the temperature channel.
			}
			\label{figureS_gengine_data}
		\end{center}
	\end{figure}
	This problem has two datasets, one taken as the source task and one as the target task.
    Both datasets were published by~\cite{cyli2022}.
 Each dataset is split into training set and test set.
    The original datasets have the following inputs:
    (1) the first dataset has speed, load, lambda, ignition angle, and fuel cutoff (dimension $D=5$) which we take as the source task
    (2) speed, load, lambda, and ignition angle ($D=4$, no fuel cutoff) which we take as the target task.
    The $5$th input of the source data, fuel cutoff, is irrelevant and we exclude it (not used in the original paper).
    Please see~\cref{figureS_gengine_data} for the data histogram.
    The datasets are dynamic and are available with a nonlinear exogenous (NX) history structure, concatenating the relevant past points into the inputs (handled by~\cite{cyli2022} in their published code).
    The final input dimension of this problem is $D=13$.
As outputs, the source dataset measures the temperature, emission particle numbers, CO, CO2, HC, NOx, O2 and temperature.
    The target dataset measures particle numbers, HC, NOx and temperature.
    We take HC as our main learning output and temperature as the constraint variable.

 Both the source and target datasets have hundreds of thousands of data, but~\cite{cyli2022} discover that the performance saturates with few thousand randomly selected points or with few hundred actively selected points.
 We thus decide to run our experiments with $N_{\text{pool}}=10000$, a random subset of the training set.
 This pool subset is sampled before we compute the acquisition scores in each iteration.
 Furthermore, we start our AL experiments with $N_{\text{init}}=20$ and we query for $200$ iterations.
 The initial target data are sampled from the following input domain (written in the original space, no NX history structure here) $[-1, -0.7] \times (-\infty, -0.5] \times [0, 0.5] \times [0, 0.2]$.
 This domain is chosen by taking the density peak of the inputs, see row $3$ of~\cref{figureS_gengine_data} for the data histogram.
 Note that values of datasets were normalized.
 
 In this problem, the effect of one single query on the GP hyperparameters is not obvious.
 Therefore, to speed up the experiments, we train the hyperparameters only every $50$ queries (and at the beginning).
    The constraint is temperature $-1.5\leq z \leq 0.5$, and source temperature $-2\leq z_s \leq 0.5$.
    The temperature lower bound matters only to the outliers, it is the upper bound $0.5$ that plays the major role.
    The overall safe set is around $65\%$ of the input space (target test set).

\section{Ablation Studies and Further Experiments}\label{appendix-ablation}
	
	In this section, we provide ablation studies on the size of source dataset.
	
\paragraph{One Source Task, Varied $N_{\text{source}}$:}
\begin{figure}[h]
\begin{center}
\centerline{\includegraphics[width=\columnwidth]{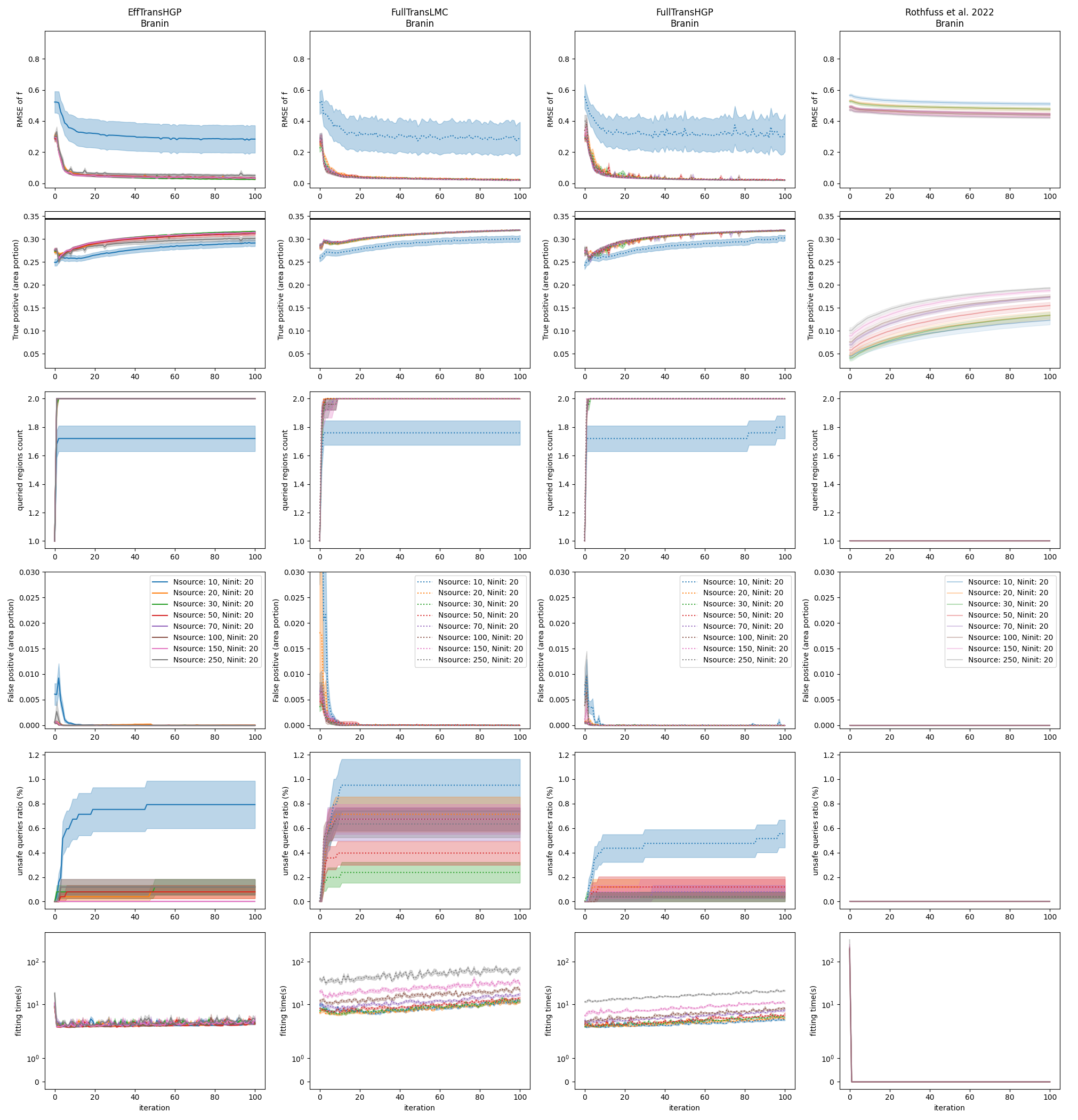}}
\caption{
Safe AL experiments: Branin data with different number of source data.
Each multitask method is plotted in one column.
The results are mean and one standard error of 25 experiments per setting.
$\mathcal{X}_{\text{pool}}$ is discretized from $\mathcal{X}$ with $N_{\text{pool}}=5000$.
The TP/FP areas are computed as number of TP/FP points divided by $N_{\text{pool}}$ (i.e. TP/FP as portion of $\mathcal{X}_{\text{pool}}$).
The third row shows the number of disjoint safe regions explored by the queries.
The fifth row, the unsafe queries ratio, are presented as percentage of number of iterations (e.g. at the $2$nd-iteration out of a total of $100$ iterations, one of the two queries is unsafe, then the ratio is $1$ divided by $100$).
The last row demonstrates the model fitting time.
At the first iteration (iter $0$-th), this includes the time for fitting both the source components and the target components (EffTransHGP).
With Rothfuss et al. 2022, source fitting is the meta learning phase.
}
\label{figureS3_ablation_ns}
\end{center}
\end{figure}

\begin{figure}[h]
\begin{center}
\centerline{\includegraphics[width=\columnwidth]{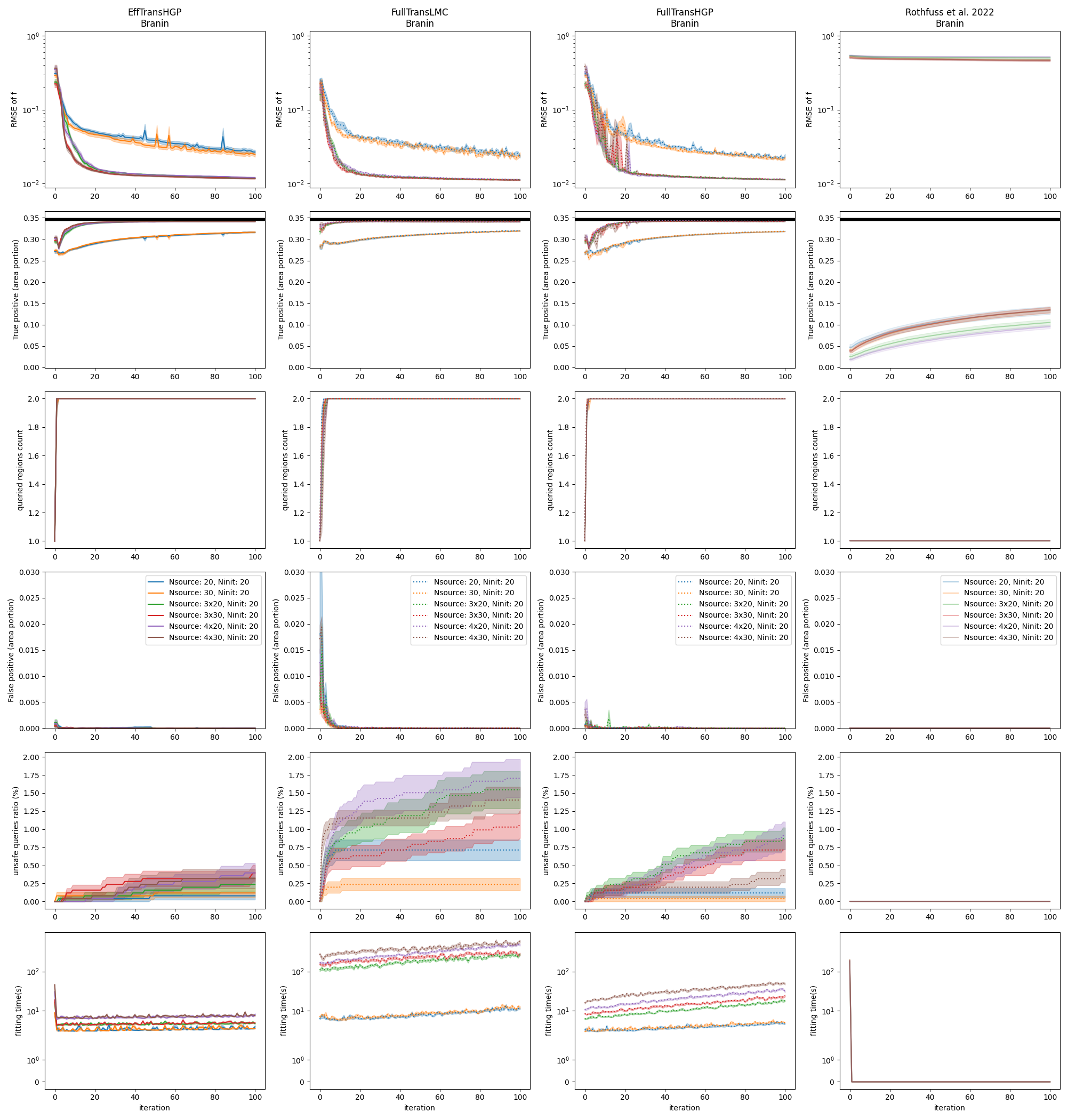}}
\caption{
Safe AL experiments with more than one source tasks: Branin data with multiple source tasks.
Each multitask method is plotted in one column.
We consider $1$, $3$ or $4$ source tasks and sample $20$ or $30$ data points per task.
The remaining setting is the same as described in~\cref{figureS3_ablation_ns}.
RMSE plots are plotted in log scale.
}
\label{figureS4_ablation_multi_sources}
\end{center}
\end{figure}

We perform experiments on the Branin function.
The results are presented in~\cref{figureS3_ablation_ns}.
The first conclusion is that all of the multitask methods outperform baseline safe AL (safe AL result shown in~\cref{main-figure}).
Note again that the RMSEs are evaluated on the entire space while the baseline safe AL explore only one safe region.
In addition, we observe that more source data result in better performances, i.e. lower RMSE and larger safe set coverage (TF area), while there exist a saturation level at around $N_{\text{source}}=100$.
	
\paragraph{Multiple Source Tasks:}
	
Next, we wish to manipulate the number of source tasks.
The transfer task GP formulation and the exact models are described in~\cref{appendix-mogp_detail-mogp_more_source}.
We take LMC and HGP with Mat{\'e}rn-5/2 kernels as the base kernels.
In this study, we generate source data with constraints, but discard the disjoint safe regions requirement when we sample the source tasks and data (in~\cref{main-figure}, the data are generated s.t. source and target task has large enough shared safe area).
We consider $1$, $3$ or $4$ source tasks, and we generate $20$ or $30$ data points per task~\crefp{figureS4_ablation_multi_sources}.
	In general, we see that $3$ source tasks significantly outperform $1$ source task while the performance saturates as adding $10$ more points per source task seems to benefit more than adding one more source task.
	Note here that all source data are generated independently, i.e. the observations of each task are not restricted to the same input locations.
	
	\paragraph{Further Plots and Experiments:}
	
	The main~\cref{table-discovered_regions} and~\cref{table-infer_time} present only the summary results.
	In~\cref{figureS3_mogp_branin}, we additionally provide the region clustering and fitting time w.r.t. AL iterations.
	Furthermore,~\cref{table-safety_al} counts the AL selected queries which, after a safety measurements are accessed, actually satisfy the safety constraints.
	This table is a sanity check that the methods are selecting points safely.
	
	With the PEngine datasets, we perform additional experiments of learning $\bm{f}=\bm{q}=$temperature, and the results are shown in~\cref{figureS4}.

	\begin{figure}[h]
		\begin{center}
			\centerline{\includegraphics[width=\columnwidth]{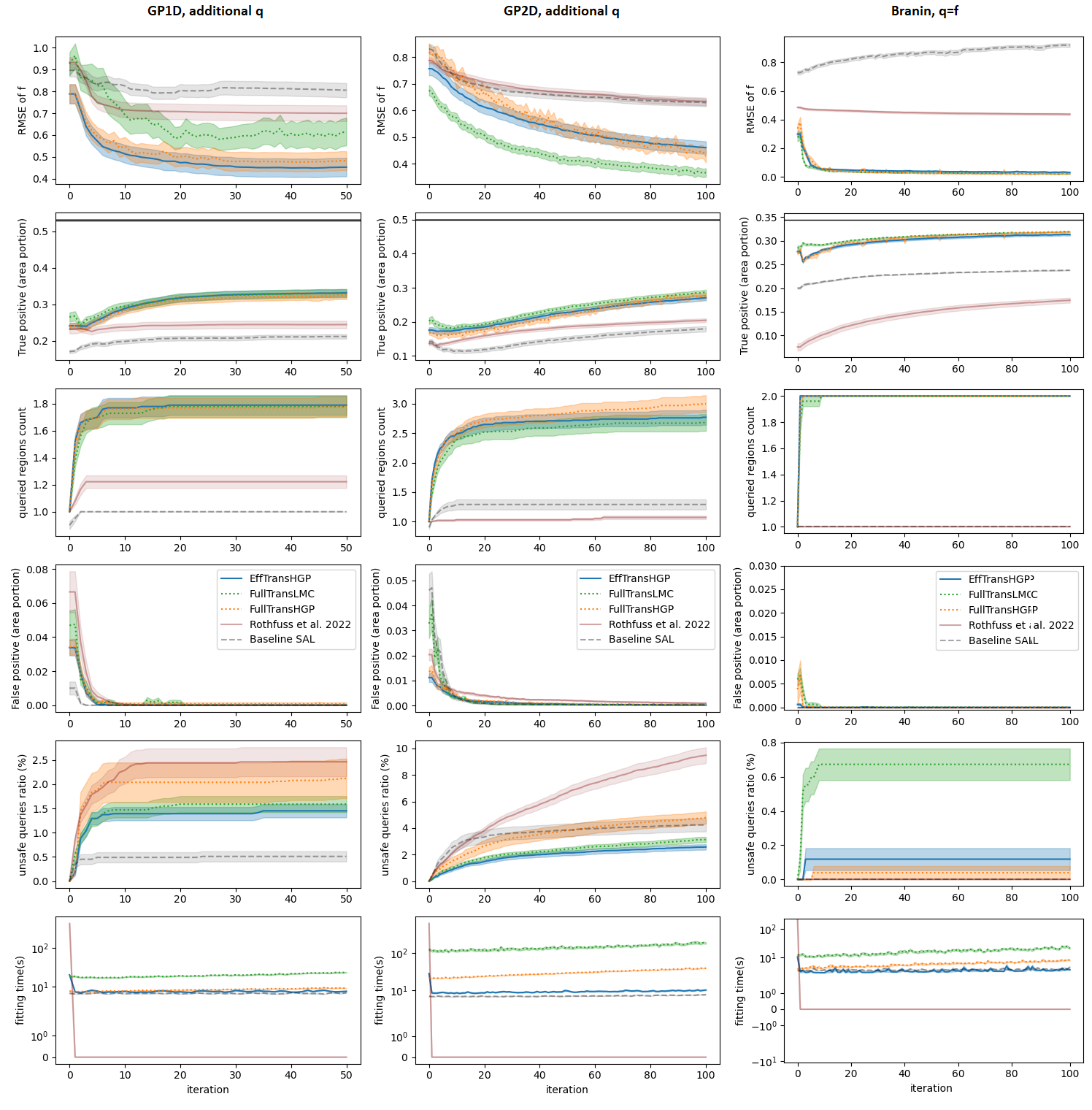}}
			\caption{
				Safe AL experiments on three benchmark datasets: GP data with $\mathcal{X}=[-2, 2]^D$, $D=1$ or $2$, constrained to $z \geq 0$, and the benchmark Branin function with constraint $y \geq 0$.
				The results are mean and one standard error of 100 (GP data) or 25 (Branin data) experiments.
				$\mathcal{X}_{\text{pool}}$ is discretized from $\mathcal{X}$ with $N_{\text{pool}}=5000$.
				We set $N_{\text{source}} = 100$ and $N$ is from $10$ (0th iteration) to $60$ (50th iteration) for GP1D, $N_{\text{source}} = 250, N$ is $20$ to $120$ for GP2D, and $N_{\text{source}} = 100, N$ is $20$ to $120$ for Branin.
				The first, second and fourth rows are presented in~\cref{main-figure} of the main paper.
				The TP/FP areas are computed as number of TP/FP points divided by $N_{\text{pool}}$ (i.e. TP/FP as portion of $\mathcal{X}_{\text{pool}}$).
				The third row shows the number of disjoint safe regions explored by the queries (main~\cref{table-discovered_regions} is taken from the last iteration here).
				The fifth row, the unsafe queries ratio, are presented as percentage of number of iterations (e.g. at the $2$nd-iteration out of a total of $50$ iterations, one of the two queries is unsafe, then the ratio is $1$ divided by $50$).
				The last row demonstrates the model fitting time.
				At the first iteration (iter $0$-th), this includes the time for fitting both the source components and the target components (EffTransHGP).
				With Rothfuss et al. 2022, source fitting is the meta learning phase.
			}
			\label{figureS3_mogp_branin}
		\end{center}
	\end{figure}
	
	\begin{table}
		\caption{Ratio of Safe Queries} \label{table-safety_al}
		\begin{center}
			\begin{tabular}{r|ccccc}
				\toprule
				\textbf{Methods}
				&\textbf{GP1D}
				&\textbf{GP2D}
				&\textbf{Branin}
				&\textbf{Hartmann3}
				&\textbf{GEngine}\\
				$N_{\text{query}}$
				&$50$
				&$100$
				&$100$
				&$100$
				&$200$\\
				\hline
				EffTransHGP
				& $ 0.986 \pm 0.001 $
				& $ 0.974 \pm 0.002 $
				& $ 0.999 \pm 0.0006 $
				& $ 0.972 \pm 0.003 $
				& $ 0.936 \pm 0.003 $\\
				FullTransHGP
				& $ 0.979 \pm 0.004 $
				& $ 0.952 \pm 0.005 $
				& $ 0.9996 \pm 0.0004 $
				& $ 0.972 \pm 0.003 $
				& $ 0.947 \pm 0.01 $\\
				FullTransLMC
				& $ 0.984 \pm 0.002 $
				& $ 0.969 \pm 0.002 $
				& $ 0.993 \pm 0.0009 $
				& $ 0.968 \pm 0.003 $
				& $ 0.91 \pm 0.008 $\\
				Rothfuss2022
				& $ 0.975 \pm 0.003 $
				& $ 0.905 \pm 0.006 $
				& $ 1.0 \pm 0.0 $
				& $ 0.84 \pm 0.011 $
				& $ 0.765 \pm 0.035 $\\
				SAL
				& $ 0.995 \pm 0.001 $
				& $ 0.958 \pm 0.005 $
				& $ 1.0 \pm 0.0 $
				& $ 0.966 \pm 0.002 $
				& $ 0.954 \pm 0.005 $\\
				\bottomrule
			\end{tabular}
		\end{center}
		Ratio of all queries selected by the methods which are safe in the ground truth (initial data not included, see~\cref{section-experiments} for the experiments).
		This is a sanity check in additional to FP safe set area, demonstrates that all the methods are safe during the experiments.
  Note that our benchmark problems all have around $35$\% to $65$\% of the space unsafe.
    Note that $\beta=4$ implies that, with a well-fitted safety GP, we tolerate a $2.275$\% probability of unsafe evaluations.
		%Note: $\beta=4$ (equivalently $\alpha=1-\Phi(\beta^{1/2})=0.002275$) implies $2.275$ \% unsafe tolerance is allowed by each fitted GP safety model.
		PEngine results are not shown because the queries are all safe (the modeling FP safe set area is almost zero in this problem, see~\cref{main-figure} and~\cref{figureS4}).
	\end{table}
	
	\begin{figure}[h]
		\vskip 0.2in
		\begin{center}
			\centerline{\includegraphics[width=\columnwidth]{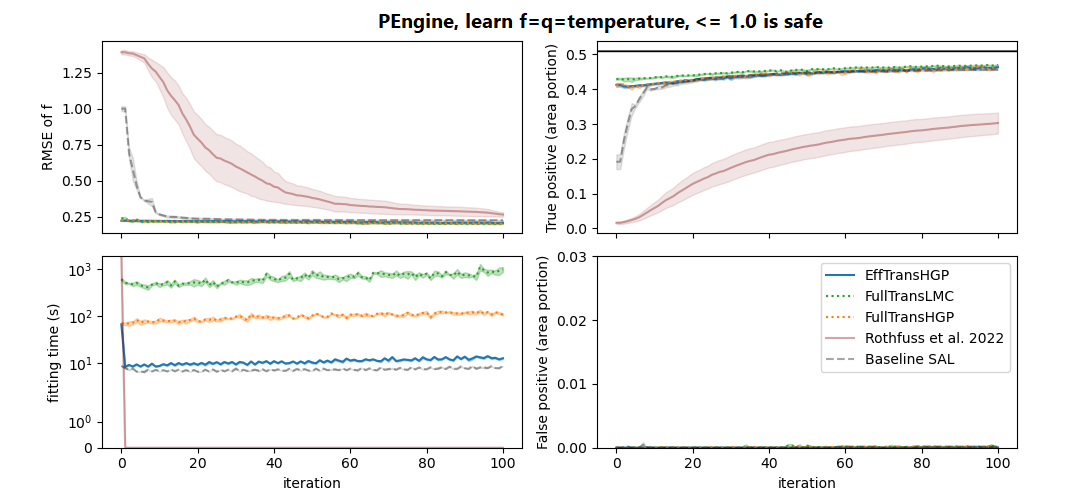}}
			\caption{
				Safe AL experiments on PEngine temperature, AL on $f$ (temperature) constrained by $q=f \leq 1.0 $.
				Baseline is safe AL without source data.
				Transfer is LMC without modularization.
				Efficient\_transfer is HGP with fixed and pre-computed source knowledge.
				$N_{\text{source}} = 500$, $N$is from $20$ to $120$.
				The results are mean and one standard error of 5 repetitions.
				The fitting time is in seconds.
			}
			\label{figureS4}
		\end{center}
		\vskip -0.2in
	\end{figure}
	
	\newpage
	
	\clearpage % force a pagebreak and flush all deferred `table` and `figure` environments

\end{document}